\documentclass[twoside,11pt]{article}

\usepackage{jmlr2e}

\usepackage{soul}
\usepackage{microtype}
\usepackage{graphicx}
\usepackage{subfigure}
\usepackage{booktabs} 
\usepackage{makecell}
\usepackage{fancybox}
\usepackage{soul} 
\usepackage{graphicx} 
\usepackage{amsmath}
\usepackage{booktabs}
\usepackage{algorithm}
\usepackage{algorithmic} 
\usepackage{todonotes}

\usepackage{amsfonts} 
\usepackage{verbatim} 
\usepackage{amssymb}
\usepackage{verbatim}
\usepackage{enumitem}
\usepackage{sidecap,color} 

\def \w {\mathbf{w}}
\def \z {\mathbf{z}}
\def \W {\mathbf{W}}

\def \E {\mathbb{E}}
\def \I {\mathbb{I}}

\def \G {\mathcal{G}} 
\def \R {\mathbb{R}} 
 
\def \ty {\tilde{y}}  
\def \a {\mathbf{a}}

\def \bx {\bar{x}}
\def \by {\bar{y}}
\def \bu {\bar{u}}
\def \bv {\bar{v}}

\newtheorem{assumption}{Assumption} 

\newlength\myindent
\usepackage{lastpage}
\jmlrheading{24}{2023}{1-\pageref{LastPage}}{12/21; Revised
1/23}{2/23}{21-1471}{Zhishuai Guo,  Yan Yan, Zhuoning Yuan, Tianbao Yang. This work was done when all authors were affiliated with the University of Iowa}

\ShortHeadings{Fast Objective \& Duality Gap Convergence}{Guo, Yan, Yuan, Yang} 
\firstpageno{1}

\begin{document}

\title{Fast Objective \& Duality Gap Convergence for Non-Convex Strongly-Concave Min-Max Problems with PL 
Condition}

\author{\name Zhishuai Guo$^\dagger$ 
        \email zhishguo@tamu.edu \\
        \name Yan Yan$^\ddagger$
        \email yan.yan1@wsu.edu\\
        \name Zhuoning Yuan$^\S$
        \email zhuoning-yuan@uiowa.edu\\
        \name Tianbao Yang$^\dagger$ \email tianbao-yang@tamu.edu\\ 
        \addr $^\dagger$Department of Computer Science and Engineering,  Texas A$\&$M University \\
        \addr $^\ddagger$School of Electrical Engineering and Computer Science, Washington State University\\
        \addr $^\S$Department of Computer Science, The University of Iowa
        } 
\editor{Francesco Orabona} 

\maketitle

\begin{abstract}
This paper focuses on stochastic methods for solving smooth non-convex strongly-concave min-max problems, which have received increasing attention  due to their potential applications in deep learning (e.g., deep AUC maximization, distributionally robust optimization).  
However, most of the existing algorithms are slow in practice, and their analysis revolves around the convergence to a nearly stationary point.
We consider leveraging the  Polyak-\L ojasiewicz (PL) condition to design faster stochastic algorithms with stronger convergence guarantee. Although PL condition has been utilized for designing many stochastic minimization algorithms, their applications for non-convex min-max optimization remain rare. 
In this paper, we propose and analyze a generic framework of proximal stage-based method with many well-known stochastic updates embeddable. Fast convergence is established in terms of both {\bf the primal objective gap and the duality gap}. Compared with existing studies,  (i) our analysis is  based on a novel Lyapunov function consisting  of the primal objective gap and the duality gap of a regularized function, and (ii) the results are more comprehensive with improved rates that have better dependence on the condition number under different assumptions.   
We also conduct deep and non-deep learning experiments to verify the effectiveness of our methods. 
\end{abstract}

\section{Introduction}
Min-max optimization has a broad range of applications in machine learning. In this paper, we consider a family of min-max optimization problems where the objective function is non-convex in terms of the min variable and is strongly concave in terms of the max variable. It covers a number of important applications in machine learning, such as deep AUC maximization \citep{ying2016stochastic,liu2019stochastic,guo2020communication} and distributionally robust optimization (DRO) \citep{namkoong2016stochastic,namkoong2017variance,rafique2018non}.  In particular,  we study stochastic gradient methods for solving the following {\bf non-convex strongly-concave (NCSC)} min-max problem: 
\begin{align}\label{eqn:op} 
    \min\limits_{x\in \R^d} \max\limits_{y\in \mathcal{Y}} f(x, y),
\end{align} 
where $\mathcal{Y}\subseteq\R^{d'}$ is a convex closed set, $f(x, y)$ is smooth, non-convex in $x$ and strongly concave in $y$. 
We assume the optimization is only through a stochastic gradient oracle that for any $x, y$ returns unbiased stochastic gradient $(\mathcal G_x(x, y; \xi), \mathcal G_y(x, y; \xi))$, i.e.,  $\E[\mathcal G_x(x, y; \xi)] = \nabla f_x(x, y)$ and $\E[\mathcal G_y(x, y; \xi)] = \nabla f_y(x, y)$.  

Stochastic algorithms for solving~(\ref{eqn:op}) have been studied in some recent papers~\citep{lin2019gradient,lin2020near,liu2019stochastic,rafique2018non,yan2020sharp,yang2020global}. However, most of them are slow in  practice by suffering from a high order of stochastic first-order oracle call complexity, while others hinge on a special structure of the objective function for constructing the  update~\citep{liu2019stochastic}.  {\it How to improve the convergence for generic non-convex strongly-concave min-max problems remains an active research area.}  There are two lines of work trying to reduce the stochastic first-order oracle call complexity of stochastic algorithms for NCSC min-max optimization.  The first line is to leverage the geometrical structure of the objective function, in particular the   Polyak-\L ojasiewicz (PL) condition~\citep{liu2019stochastic,yang2020global}.  The second line is 
leverage variance-reduction techniques ~\citep{luo2020stochastic,yang2020global,DBLP:journals/corr/abs-2008-08170,xu2020enhanced,rafique2018non}.  

In this paper, we conduct a comprehensive study to improve the convergence for NCSC min-max optimization by leveraging the Polyak-\L ojasiewicz (PL) condition of the objective function. 
A smooth function $h(x)$ satisfies $\mu$-PL condition on $\R^d$, if for any $x\in \R^d$ there exists $\mu>0$ such that $\|\nabla h(x)\|^2\geq 2\mu(h(x) - h(x_*))$, where $x_*$ denotes a global minimum of $h$. 
Although the PL condition has been utilized extensively to improve the convergence for minimization problems~\citep{allen2018convergence,arora2018convergence,charles2017stability,du2018gradient,hardt2016identity,karimi2016linear,lei2017non,li2018learning,li2017convergence,li2018simple,nguyen2017stochastic,polyak1963gradient,reddi2016stochastic,wang2018spiderboost,zhou2018stochastic,zhou2017characterization}, its application to non-convex min-max problems remains rare~\citep{liu2019stochastic,nouiehed2019solving,yang2020global}.  The key difference between the present work and these previous studies is that we focus on {\bf improving the dependence of convergence rate on the condition number} (the ratio of smoothness parameter to the PL constant)  for NCSC min-max optimization.  Our contributions are summarized below. 

\begin{itemize}[leftmargin=*]
\item{\bf Algorithms.} We analyze a generic framework of  proximal stage-wise stochastic (PES) method, which in design is similar to practical stochastic gradient methods for deep learning.  In particular, the step sizes are decreased geometrically in a stage-wise manner. Various stochastic updates can be leveraged as a plug-in in the PES framework, including stochastic optimistic gradient descent ascent (OGDA) update,  stochastic gradient descent ascent (SGDA) update, and min-max adaptive stochastic gradient (AdaGrad) update, and min-max STORM update (a recursive variance reduced method). 

\item {\bf Analysis.} We conduct novel analysis of the proposed stochastic methods by establishing fast convergence in terms of both {\it the primal objective gap} and {\it the duality gap} under different PL conditions. The analysis is based on a novel Lyapunov function that consists of the primal objective gap and the duality gap of a regularized problem. The convergence of the primal objective gap  only requires a weaker PL condition defined on the primal objective. For the convergence of the duality gap,  the objective function satisfying a pointwise PL condition in terms of $x$ is assumed. 

\item {\bf Improvements.} We make non-trivial improvements of the basic convergence rate by improving its dependence on the condition number under different conditions, include the almost-convexity condition with a small weak-convexity parameter, the slow growth condition of stochastic gradient for AdaGrad update, the individual smoothness condition for STORM update. The  dependence on the condition number can be reduced from $O(\ell^4/\mu^2)$ to $O(\ell^2/\mu)$ and $O(\ell/\mu)$ under appropriate conditions. We summarize our convergence results  on both objective gap and duality gap in Table \ref{tab:compare_with_two_sided_PL}.

\end{itemize}
\vspace*{-0.1in}
Finally, we demonstrate the effectiveness of the proposed methods on non-convex AUC maximization with a square surrogate loss and non-convex distributionally robust optimization. 
It is also notable that the proposed method has been used in the literature for maximizing a robust objective for deep AUC maximization~\citep{robustdeepAUC}, which further demonstrates the effectiveness of the proposed methods. 

\begin{table*}[t]
\caption{Comparison of sample complexities for achieving $\epsilon$-Objective Gap and and $\epsilon$-Duality Gap. $P(x)$ is $L$-smooth and is assumed to obey $\mu$-PL condition;  $f(x, y)$ is $\ell$-smooth in terms of $x$ and $y$, and is $\mu_y$ strongly concave in terms of $y$. For duality gap convergence, it requires a stronger assumption that $f(x, y)$  satisfies $x$-side $\mu_x$-PL condition. $^*$ marks the results that are not available in the original work but are derived by us.}

\scalebox{0.65}{ 
\begin{tabular}{c|c|c|c|c|c} 
\hline
& \multicolumn{2}{|c|}{Objective Gap} & \multicolumn{2}{|c|}{Duality Gap} & Remarks on \\  
& \multicolumn{2}{|c|}{} & \multicolumn{2}{|c|}{} & Conditions \\ 
\hline
& $L = \ell+\frac{\ell^2}{\mu_y} $   & $L<\ell+\frac{\ell^2}{\mu_y}$ & $L = \ell+\frac{\ell^2}{\mu_y} $   & $L<\ell+\frac{\ell^2}{\mu_y}$ \\
\hline
\makecell{Stoc-AGDA\\\citep{yang2020global}}& $O\left(\frac{\ell^5}{\mu^2 \mu_y^4 \epsilon}\right)$  &$O\left(\frac{\ell^5}{\mu^2 \mu_y^4 \epsilon}\right)$  & $O\left(\frac{\ell^7}{\mu^2\mu_x \mu_y^5 \epsilon}\right)^*$ 
& $O\left(\frac{\ell^7}{\mu^2\mu_x \mu_y^5 \epsilon}\right)^*$ & w/o strong concavity\\
\hline 
\makecell{PES-OGDA\\PES-SGDA}& $\widetilde{O}\left(\frac{\ell^4}{\mu^2 \mu_y^3 \epsilon}\right)$  & $\widetilde{O}\left(\frac{(L+\ell)^2}{\mu^2 \mu_y\epsilon }\right)$ &   $\widetilde{O}\left(\frac{\ell^5}{\mu^2 \mu_x \mu_y^3 \epsilon}\right)$ & $\widetilde{O}\left(\frac{(L+\ell)^2 \ell}{\mu^2\mu_x \mu_y\epsilon}\right)$ & w/ strong concavity\\ 
\hline 
\hline 
\makecell{PES-OGDA\\PES-SGDA}& $\widetilde{O}\left(\frac{\ell}{\min\{\mu,\mu_y\} \epsilon}\right)$  &    $\widetilde{O}\left(\frac{\ell}{\min\{\mu,\mu_y\} \epsilon}\right)$  &    $\widetilde{O}\left(\frac{\mu\ell}{\mu_x\min\{\mu,\mu_y\} \epsilon}\right)$ & $\widetilde{O}\left(\frac{\mu\ell}{\mu_x\min\{\mu,\mu_y\} \epsilon}\right)$ & \makecell{$\rho$-weakly Convex\\ $\rho<O(\mu)$} \\ 
\hline
PES-AdaGrad& $\widetilde{O}\left(\left(\frac{\ell^4}{\mu^2\mu^3_y\epsilon} \right)^{\frac{1}{2(1-\alpha)}}\right)$  &$\widetilde{O}\left(\frac{(L+\ell)^2}{\mu^2\mu_y\epsilon}\right)^{\frac{1}{2(1-\alpha)}}$ & $\widetilde{O}\left( \left(\frac{\ell^5}{\mu^2\mu_x\mu^3_y\epsilon}   \right)^{\frac{1}{2(1-\alpha)}}\right)$   &$\widetilde{O}\left(\left(\frac{(L+\ell)^2\ell}{\mu^2\mu_x\mu_y\epsilon}\right)^{\frac{1}{2(1-\alpha)}} \right)$ & \makecell{Slow SG Growth\\ (growth rate $\alpha\in(0,1/2)$)} \\
\hline 
PES-STORM& $\widetilde{O} \left( \frac{\ell^2}{\mu \mu_y^2 \epsilon} \right)$  &  $\widetilde{O} \left( \frac{\ell^2}{\mu \mu_y^2 \epsilon} \right)$  & $\widetilde{O} \left( \frac{\ell^4}{\mu \mu_x \mu_y^3 \epsilon} \right)$
& $\widetilde{O} \left( \frac{\ell^4}{\mu \mu_x \mu_y^3 \epsilon} \right)$    &Individual Smoothness \\   
\hline 
\end{tabular} 
}

\label{tab:compare_with_two_sided_PL}
\end{table*}
\section{Related Work}
\label{sec:relatedwork}
\subsection{Non-Convex Min-Max Optimization}
Recently, there has been an increasing interest on non-convex min-max optimization  \citep{rafique2018non,jin2019minmax,lin2018solving,lin2019gradient,liu2020towards,lu2019hybrid,nouiehed2019solving,sanjabi2018solving,DBLP:conf/nips/Thekumparampil019,ostrovskii2020efficient,lin2020near,yang2020global,luo2020stochastic,xu2020enhanced,DBLP:journals/corr/abs-2008-08170,DBLP:conf/nips/Tran-DinhLN20,lu2019hybrid,2020arXiv200713605I,zhao2020primal,wang2020zeroth,yang2020catalyst,zhang2021complexity,qiu2020single,han2021lower,tran2020hybrid,huang2021efficient,xian2021faster,luo2021finding,fiez2021minimax,xu2021zeroth,lei2021stability}.   
Below, we focus on related works on stochastic optimization for non-convex concave min-max problems.  \citet{rafique2018non} proposed stochastic algorithms for solving non-smooth weakly-convex and concave problems based on a proximal point method \citep{rockafellar1976monotone}. 
They established a convergence to a nearly stationary point of the primal objective function in the order of $O(1/\epsilon^6)$, where $\epsilon$ is the  level for the first-order stationarity. When the objective function is strongly concave in terms of $y$ and has certain special structure, they can reduce the stochastic first-order oracle call complexity to $O(1/\epsilon^4)$. The same order stochastic first-order oracle call complexity was achieved in~\citep{yan2020sharp} for weakly-convex strongly-concave problems without  a special structure of the objective function.  \citet{lin2019gradient} analyzed  a single-loop stochastic gradient descent ascent method for smooth non-convex (strongly)-concave min-max problems. Their analysis yields an stochastic first-order oracle call complexity of $O(1/\epsilon^8)$ for smooth non-convex concave problems and  $O(1/\epsilon^4)$ for smooth non-convex strongly-concave problems.
Recently, \citet{2020arXiv200713605I} extends the analysis to stochastic alternating (proximal) gradient descent ascent method.
Improved first-order convergence for smooth problems has been established by leveraging  variance-reduction techniques  in~\citep{luo2020stochastic,yang2020global,DBLP:journals/corr/abs-2008-08170,xu2020enhanced,rafique2018non}. However, none of these works explicitly use the PL condition to improve the convergence. Directly applying PL condition to the first-order convergence result leads to a stochastic first-order oracle call complexity worse than $O(1/\epsilon)$ for the objective gap.

\subsection{PL Games}  PL conditions have been considered in min-max games. For example, \citet{nouiehed2019solving} assumed that $h_x(y) = - f(x, y)$ satisfies PL condition for any $x$, which is referred to as {\bf $y$-side PL condition}. The authors utilize the condition to design deterministic multi-step gradient descent ascent method for finding a first-order stationary point. In contrast, we consider the objective is strongly concave in terms of $y$, which is stronger than $y$-side pointwise PL condition. Recently, \citet{liu2018fast} assume a PL condition for a NCSC formulation of deep AUC maximization, in which the PL condition is defined over the primal objective $P(x) = \max_{y\in \mathcal{Y}}f(x, y)$, which is referred to as {\bf primal PL condition}. They established a stochastic first-order oracle call complexity  of $O(1/\epsilon)$ for the primal objective gap convergence only. However, their algorithm and analysis are not applicable to a general NCSC problem without a special structure. In contrast, our algorithm is more generic and simpler as well, and we derive stronger convergence result in terms of the duality gap. In addition, our analysis is based on a novel Lyapunov function that consists of the primal objective gap and the duality gap of a regularized function, which allows us to establish the convergence of both the primal objective gap and the duality gap.

More recently,  \citet{yang2020global} considered a class of smooth non-convex non-concave problems, which satisfy both the $y$-side PL condition and $x$-side  PL condition\footnote{We notice that the $x$-side PL condition can be replaced by the primal PL condition for their analysis.}.  They proposed stochastic alternating gradient descent ascent (Stoc-AGDA)  algorithms and established a global convergence for a Lyapunov function $P(x_t) - P_* + \lambda (P(x_t) - f(x_t, y_t))$ for a constant $\lambda$, which directly implies the convergence for the primal objective gap. After some manipulation, we can also derive the convergence for the duality gap under the assumption that $x$-side PL condition holds. This work is different from \citep{yang2020global} in several perspectives: (i) their algorithm is based on alternating gradient descent ascent method with polynomially decreasing or very small step sizes, in contrast our algorithm is based on  stage-wise stochastic methods with geometrically decreasing step sizes. This feature makes our algorithm more amenable to deep learning applications~\citep{robustdeepAUC}; 
 (ii) we make use of strong concavity of the objective function in terms of $y$ and develop stronger convergence results. In particular, our stochastic first-order oracle call complexities have better dependence on condition numbers. 
 
Finally, we note that there are a lot of research on deep learning to justify the PL condition.
PL condition of a risk minimization problem has been shown to hold globally or locally on some networks
with certain structures, activation or loss  functions \citep{allen2018convergence,arora2018convergence,charles2017stability,du2018gradient,hardt2016identity,li2018learning,li2017convergence,zhou2017characterization}. 
For example, in \citep{du2018gradient}, they have shown that if the width of a two layer neural network is sufficiently large, PL condition holds within a ball centered at the initial solution and the global optimum would lie in this ball.
\citet{allen2018convergence} further shows that in overparameterized deep neural networks with ReLU activation, PL condition holds for a global optimum around a random initial solution. 

\section{Preliminaries}
We denote by $\|\cdot\|$ the Euclidean norm of a vector. 
A function $h(x)$ is $\lambda$-strongly convex on $\mathcal{X}$ if for any $x, x'\in \mathcal{X}$, $\nabla h(x')^{\top} (x-x') + \frac{\lambda}{2}\|x - x'\|^2 \leq h(x) - h(x')$.
A function $h(x)$ is $\rho$-weakly convex on $\mathcal{X}$ if for any $x, x'\in \mathcal{X}$, 
$\nabla h(x')^{\top} (x - x') - \frac{\rho}{2} \|x-x'\|^2 \leq h(x) - h(x')$. $h(x)$ is $L$-smooth if its gradient is $L$-Lipchitz continuous, i.e., $\|\nabla h(x) - \nabla h(x')\|\leq L\|x - x'\|, \forall x, x'\in \mathcal{X}$. An $L$-smooth function is also a $L$-weakly convex function.  A smooth function $h(x)$ satisfies $\mu$-PL condition on $\R^d$, if for any $x\in \R^d$ there exists $\mu>0$ such that $\|\nabla h(x)\|^2\geq 2\mu(h(x) - h(x_*))$, where $x_*$ denotes a global minimum of $h$. 
Let $\hat{x}(y) = \arg\min_{x'} f(x', y)$ denote the set of optimal $x$ for the fixed $y$ and when the context is clear we abuse the notation $\hat{x}(y)$ to denote any point in that set.  
Let $\hat{y}(x) = \arg\max\limits_{y'\in\mathcal{Y}} f(x, y')$ denote the optimal $y$ for the fixed $x$. 
 
For simplicity, we let $z= (x, y)^{\top}$, $\mathcal{Z}=\mathcal{X} \times \mathcal{Y}=\R^d\times \mathcal{Y}$, $F(z) = (\nabla_x f(x, y), - \nabla_y f(x, y))^{\top}$ and $\mathcal G(z;\xi) = (\nabla_x f(x, y; \xi), - \nabla_y f(x, y; \xi))^{\top}\in\R^{d+d'}$. We abuse the notations $\|z\|^2 = \|x\|^2 + \|y\|^2$ and $\|F(z) - F(z')\|^2 =\|\nabla_x f(x, y) - \nabla_x f(x', y')\|^2+ \|\nabla_y f(x, y)- \nabla_y f(x', y')\|^2 $. Let $P(x)=\max_{y\in \mathcal{Y}}f(x, y)$.  The primal objective gap of a solution $x\in \mathcal{X}$ is defined as $P(x) - \min_{x\in \mathcal{X}}P(x)$. 
Below, we state some assumptions that will be used  in our analysis.

\begin{assumption}\label{ass1}
(i) $F$ is $\ell$-Lipchitz continuous, i.e., $\|F(z) - F(z')\| \leq \ell \|z - z'\|$, for any $z, z'\in Z$
(ii) $f(x, y)$ is $\mu_y$-strongly concave in $y$ for any $x$;
(iii) $P(x)=\max_{y\in \mathcal{Y}}f(x, y)$ is $L$-smooth and has a non-empty optimal set. 
\end{assumption} 

{\bf Remark:} Assumption~\ref{ass1}(i) implies that $f(x,y)$ is $\ell$-smooth in terms of $x$ for any $y\in \mathcal{Y}$.  Note that under Assumption~\ref{ass1}(i) and (ii), we can derive that $P(x)$ is  $(\ell+\ell^2/\mu_y)$-smooth~\citep{lin2019gradient}. However, we note that the smoothness parameter $L$ could be much smaller than $(\ell+\ell^2/\mu_y)$, and hence we keep  dependence on $L$, $\ell$, $\mu_y$ explicitly. 
For example, consider $f(x,y) = x^{\top}y - \frac{\mu_y}{2}\|y\|^2 -(\frac{1}{2\mu_y} - \frac{L}{2})\|x\|^2$, $\mathcal{Y}=\R^{d'}$ with $L\ll 1\ll 1/\mu_y$. Then we can see that $F(z)$ is $\ell = (1+\frac{1}{\mu_y}  - L)$-Lipchitz continuous. 
However, $P(x) =  \frac{L}{2}\|x\|^2$ is $L$-smooth function and $L$ could be much smaller than $\ell+\ell^2/\mu_y$. 

The following assumption is assumed regarding the stochastic gradients unless specified otherwise.
\begin{assumption}\label{ass2}
There exists $\sigma>0$ such that $\E[\|\nabla_x f(x, y; \xi) - \nabla_x f(x, y)\|^2]\leq \sigma^2$ and $\E[\|\nabla_y f(x, y; \xi) - \nabla_y f(x, y)\|^2]\leq \sigma^2$. 
\end{assumption}
{\bf Remark: } In order to use a simple stochastic gradient descent ascent update, we need to impose a different (non-typical) assumption on stochastic gradients for analysis, i.e., there exists $B>0$ such that $\E[\|\nabla_x f(x, y; \xi)\|^2]\leq B^2$ and $\E[\|\nabla_y f(x, y; \xi)\|^2]\leq B^2$.

If $f(x, y)$ is $\ell$-smooth, it is then weakly convex with a coefficient $\rho$ no greater than $\ell$, however, $\rho$ can be much less than $\ell$. 
In order to explore possibilities for deriving faster convergence, we could leverage the weak convexity of $f(x, y)$ in terms of $x$. 
\begin{assumption}\label{ass5}
$f(x, y)$ is $\rho$-weakly convex in terms of $x$ for any $y\in \mathcal{Y}$ with $0<\rho\leq \ell$. 
\end{assumption}
For example, consider $f(x, y) = \ell x^{\top}y - \frac{\mu_y}{2}\|y\|^2 - \frac{\rho}{2}\|x\|^2$ with $\rho\leq  \ell$. Then $F(z)$ is $(\ell+\max(\rho,\mu_y))$-Lipchitz continuous. However, $f(x,y)$ is $\rho$-weakly convex in terms of $x$ for any $y$. 

In the algorithms, let $\Pi_{\bar z}(\mathcal G)\in \mathcal{Z}$ and $\Pi^{\gamma}_{\bar z, x_0}(\mathcal G)\in \mathcal{Z}$ be defined as 
\begin{equation}
\begin{aligned}
&\Pi_{\bar z}(\mathcal G) = \arg\min_{z\in \mathcal{Z}} \mathcal G^{\top}z + \frac{1}{2}\|z - \bar z\|^2, \\ 
&\Pi^{\gamma}_{\bar z, x_0}(\mathcal G) = \arg\min_{z\in \mathcal{Z}} \mathcal G^{\top}z + \frac{1}{2}\|z - \bar z\|^2 + \frac{\gamma}{2}\|x - x_0\|^2.
\end{aligned} 
\end{equation}  
Let $\mathcal{P}_{\mathcal{Y}}(\cdot)$ denote an Euclidean projection to $\mathcal{Y}$.

\section{PL-Strongly-Concave Problems and Applications in Machine Learning}
\label{sec:app}
Firstly, based on the definition of PL condition given in the last section, we define the different PL conditions for the min-max problem. 
\begin{definition}
$f(x, y)$ satisfies a primal $\mu$-PL condition for some constant $\mu>0$ if $P(x) = \max_{y\in \mathcal{Y}} f(x, y)$ satisfies $\mu$-PL condition, i.e., $\|\nabla P(x)\|^2 \geq 2\mu (P(x) - \min_{x'} P(x'))$.  
\end{definition}  

\begin{definition}
$f(x, y)$ satisfies a $x$-side $\mu_x$-PL condition for some constant $\mu_x>0$ if for any $y\in \mathcal{Y}$, $f(x, y)$ satisfies $\mu_x$-PL condition, i.e., $\forall y\in \mathcal{Y}$, $\|\nabla_x f(x, y)\|^2 \geq 2\mu_x (f(x, y) - f(\hat{x}(y), y))$.  
\end{definition} 

We  define almost PL conditions as follows.
\begin{definition}
$f(x, y)$ satisfies an $\epsilon$-almost primal $\mu$-PL condition if for $P(x) = \max_{y\in \mathcal{Y}} f(x, y)$, there exists  $\mu>0$ such that $\|\nabla P(x)\|^2 \geq 2\mu (P(x) - \min_{x'} P(x') - \epsilon)$, where $\epsilon>0$ is the accuracy level.  
\end{definition} 

\begin{definition}
$f(x, y)$ satisfies an $\epsilon$-almost $x$-side $\mu_x$-PL condition if there exists  $\mu_x>0$ such that $\|\nabla_x f(x, y)\|^2 \geq 2\mu_x (f(x, y) - f(\hat{x}(y), y) - \epsilon)$, where $\epsilon>0$ is the accuracy level.   
\end{definition} 
It is not hard to see that convergence rates under the $\epsilon$-almost $x$-side PL condition or the $\epsilon$-almost primal PL condition are identical to that under the $x$-side PL condition or the primal PL condition, respectively. Therefore, in the convergence analysis we focus on the $x$-side PL condition and the primal PL condition. 

We define two kinds of PL-strongly-concave problems as follows.
\begin{definition}
$f(x, y)$ is primal-PL-strongly-concave if $f(x, y)$ satisfies a primal $\mu$-PL condition and is strong concave in $y$ for any $x$.
\end{definition}
\begin{definition}
$f(x, y)$ is $x$-side-PL-strongly-concave if $f(x, y)$ satisfies a $x$-side $\mu_x$-PL condition and is strong concave in $y$ for any $x$.
\end{definition}

It has been shown in \cite{yang2020global} that the $x$-side $\mu_x$-PL condition of $f(x,y)$ is stronger than $\mu$-PL condition of $P(x)$ under strong concavity of $f(x,y)$ in terms of $y$. 

\begin{lemma}[Lemma A.3 of \citet{yang2020global}] 
If  $f(x, y)$ satisfies $x$-side $\mu_x$-PL condition on $\R^d$ and is strongly concave in $y$,  then $P(x)=\max_{y\in \mathcal{Y}}f(x, y)$  satisfies $\mu$-PL condition for some $\mu\geq \mu_x$. 
\label{lem:primal_PL} 
\end{lemma}

Here we show cases where the $x$-side $\mu_x$-PL condition holds or does not hold. Fortunately, $x$-side PL condition (Assumption \ref{ass4}) is only needed in  Section \ref{sec:dualitygap} to develop duality gap convergence. 
We can construct a function that does not obey a $x$-side  $\mu_x$-PL condition but satisfies a primal $\mu$-PL condition. Let us consider $f(x, y)= xy - \frac{1}{2}y^2 - \frac{1}{4}x^2$ and $\mathcal{Y}=\R$. First, we show that $\mu_x$-PL condition does not hold. To this end, fix $y=1$, we can see that $|\nabla_x f(x, y)|^2 = (1-x/2)^2$, and $\min_{x\in \mathcal{X}} f(x, 1) = \min_{x}x(1-x/4) -  \frac{1}{2}= - \infty$. Hence, for $x=2+\epsilon$, we have $|\nabla_x f(x, y)|^2=(\epsilon/2)^2$ and $f(x, 1) - \min_{x\in \mathcal{X}} f(x, 1) =  \infty$. However, there exists no constant $\mu_x$ such that $|\nabla_x f(x, y)|^2\geq \mu_x(f(x, 1) - \min_{x\in \mathcal{X}} f(x, 1))$ for $\epsilon\rightarrow 0$. Second, we can see that $P(x)=\max_{y}f(x, y) = \frac{x^2}{4}$ satisfies $\mu$-PL condition with $\mu = 1/2$.
This argument together with Theorem~\ref{thm:1} implies that our result for the convergence of the primal objective gap only requires a weaker $\mu$-PL condition other than the $x$-side $\mu_x$-PL condition imposed in~\citep{yang2020global}. 
An example that satisfies both the $x$-side $\mu_x$-PL condition and $y$-side strong concavity is $f(x, y) = \frac{1}{2}x^2 + \sin^2 x \sin^2 y - 2y^2$, which is verified in Lemma \ref{lem:x_side_PL_case} in the Appendix.

Instead of imposing the $x$-side PL condition as in \citep{yang2020global},
we use primal PL condition (Assumption~\ref{ass3}) for proving the convergence of the primal objective gap, and use $x$-side PL condition (Assumption~\ref{ass4}) only for proving the convergence of the duality gap. \cite{yang2020global} also makes an extra assumption that there exists a saddle point, i.e., there exists $(x_*, y_*)$ such that $f(x_*, y) \leq f(x_*, y_*) \leq f(x, y_*)$. However, we show in Lemma \ref{lem:saddle_point} that a saddle point $(x_*, y_*)$  exists for the $x$-side-PL-strongly-concave problem. 

\begin{lemma}
\label{lem:saddle_point} 
Assume $f(x, y)$ satisfies a $x$-side $\mu_x$-PL condition and is strongly concave in $y$ and let $x_* = \arg\min_{x'} P(x')$ where $P(x) = \max_{y\in \mathcal{Y}} f(x, y)$. Then $(x_*, \hat{y}(x_*))$ is a saddle point of $f(x, y)$. 
\end{lemma}

It has been shown in Lemma 2.1 of \citep{yang2020global} that if the $x$-side $\mu_x$-PL condition holds, then the saddle points, global min-max points, and stationary points are equivalent when $\mathcal{Y}=\mathbb{R}^{d'}$, where global min-max points, and stationary points are defined as 
\begin{enumerate}
    \item $(x_*, y_*)$ is a global min-max point if for any $(x, y)$: $f(x_*, y) \leq f(x_*, y_*) \leq \max_{y'} f(x, y')$.
    
    \item $(x_*, y_*)$ is a stationary point if $\nabla_x f(x_*, y_*) = \mathbf{0}$ and  $\nabla_y f(x_*, y_*) = \mathbf{0}$. 
\end{enumerate}

Next we show two concrete application examples of PL-strongly-concave problems in machine learning. 

{\bf Deep AUC Maximization}  
The area under the ROC curve (AUC) on a population level for a scoring function $h: \mathcal{X}\rightarrow\mathbb{R}$ is defined as 
\vspace{-0.1in} 
\begin{equation} 
AUC(h) = \text{Pr}(h(\mathbf{a}) \geq h(\mathbf{a}') | b=1, b'=-1),
\end{equation}  
where $\mathbf{a}, \mathbf{a}' \in \mathbf{R}^{d_0}$ are data features, $b, b' \in \{-1, 1\}$ are the labels, $\z = (\mathbf{a}, b)$ and $\z'=(\mathbf{a}', b')$ are drawn independently from $\mathbb{P}$.
By employing the squared loss as the surrogate for the indicator function which is commonly used by previous studies~\citep{ying2016stochastic,liu2018fast,liu2019stochastic}, the deep AUC maximization problem can be formulated as 
\begin{equation}
\min\limits_{\w\in\mathbb{R}^d} \E_{\z, \z'}\left[(1-h(\w;\mathbf{a})+h(\w;\mathbf{a}'))^2| b=1, b'=-1\right], 
\label{prob:auc_square} 
\end{equation}  
where $h(\w; \mathbf{a})$ denotes the prediction score for a data sample $\mathbf{a}$ made by a deep neural network parameterized by $\w$. 
It was shown in \citep{ying2016stochastic} that the above problem is equivalent to the following min-max problem: 
\begin{equation}
\label{auc_min-max-1}
\min\limits_{(\w, s, r)} \max\limits_{y\in \mathbb{R}} f(\w, s, r, y)=\E_\z[F(\mathbf{w}, s, r, y, \z)], 
\end{equation}
where 
\begin{equation}
\begin{split}
F(\w, s, r, y; \z) =& (1-p) (h(\w; \mathbf{a})- s)^2 \mathbb{I}_{[b=1]} 
+ p(h(\w; \mathbf{a}) - r)^2\mathbb{I}_{[b=-1]}  \\
&
+ 2(1+y)(p h(\w; \mathbf{a})\mathbb{I}_{[b=-1]} 
- (1-p) h(\w;\mathbf{a}) \mathbb{I}_{[b=1]}) - p(1-p)y^2,
\end{split} 
\end{equation}
where $p=\Pr(b=1)$ denotes the prior probability that an example belongs to the positive class, and $\mathbb I$ denotes an indicator function whose output is $1$ when the condition holds and $0$ otherwise. 
We denote the primal variable by $x=(\w, s, r)$. 

Obviously, the problem \eqref{auc_min-max-1} is strongly concave on dual variable $y$ for any primal variable $x$. Also, in the next lemma we show that $f(x, y)$ satisfies an $\epsilon$-almost $\mu$-PL condition with a high probability following the theory of over-parameterized deep learning for minimization problems in Theorem 1, 2, 3, 5 of \citep{allen2018convergence}. We put all the proof in the appendix. 
\begin{lemma} 
Assume that input data $\{(\a_1, b_1),\ldots, (\a_n, b_n)\}$, where $\a_i\in \mathbb{R}^{d_0}, b_i\in \{-1, 1\}$, satisfies $\|\a_i\|=1$ and  $\|\a_i - \a_j\|\geq \delta$. 
Consider a deep neural network with $h_{i,0} = \phi(A\a_i), h_{i,l} = \phi(W_l h_{i, l-1}), l=1,\ldots, \tilde{L}, \hat b_i = B^T h_{i,\tilde{L}}$ where $A\in \R^{m\times d_0}, W_l\in \R^{m\times m}, B\in \R^{m}$ are randomly initialized, and $\phi$ is the ReLU activation function. Let $\w$ denote the vectorization of $(\W_1, \cdots, \W_{\tilde{L}})$ 
and $x=(\w, s, r)$ denote the primal variable. $h(\w; a_i) = \hat{b}_i$ be the output logit for the $i$-th data.    
Take $m = \widetilde{\Omega}(\text{poly}(n, \Tilde{L}, \delta^{-1}, \epsilon))$, then with a high probability over randomness of $W_0, A, B$ for every $x$ with $\|\w - \w_0\|\leq O(\frac{\log m}{\sqrt{m}})$, $f(x, y)$ satisfies an $\epsilon$-almost primal $\mu$-PL condition.  
\label{lem:pl_auc}  
\end{lemma} 

{\bf Distributionally Robust Optimization (DRO)} 
DRO problem \citep{namkoong2017variance,rafique2018non} has a min-max formulation of 
\begin{equation}
\begin{split}
\min\limits_{x} \max\limits_{y\in \mathcal{Y}} \frac{1}{n} \sum\limits_{i=1}^{n} y_i f_i(x) - r(y), 
\end{split}
\end{equation} 
where $f_i(x)$ can be a loss function on the $i$-th data using a neural network backbone parameterized by $x$, and $r(y)$ is a reguralization function. 
The spirit of this formulation is to put more weights to the data points with high losses, thus to increase the robustness of models.
It would be strongly concave on $y$ for any $x$ if $r(y)$ is a strongly convex function. It has been shown in proof of Lemma 2 of \citep{qi2020practical} that $f(x, y)$ satisfies an $\epsilon$-almost $x$-side $\mu_x$-PL condition with a high probability for a similar network structure as in the above Lemma \ref{lem:pl_auc}. 

\section{Algorithms and Objective Gap Convergence} 
\label{sec:primal_convergence}
In this section, we make the assumption of the primal PL condition. 
\begin{assumption}\label{ass3}
$P(x) = \max\limits_{y\in \mathcal{Y}} f(x, y)$ satisfies $\mu$-PL condition.
\end{assumption} 

We present the proposed stochastic method in Algorithm~\ref{alg:main}. We would like to point out that our method follows the proximal point framework analyzed in \citep{liu2019stochastic,rafique2018non,yan2020sharp}. In particular, the proposed method includes multiple consecutive stages. In each stage, we employ a stochastic algorithm to solve the  following proximal problem approximately:
\begin{align}
f_k(x, y) = f(x, y) + \frac{\gamma}{2} \|x - x^k_{0}\|^2,
\end{align}
where $\gamma$ is an appropriate regularization parameter to make $f_k$ to be strongly convex and strongly concave.
The reference point $x^k_{0}=\bx_{k-1}$ is updated after each stage, i.e., after each  inner loop.
Let $\hat{x}_k(y) = \arg\min_{x'} f_k(x', y)$ denote the optimal $x$ for the fixed $y$ and $\hat{y}_k(x) = \arg\max\limits_{y'\in\mathcal{Y}} f_k(x, y')$
denote the optimal $y$ for the fixed $x$. 

However, there are some key differences between the proposed method from that are analyzed in ~\citep{liu2019stochastic,rafique2018non,yan2020sharp}. We highlight the differences below. First, our method explicitly leverages the PL condition of the objective function by decreasing $\eta_k, 1/T_k$ geometrically (e.g, $e^{-\alpha k}$ for some $\alpha>0$). In contrast, \citet{rafique2018non} and \citet{yan2020sharp} proposed to decrease  $\eta_k, 1/T_k$ polynomially (e.g., $1/k$). Second, the restating point and the reference point $(\bx_{k-1}, \by_{k-1})$ is simply the averaged or sampled solution of stochastic updates in our employed stochastic algorithm $\mathcal A$.  In contrast, \citet{liu2019stochastic} and \citet{rafique2018non} assumed a special structure of the objective function and leverage its structure to compute a restarted solution for $y$. This makes our method much simpler to be implemented but makes the analysis more involved. 

For stochastic algorithm $\mathcal A$, one can employ many stochastic primal-dual methods to solve $\min_{x}\max_{y}f_k(x,y)$. We consider four well-known methods with different stochastic updates. Stochastic gradient descent ascent (SGDA) update (option I) and min-max adaptive stochastic gradient (MinMax-AdaGrad) update (option III) are mostly interesting to practitioners. Stochastic optimistic gradient descent ascent (OGDA) update (option II) yields an algorithm with provable convergence result under standard assumptions for smooth problems that is more interesting to theoreticians, which was originated from stochastic mirror prox method proposed by~\citep{
juditsky2011solving}.   Min-max stochastic update based on the recursive variance reduced estimator STORM~\citep{cutkosky2019momentum} (option IV) can lead to an improved rate without using large mini-batch.

\begin{algorithm}[t]
\caption {Proximal Stage Stochastic Method: PES-$\mathcal A$}
\begin{algorithmic}[1]
\STATE{Initialization: $\bx_0\in \R^d, \by_0 \in \mathcal{Y}, \gamma, T_1, \eta_1, a$. } 
\STATE{Option III: $\bu_0 = \nabla_x f(\bx_0, \by_0;\bar{\xi}), \bv_0 = \nabla_y f(\bx_0, 
\by_0;\bar{\xi})$}.
\FOR{$k=1,2, ..., K$} 
\STATE{$x_0^k = \bx_{k-1}$, $y_0^k = \by_{k-1}$;} 
\STATE{Option I$\sim$ III: $ (\bx_k, \by_k)$ = $\mathcal A(f, x^k_0, y^k_0, \eta_k, T_k, \gamma)$;} 
\STATE{Option IV: $ (\bx_k, \by_k, \bu_k, \bv_k)$ = $\mathcal A(f, x^k_0, y^k_0, \eta_k, T_k, \gamma, \bu_{k-1}, \bv_{k-1})$;} 
\STATE{$\eta_{k+1}=\eta_k/a$,  $\eta^y_{k+1}=\eta^y_k/a$, $T_{k+1}=a T_k$;} 
\ENDFOR 
\STATE{\textbf{return} $(\bx_K, \by_K)$.}
\end{algorithmic}
\label{alg:main}
\end{algorithm}

\begin{algorithm}[t]
\caption {Stochastic Algorithm for Each Stage}
Option I$\sim$III: $\mathcal A$($f, x_0, y_0, \eta, T, \gamma)$,\\
Option IV: $\mathcal{A}$($f, x_0, y_0, \eta, T, \gamma, u_0, v_0$)
\begin{algorithmic}
\STATE{Initialization: 
$\tilde{z}_0 = z_0= (x_0, y_0)$, Option III: $g_{1:0}=[]$}
\STATE Let $\{\xi_0, \xi_1,\ldots, \xi_T\}$
be independent random variables, and $ \mathcal G_\gamma(z; \xi) = \left(\begin{array}{c}\nabla_x f(x, y; \xi) + \gamma (x - x_0)\\  - \nabla_y f(x, y; \xi)\end{array}\right)$. 
\FOR{$t=1, ..., T$} 
\STATE {Option I: \text{SGDA update:} }
\STATE {~~~~~$z_t = \Pi^{\gamma}_{z_{t-1}, x_0}(\eta \mathcal G(z_{t-1}; \xi_{t-1})) $;}
\vspace*{0.1in}
\STATE {Option II: \text{OGDA update:}} 
\hspace*{0.1in}\framebox[0.59\columnwidth]{
\begin{minipage}[b]{0.59\columnwidth} 
 $z_{t} = \Pi_{\tilde{z}_{t-1}}(\eta \mathcal G_\gamma(z_{t-1}; \xi_{t-1}))$;\newline
 $\tilde z_{t} = \Pi_{\tilde{z}_{t-1}}(\eta \mathcal G_\gamma(z_{t}; \xi_t))$;
 \end{minipage}} 
\vspace*{0.1in}
\STATE {Option III: \text{Min-Max AdaGrad update:}}

\hspace*{0.05in}\framebox[0.9\columnwidth]{
\begin{minipage}[b]{0.9\columnwidth}
$g_{1:t} = [g_{1:t-1}, \G_\gamma(z_t;\xi_t)], \text{~and~} s_{t, i} = \|g_{1:t, i}\|_2$;\newline 
Set $H_{t} = \delta I + \text{diag}(s_{t}), \psi_{t} (z) = \frac{1}{2}\langle z-z_0, H_{t}(z-z_0) \rangle$;\newline 
$z_{t+1} = \arg\min\limits_{z\in \mathcal{Z}} \eta z^T \left(  \frac{1}{t}\sum\limits_{\tau=1}^{t}\G_\gamma(z_{\tau}; \xi_{\tau})\right)+ \frac{1}{t}\psi_{t} (z)$;
\end{minipage}}
\vspace*{0.1in}
\STATE {Option IV: \text{Min-Max STORM update:}} \\
\hspace*{0.05in}\framebox[0.9\columnwidth]{
\begin{minipage}[b]{0.9\columnwidth}
\hspace*{0.02in}$x_t = x_{t-1} - \eta^x u_{t-1}$,\newline
\hspace*{0.02in}$y_t = y_{t-1} + \eta^y (\mathcal{P}_{\mathcal{Y}}(y_{t-1}+\lambda v_{t-1})- y_{t-1}))$; \\
\hspace*{0.02in}$u_t = (1-a_x)u_{t-1} + \nabla_x f(x_t, y_t; \xi_t) - (1-a_x)\nabla_x f(x_{t-1}, y_{t-1}; \xi_t)$,\\
\hspace*{0.02in}$v_t = (1-a_y)v_{t-1}  + \nabla_y f(x_t, y_t; \xi_t) - (1-a_y)\nabla_y f(x_{t-1}, y_{t-1}; \xi_t) $;
\end{minipage}
}
\ENDFOR  
\STATE{Option I$\sim$III: \textbf{return} $\bar{x} = \frac{1}{T}  \sum\limits_{t=1}^{T} x_t, \bar{y} = \frac{1}{T}  \sum\limits_{t=1}^{T} y_t$.}
\STATE{Option IV: \textbf{return} $(x_\tau, y_\tau, u_\tau, v_\tau)$ with a random index $\tau\in\{1, \ldots, T\}$.} 
\end{algorithmic} 
\label{alg:subroutine}  
\end{algorithm}

\subsection{Basic Results} 
Below, we present the basic convergence results of Algorithm~\ref{alg:main} by employing stochastic OGDA update. Let ${\text{Gap}}(x, y) = \max\limits_{y'\in \mathcal{Y}} f(x, y')- \min\limits_{x'\in \mathcal{X}} f(x', y)$ be the duality gap of $(x, y)$ on $f$ and ${\text{Gap}}_k(x, y) = \max\limits_{y'\in \mathcal{Y}} f_k(x, y')- \min\limits_{x'\in \mathcal{X}} f_k(x', y)$ be the duality gap of $(x, y)$ on $f_k$.

\vspace*{-0.1in}
\begin{theorem}\label{thm:1}
Consider Algorithm \ref{alg:main} that uses Option II: OGDA update in subroutine Algorithm \ref{alg:subroutine}.
Suppose Assumption \ref{ass1}, \ref{ass2}, \ref{ass5}, \ref{ass3}  hold.
Take $\gamma = 2\rho$ and
denote $\hat{L} = L + 2\rho$ and $c = 4\rho+\frac{248}{53} \hat{L} \in O(L + \rho)$.
Define $\Delta_k = P(x_0^k) - P(x_*) + \frac{8\hat{L}}{53c} \emph{\text{Gap}}_k (x_0^k, y_0^k)$ and $\epsilon_0=\emph{\text{Gap}}(\bx_0, \by_0)$.
Then we set $\eta_k = \eta_0  \exp(-(k-1)\frac{2\mu}{c+2\mu})
\leq \frac{1}{2\sqrt{2}\ell}$, $T_k = \left\lceil\frac{212}{\eta_0 \min\{\rho, \mu_y\}}
\exp\left((k-1)\frac{2\mu}{c+2\mu}\right)\right\rceil$. 
After $K =  \left\lceil\max\left\{\frac{c+2\mu}{2\mu}\log \frac{4\epsilon_0}{\epsilon},
\frac{c+2\mu}{2\mu} \log \frac{208\eta_0 \hat{L}K\sigma^2}{(c+2\mu)\epsilon} \right\}\right\rceil$ stages, we have $\E[\Delta_{K+1}] \leq \epsilon$. 
The total stochastic first-order oracle call complexity is 
$\widetilde{O}\left( \max\left\{\frac{\ell(L+\rho) \epsilon_0}{\mu\min\{\rho, \mu_y\} \epsilon}, 
\frac{(L + \rho)^2 \sigma^2} {\mu^2 \min\{\rho,\mu_y\} \epsilon} \right\} \right)$. 
\label{thm:ogda_primal}
\end{theorem} 
\vspace*{-0.1in} 
\noindent\textbf{Remark. } This result would imply that it takes $\widetilde{O}\left(\frac{(L+\ell)^2}{\mu^2\mu_y \epsilon}\right)$ stochastic first-order oracle calls to reach an $\epsilon$-level objective gap by setting $\rho=\ell$ (i.e., $f(x,y)$ is $\ell$-weakly convex in terms of $x$ under Assumption~\ref{ass1}).
With the worse-case value of $L=\ell + \frac{\ell^2}{\mu_y}$ (i.e.,  the $\ell$-smoothness of $f(x,y)$ and $\mu_y$-strongly concavity can imply the $\ell + \ell^2/\mu_y$-smoothness of $P(x)$~\citep{nouiehed2019solving}) , the total stochastic first-order oracle call complexity would be no greater than $\widetilde{O}\left(\frac{\ell^4}{\mu^2\mu^3_y \epsilon} \right)$. This is better than the stochastic first-order oracle call complexity of stochastic AGDA method in the order of ${O}\left(\frac{\ell^5}{\mu^2\mu^4_y \epsilon} \right)$~\citep{yang2020global}. 

The above result is achieved by analysis based on a novel Lyapunov function that consists of the primal objective gap $P(x_0^k) - P(x_*)$ and the duality gap of the proximal function $f_k(x, y)$. As a result, we can induce the convergence of duality gap in next section of the original problem with some extra assumptions. 

The convergence results of using SGDA update are similar to the results presented above except that $\sigma^2$ is replaced by the upper bound $B^2$ of stochastic gradients, i.e., there exists $B>0$ such that  $\E[\|\nabla_x f(x, y; \xi)\|^2]\leq B^2$ and $\E[\|\nabla_y f(x, y; \xi)\|^2]\leq B^2$. This is a more restrictive assumption but holds in many practical applications~\citep{JMLR:v15:hazan14a,duchi2011adaptive}. 

Note that the number of iterations in $k$-th stage (i.e. $T_k$) does not depend on the initial solution $(\bar{x}_{k-1}, \bar{y}_{k-1})$.  In each stage, we do not expect to solve the sub-problem accurately, i.e, to some $\epsilon$-accurate level. Instead, each stage just optimizes the sub-problem in order to make the upper bound of Lyaponov function $\Delta_k  = P(x_0^k) - P(x_*) + \frac{8\hat{L}}{53c}{\text{Gap}}_k(x_0^k, y_0^k)$ decrease by a constant factor. And as $k$ grows, $(\bar{x}_{k-1}, \bar{y}_{k-1})$ becomes a better and better solution to the original problem.

Below we highlight the proof sketch. For details of proof, please refer to that of Theorem 27 in the Appendix. What we need from the sub-problem solver is that it can provide a convergence bound as
\begin{align*}
\begin{split}
\E[{\text{Gap}}_k(\bar{x}_k, \bar{y}_k)] 
\leq \frac{C_1}{\eta_k T_k} \E[\|\hat{x}_k(\bar{y}_k) - x_0^k\|^2 + \|\hat{y}_k(\bar{x}_k) - y_0^k\|^2]
+ \eta_k C_2,
\end{split} 
\end{align*}
which has Lemma 24 as an instantiation of Option II: OGDA subroutine. It is notable that the above upper bound depends on the initial solution $(x_0^k, y_0^k)$ of this stage. To achieve this the number of iterations $T_k$ does not need to depend on $(x_0^k, y_0^k)$ and the constants $C_1$ and $C_2$ do not depend on the initial solution and do not depend on the stage index $k$. In Lemma 24, we can see $C_1=1$ and $C_2 = 13\sigma^2$, independent of the initial solution $(\bar{x}_{k-1}, \bar{y}_{k-1})$ and stage index $k$. 

Then by setting $\eta_k = \eta_0  \exp(-(k-1)\frac{2\mu}{c+2\mu})$, $T_k = \left\lceil\frac{212 C_1}{\eta_0 \min\{\rho, \mu_y\}}\exp\left((k-1)\frac{2\mu}{c+2\mu}\right)\right\rceil$, both independent of the initial solution, we can guarantee that 
\begin{equation}
\begin{split}
\E[{\text{Gap}}_k(\bar{x}_k, \bar{y}_k)] 
\leq \frac{\min\{\rho, \mu_y\}}{212} \E[\|\hat{x}_k(\bar{y}_k) - x_0^k\|^2 + \|\hat{y}_k(\bar{x}_k) - y_0^k\|^2]
+ \eta_k C_2, 
\end{split}
\end{equation}
which then by some theoretical deduction can lead to 
\begin{align}
\begin{split}
&\E[\Delta_{k+1}] \leq \frac{c}{c+2\mu} \E[\Delta_{k}] +  \frac{8\eta_k \hat{L} C_2}{c+2\mu}.
\end{split}
\end{align} 
As $\eta_k$ decreases exponentially as $k$ increases, we can then guarantee the convergence of the $\Delta_{k+1}$ and therefore the convergence of the original problem. 

\vspace*{-0.1in} 
\subsection{Improved Rates when $\rho<O(\mu)$} 
Our first improved rate is for almost convex function, whose weak convexity parameter $\rho$ is small enough.  Such a condition has been considered in the literature for improving the convergence of non-convex minimization problem~\citep{yuan2019stagewise,DBLP:conf/icml/ChenXHY19,DBLP:journals/siamjo/LanY19}.  In particular, we consider $\rho$ is smaller than $O(\mu)$.
\begin{theorem}\label{thm:2}
Suppose Assumption \ref{ass1}, \ref{ass2}, \ref{ass5} , \ref{ass3} hold and $0<\rho \leq \frac{\mu}{8}$.
Take $\gamma = \frac{\mu}{4}$.
Define $\Delta_k = 475(P(x_0^k) - P(x_*)) + 57\emph{\text{Gap}}_k(x_0^k, y_0^k)$  and $\epsilon_0=\emph{\text{Gap}}(\bx_0, \by_0)$.
Then we can set $\eta_k = \eta_0  \exp(-\frac{k-1}{16}) \leq \frac{1}{2\sqrt{2}\ell}$, $T_k = \left\lceil\frac{384}{\eta_0 \min\{\mu/8, \mu_y\}}\exp\left(\frac{k-1}{16}\right)\right\rceil$. 
After $K = \bigg\lceil\max\bigg\{16\log \frac{1200\epsilon_0}{\epsilon},$
$16\log \frac{15600\eta_0  K \sigma^2}{\epsilon} \bigg\}\bigg\rceil$ stages, we can have $\E[\Delta_{K+1}] \leq \epsilon$. 
The total stochastic first-order oracle call complexity is 
$\widetilde{O}\left(\frac{ \max\{\ell \epsilon_0, \sigma^2\}}{\min\{\mu, \mu_y\}\epsilon} \right)$. 
\label{thm:ogda_primal_rho<mu} 
\end{theorem}

\subsection{Improved Rates of Using Min-Max AdaGrad}
Similar to the literature of AdaGrad for improving convergence of convex and non-convex minimization problems~\citep{duchi2011adaptive,DBLP:conf/iclr/ChenYYZCY19,chen2018sadagrad}, we can also improve the convergence of NCSC min-max optimization by leveraging Min-Max AdaGrad update. In particular, the dependence on $1/\epsilon$ can be further reduced if the growth rate of the stochastic gradients is slow. 
In particular, we have the following theorem regarding Min-Max AdaGrad. 
\begin{theorem}(Informal)
Suppose Assumption \ref{ass1}, \ref{ass5}, \ref{ass3} hold. Let
$g^k_{1:T_k}$ denote the cumulative matrix of gradients in $k$-th stage. 
Suppose $\|g^k_{1:T_k, i}\|_2 \leq \delta T^{\alpha}_k$ and with $\alpha\in(0,1/2]$. 
Then by setting parameters appropriately, 
PES-AdaGrad has the total stochastic first-order oracle call complexity of 
$\widetilde{O}\left( \left(
\frac{\delta^2 (L+\rho)^2 (d+d')}{\mu^2 \min\{\rho, \mu_y\} \epsilon}\right)^{\frac{1}{2(1-\alpha)}} \right)$ in order to have $\E[\Delta_{K+1}] \leq \epsilon$, { where $\Delta_{k}$ is defined as  in Theorem \ref{thm:1}}. 
\label{thm:ada_primal}
\end{theorem}
\vspace*{-0.1in} 
{\bf Remark: } First let us justify the slow growth condition $\|g^k_{1:T_k, i}\|_2 \leq \delta T_k^{\alpha}$. Supposing the stochastic gradients are bounded, it is clear that $\|g^k_{1:T_k, i}\|_2 \leq  O(T^{1/2}_k)$. But the $\|g^k_{1:T_k, i}\|_2$ can actually grow in a slower order than that because as the algorithm goes on, more and more data would become easy for the model and thus generate small gradients. For example, in deep learning where the models are able to memorize a lot of the training data. We verify this assumption in the Experiment section (Figure \ref{fig:3}) and similar phenomena has been reported in previous research on min-max optimization (see Figure 2 of \cite{liu2020towards}). Indeed it has been observed that overparameterized deep neural networks exhibit interpolation phenomenon, meaning that the model will have zero gradient at every example in the limit \citep{zhang2021understanding}. Nevertheless, it is still nontrivial to prove this condition rigorously and we leave it as an open problem. 
The improvements lies at when the stochastic gradient grows slowly, the sample complexity has a better dependence than $1/\epsilon$, i.e. $O(1/\epsilon^{1/(2(1-\alpha)})\leq O(1/\epsilon)$ when $\alpha\in(0, 1/2)$. 
 
\subsection{Improved Rates of Using Min-Max STORM} 
Our last improved rate is by leveraging the recursive variance reduced stochastic gradient estimator called STORM~\citep{cutkosky2019momentum}. This estimator has been used for non-convex min-max optimization~\citep{DBLP:journals/corr/abs-2008-08170}. However, to the best of our knowledge, an improved rate under a  PL condition for a NCSC optimization problem has not been established before. 
We make an additional assumption about the problem \eqref{eqn:op}.  
\begin{assumption}\label{ass6}
$f(x, y; \xi)$ is $\ell$-smooth in terms of $x$ and $y$  in expectation, i.e., $\E_{\xi}[\|G(z; \xi) - G(z'; \xi)\|^2]\leq \ell \|z - z'\|^2$. 
\end{assumption} 
\begin{theorem}(Informal)
Suppose Assumption \ref{ass1}, \ref{ass2}, \ref{ass3}, \ref{ass6} hold. 
By setting parameters appropriately, 
PES-STORM has the total stochastic first-order oracle call complexity of 
$\widetilde{O}\left( \frac{\ell^2}{\mu \mu^2_y \epsilon }\right)$ in order to have $\E[P(\bx_{K}) - P(x_*)] \leq \epsilon$.
\label{thm:storm_primal_informal} 
\end{theorem}
{\bf Remark:} Compared to the complexity of PES-OGDA as implied by  Theorem~\ref{thm:1}, the complexity of PES-STORM has a better dependence on the PL constant $\mu$, which is usually small in practice.  

\section{Duality Gap Convergence} 
\label{sec:dualitygap}  
In this section, we provide a stronger guarantee by analyzing the duality gap convergence utilizing some extra assumptions. Similar to~\citep{yang2020global}, we make the following assumption.
However, a difference is that we can prove the existence of a saddle point instead of imposing it. 
\begin{assumption}\label{ass4} 
$h_y(x) = f(x, y)$ satisfies $x$-side $\mu_x$-PL condition for any $y\in \mathcal{Y}$, i.e., $\|\nabla_x f(x, y)\|^2\geq 2\mu_x(f(x, y) - \min_x f(x,y))$, for any $x, y\in \mathcal{Y}$. 
\end{assumption} 

Using Theorem \ref{thm:1} and Assumption \ref{ass4}, we have
\begin{corollary} \label{cor:1}
Under the same setting as in Theorem  \ref{thm:ogda_primal}, and suppose Assumption \ref{ass4} holds as well.
To achieve $\E[\emph{\text{Gap}}(\bx_K, \by_K)] \leq \epsilon$,
the total number of stochastic first-order oracle call is \\ $\widetilde{O}\left(\max\left\{\frac{(\rho/\mu_x+1)\ell(L+\rho)\epsilon_0}{\mu \min\{\rho, \mu_y\} \epsilon}, \frac{(\rho/\mu_x+1)(L + \rho)^2 \sigma^2}{\mu^2  \min\{\rho,\mu_y\}  \epsilon} \right\}\right)$. 
\label{cor:ogda_duality} 
\end{corollary} 
\noindent\textbf{Remark. } Note that compared with the stochastic first-order oracle call complexity of the primal objective gap convergence, the stochastic first-order oracle call complexity of the duality gap convergence  is worse by a factor of $\rho/\mu_x+1$. When $f(x, y)$ is $\ell$-weakly convex with $\rho=\ell$, the stochastic first-order oracle call complexity for having $\epsilon$-level duality gap is  $\widetilde{O}\left(\frac{(L+\ell)^2\ell}{\mu^2\mu_x\mu_y \epsilon}\right)$, which reduces to  $\widetilde{O}\left(\frac{\ell^5}{\mu^2\mu_x\mu_y^3 \epsilon}\right)$ for the worst-case value of $L$. This result is better than the stochastic first-order oracle call complexity of stochastic AGDA method in the order of ${O}\left(\frac{\ell^7}{\mu^2\mu_x\mu_y^5 \epsilon} \right)$ that is derived by us based on the result  of~\citep{yang2020global} (c.f. Lemma~\ref{lem:agda} in the Supplement).  In addition, when $f(x,y)$ is $\rho$-weakly convex with $\mu_x<\rho<\mu_y$, the stochastic first-order oracle call complexity of PES-OGDA for having $\epsilon$-level duality gap is $\widetilde{O}\left(\frac{(L+\ell)^2}{\mu^2\mu_x \epsilon}\right)$.
Further, when $\rho<\mu_x$, we can set $\rho=\mu_x$ and then the stochastic first-order oracle call complexity for having $\epsilon$-level duality gap is $\widetilde{O}\left(\frac{(L+\ell)^2}{\mu^2\min\{\mu_x, \mu_y\} \epsilon}\right)$.
  
Using Theorem \ref{thm:2} and Assumption \ref{ass4}, we have 
\begin{corollary}\label{cor:2}
Under the same setting as in Theorem \ref{thm:ogda_primal_rho<mu} and suppose  Assumption \ref{ass4} holds as well. 
To achieve $\E[\emph{\text{Gap}}(\bx_K, \by_K)] \leq \epsilon$,
the total number of stochastic first-order oracle calls is $\widetilde{O}\left(\frac{ \mu\max\{\ell\epsilon_0,  \sigma^2\}}{\mu_x\min\{\mu, \mu_y\}\epsilon}  \right)$. 
\label{cor:ogda_dual_rho<mu} 
\end{corollary}

{\bf Remark: } Compared with results in Theorem~\ref{thm:1} and Corollary~\ref{cor:1}, the sample complexities in Theorem~\ref{thm:2} and Corollary~\ref{cor:2} have better dependence on $\mu, \mu_y$. In addition, by setting $\mu=\mu_x$,
the rate in Corollary~\ref{thm:2} becomes $\widetilde{O}(\frac{ 1}{\min\{\mu_x, \mu_y\}\epsilon} )$, which matches that established optimal rate in~\citep{yan2020sharp} for $\mu_x$-strongly convex and $\mu_y$-strongly concave problems up to a logarithmic factor. But we only require $x$-side $\mu_x$-PL condition  instead of $\mu_x$-strongly convex in terms of $x$. 

\section{Experiments}
In this section, we show some empirical results to verify the effectiveness of the proposed algorithms for deep and non-deep learning tasks. 

\begin{figure*}[t]
\centering
\hspace{-0.1in}
\includegraphics[scale=0.3]{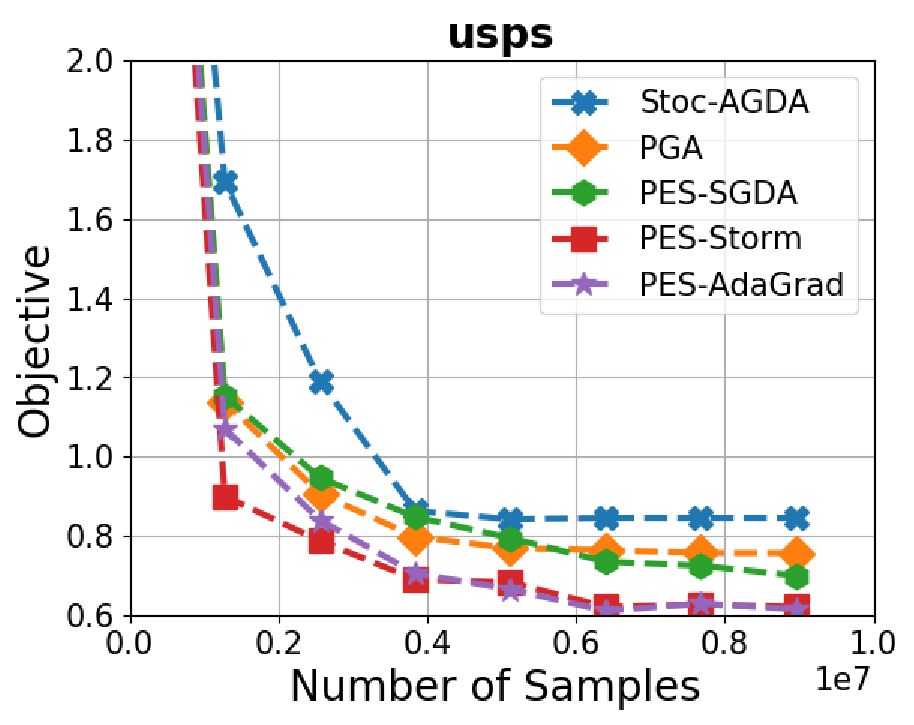}
\includegraphics[scale=0.3]{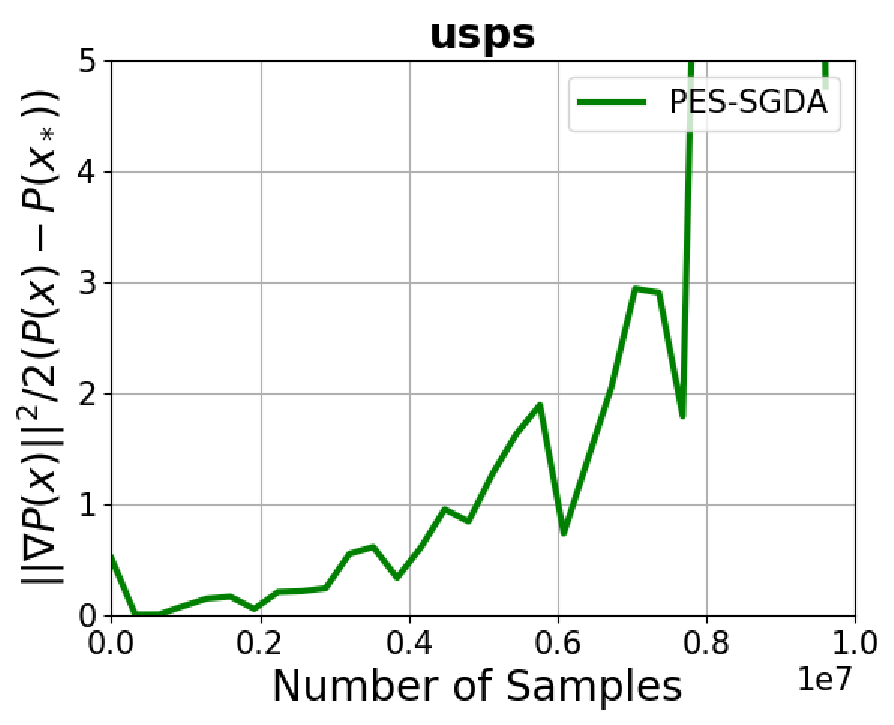} \\
\includegraphics[scale=0.3]{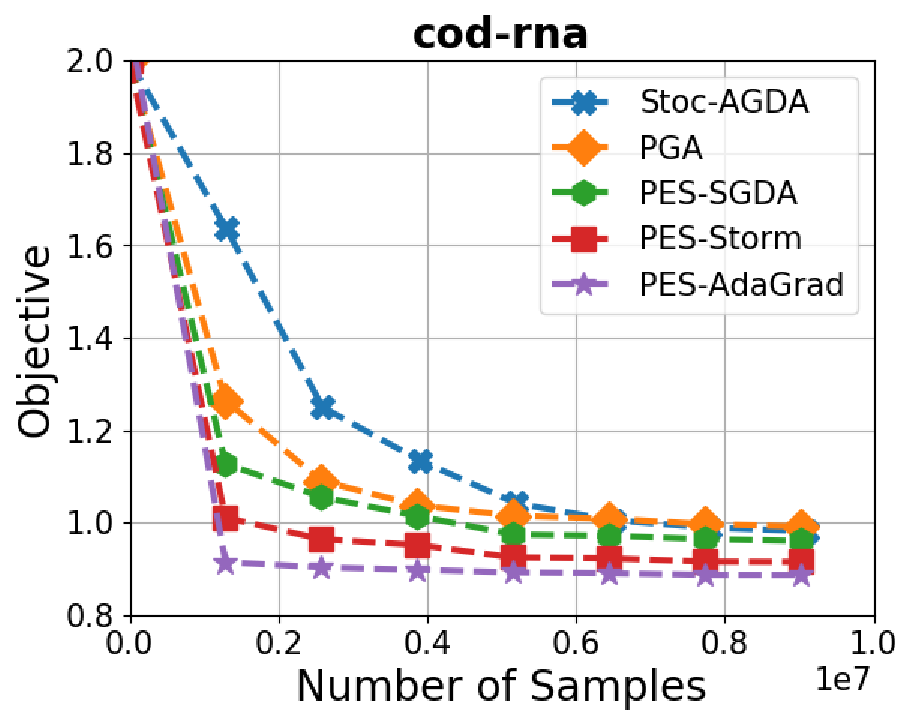} 
\includegraphics[scale=0.3]{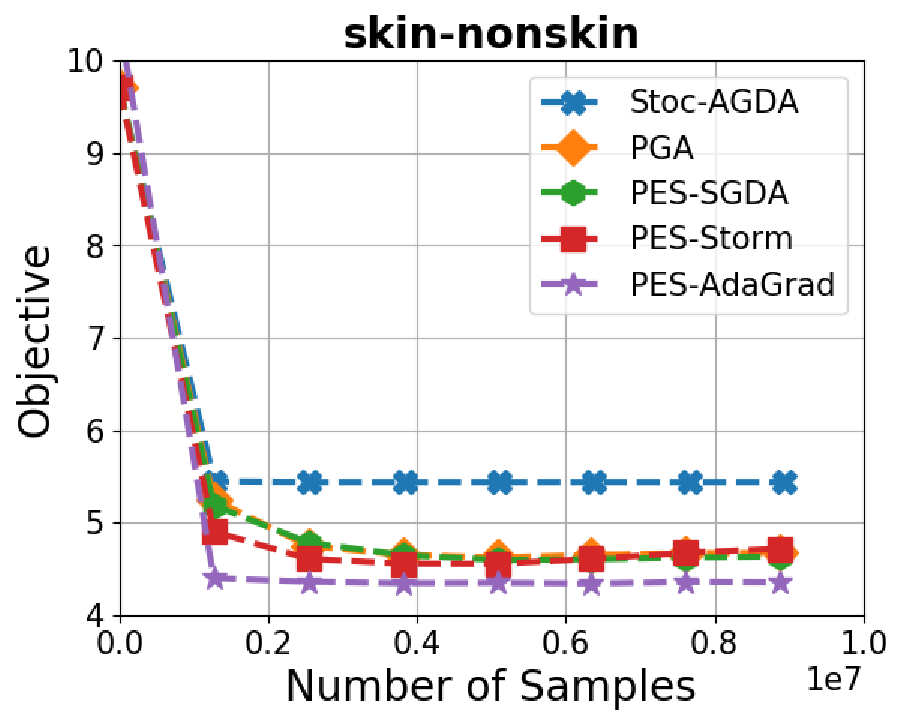} 
\includegraphics[scale=0.3]{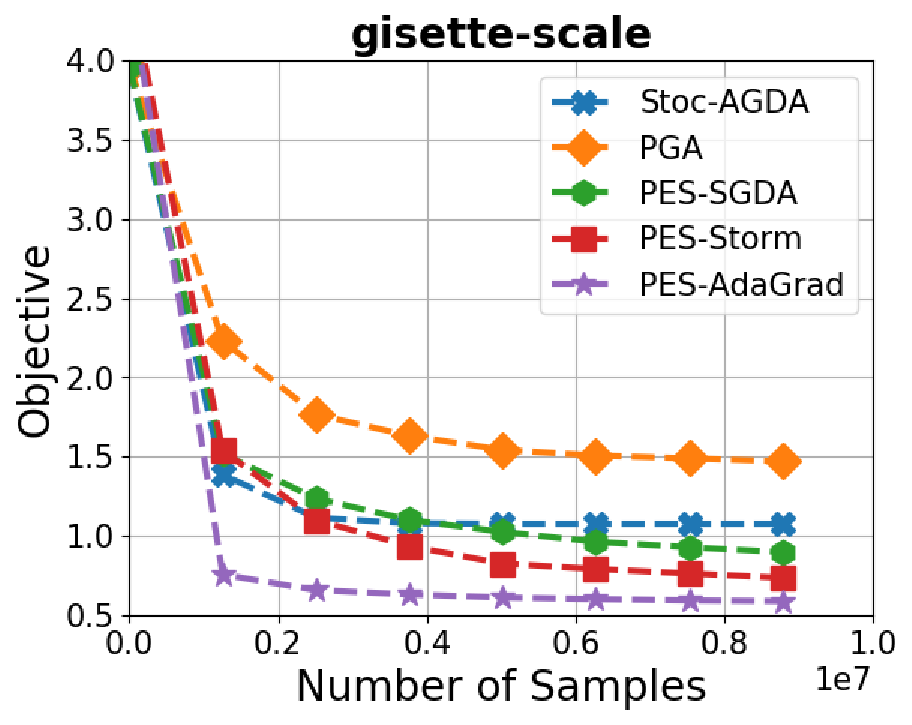} 
\vspace*{-0.1in}
\caption{Results for Non-convex DRO.}
\label{fig:1}
\end{figure*} 

{\bf Non-convex Distributionally Robust Optimization.} This task has been considered in~\citep{rafique2018non}. The problem is formulated as:
\begin{equation} 
\begin{split}
&\min\limits_{x\in\R^d} \max_{y\in\mathcal{S}} \sum_{i=1}^n y_i\phi(\log(1+\exp(-b_i\mathbf{a_i}^T x))) 
- \frac{\theta}{2} \|y - \frac{\mathbf{1}}{n}\|^2
\end{split}
\label{prob:log_truncation}
\end{equation}  
where $(\mathbf{a}_i, b_i)$ denotes feature label pair,  $b_i \in \{-1, 1\}$, $\phi(s)=\log(1+s/2)$ is a non-convex truncation function used to tackle outliers and noisy data, and $\mathcal{S}$ is a simplex. In experiments the simplex constraint is handled by a projection algorithm in \cite{duchi2008efficient}. We conduct experiments on four datasets from LibSVM website \citep{chang2011libsvm}, i.e., gisette-scale, cod-rna,  skin-nonskin and usps. For skin-nonskin, we randomly partition the dataset into training set and testing set of equal size. For other data sets we use the provided training/testing split. For usps, we make the first class to be the positive class and merge the other 9 classes into the negative class. 

We first verify the PL condition of primal problem $P(x)$ of (\ref{prob:log_truncation}) empirically. We plot $\|\nabla P(x)\|^2/2(P(x) - P(x_*))$ in the second figure of the Figure \ref{fig:1}. 

We compare three variants of our method (PES-SGDA, PES-STORM, PES-AdaGrad) with two baselines Stoc-AGDA \citep{yang2020global},  PGA (algorithm 1~\citep{rafique2018non}).
 For all algorithms, we set $\theta=10$. For Stoc-AGDA, the step sizes for $x$ and $y$ are set to be $\frac{\tau_1}{\lambda+t}$ and $\frac{\tau_1}{\lambda+t}$, respectively. $\tau_1$ and $\tau_2$ are tuned in $[1 \sim 1e3]$. $\gamma$ is tuned in $[1\sim 1e4]$. For PES-SGDA, PES-STORM, and PES-AdaGrad, we set $T_k = T_02^k$ and $\eta_k=\eta_0/2^k$, where $T_0$ and $\eta_0$ are tuned in $[500\sim 5000]$, $[0.1, 0.05, 0.01, 0.001]$. $\gamma$ is tuned in $[1\sim 2000]$.
The results are plotted in Figure~\ref{fig:1}. We can see that the proposed algorithms PES-SGDA, PES-STORM and PES-AdaGrad converge faster than the baselines in most cases. PES-STORM and PES-AdaGrad perform better than PES-SGDA on this task, which shows the potential to improve the performance by using STORM type variance techniques or adaptive methods when a task satisfies corresponding assumptions.

{\bf Deep AUC maximization.} This task is  similar to that considered in~\citep{liu2019stochastic}. Deep AUC maximization with a square surrogate loss function is formulated as a NCSC min-max problem which has been introduced in the Section \ref{sec:app}. 
We compare our algorithms, PES-SGDA (Option I), PES-OGDA (Option II), PES-AdaGrad (Option III), with five baseline methods, including stochastic gradient method (SGD) for solving a standard minimization formulation with cross-entropy loss, Stoc-AGDA \citep{yang2020global},  PGA (algorithm 1~\citep{rafique2018non}),  PPD-SG and PPD-AdaGrad~\citep{liu2019stochastic} for solving the same AUC maximization problem. 
 We learn ResNet-20 \citep{he2016identity} with an ELU activation function. 

For the parameter settings, we use a common stage-wise stepsize for SGD, i.e., the initial stepsize is divided by 10 at 40K, 60K of stochastic first-order oracle calls.  For PPD-SG and PPD-AdaGrad, we follow the instructions in their works, i.e., $T_{k}=T_{0} 3^{k}, \eta_{k}=\eta_{0} / 3^{k}$ and $T_0, \eta_0$ and $\gamma$ are tuned in $[500 \sim 2000]$, $[0.1, 0.05, 0.01, 0.001]$, and $[100 \sim 2000]$, respectively. For Stoc-AGDA, the stepsize strategy follows $\frac{\tau_1}{\lambda+t} $, $\frac{\tau_2}{\lambda+t}$ for the dual and primal variables, respectively, where $\tau_1 \ll \tau_2$. The initial values $\tau_1, \tau_2$, $\lambda$ are tuned in $[1, 5, 10, 15]$, $[5, 10, 15, 20]$, and $[1e3, 1e4]$, respectively. For our methods, we adopt the same strategy as PPD-SG and PPD-AdaGrad to tune the  parameters.   

We compare on three benchmark datasets: Cat\&Dog (C2) \citep{DBLP:conf/ccs/ElsonDHS07}, CIFAR10 (C10), CIFAR100 (C100) \citep{krizhevsky2009learning} which have 2, 10, 100 classes, respectively. To fit our task, we convert them into imbalanced datasets following the instructions in \citep{liu2019stochastic}. We firstly construct the binary dataset by splitting the original dataset into two portions with equal size (50\% positive: 50\% negative) and then we randomly remove $90\%, 80\%, 60\%$ data from negative samples on training data, which generate the imbalanced datasets with a positive:negative ratio of 91/9, 83/17, 71/29, respectively. We keep the testing data unchanged. We set the batch size to 128 for all datasets.

\begin{figure*}[t]
\centering
\hspace{-0.1in}
\includegraphics[scale=0.225]{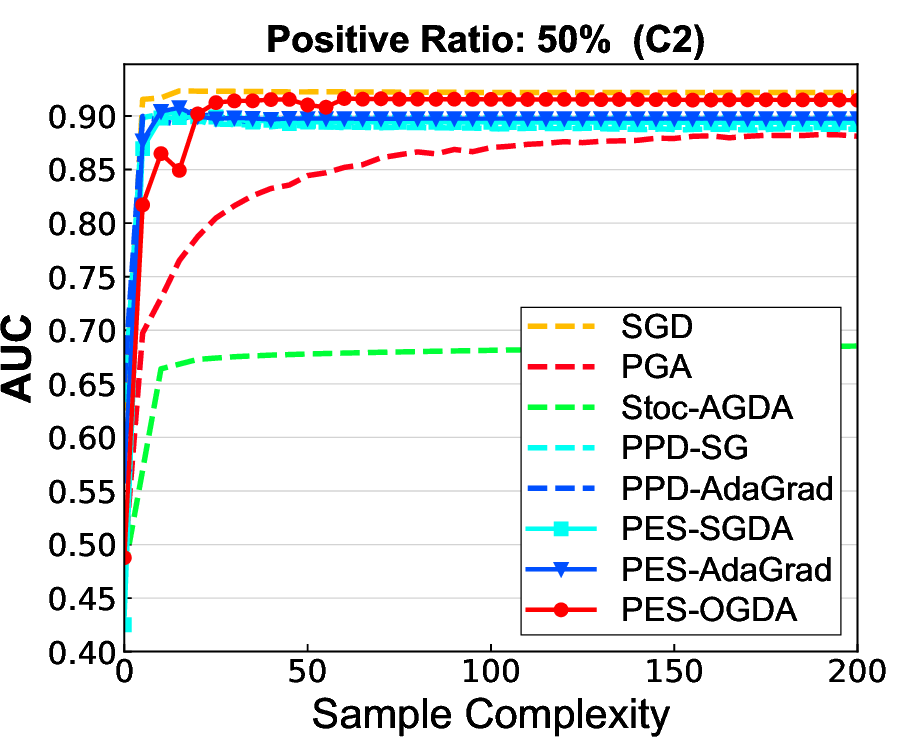}
\includegraphics[scale=0.225]{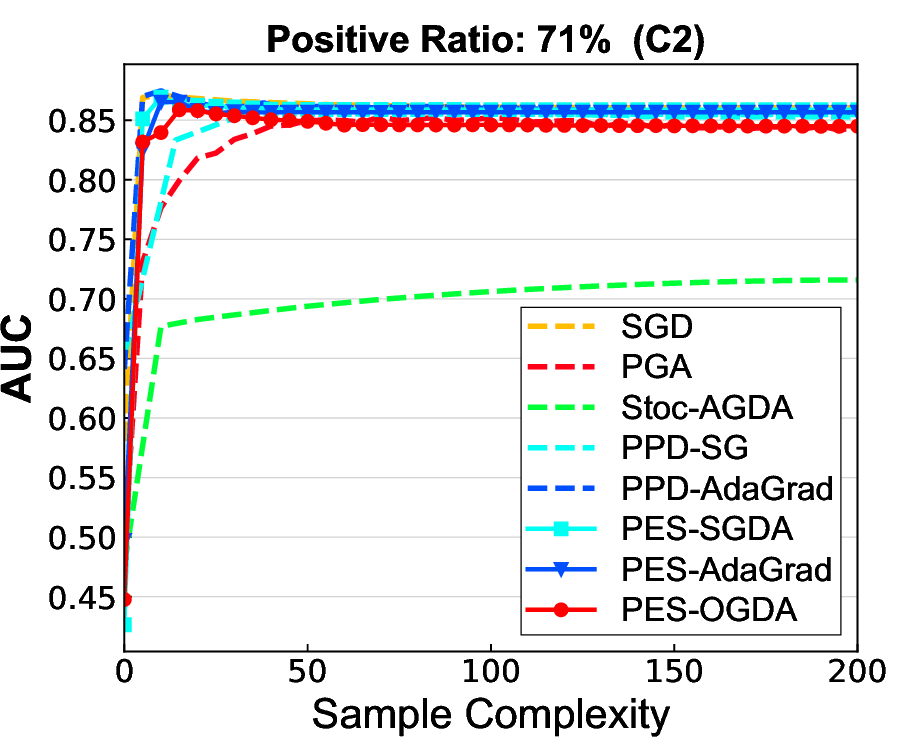}
\includegraphics[scale=0.225]{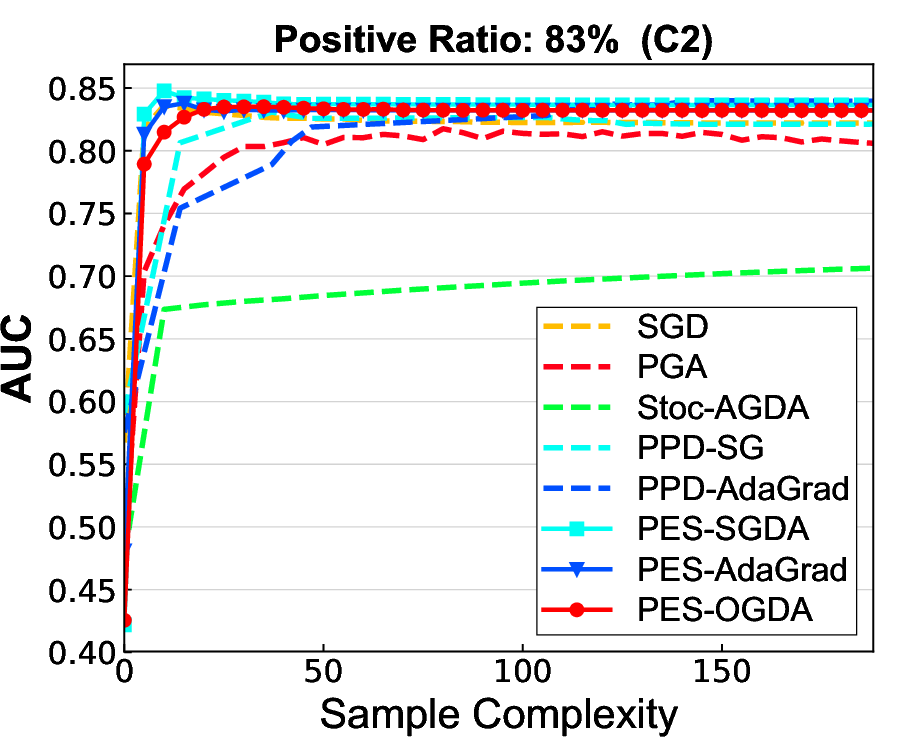}
\includegraphics[scale=0.225]{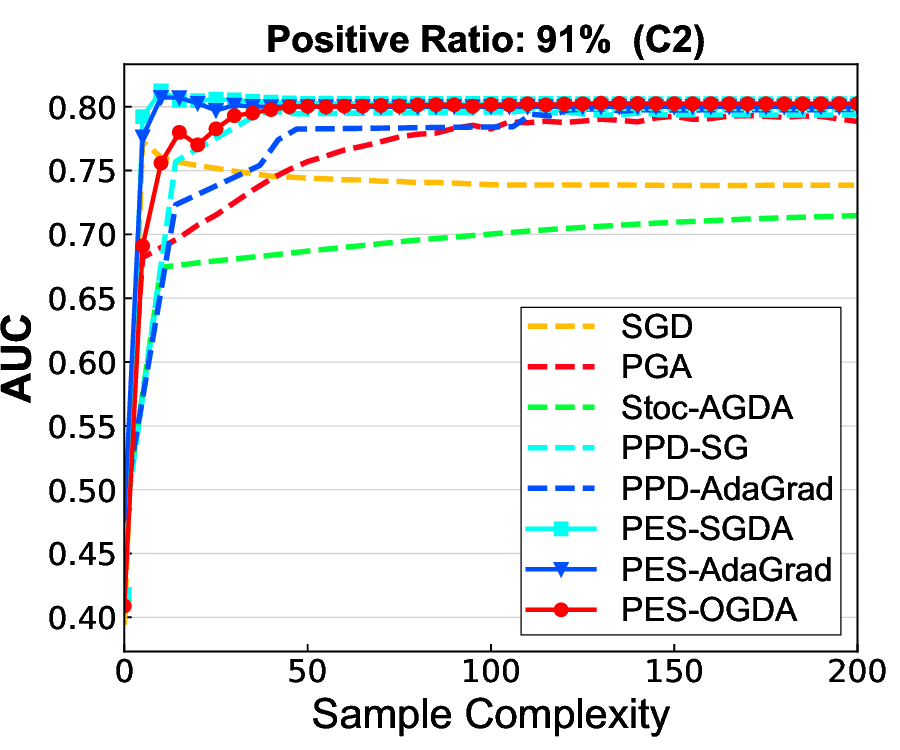}

\hspace{-0.1in}
\includegraphics[scale=0.225]{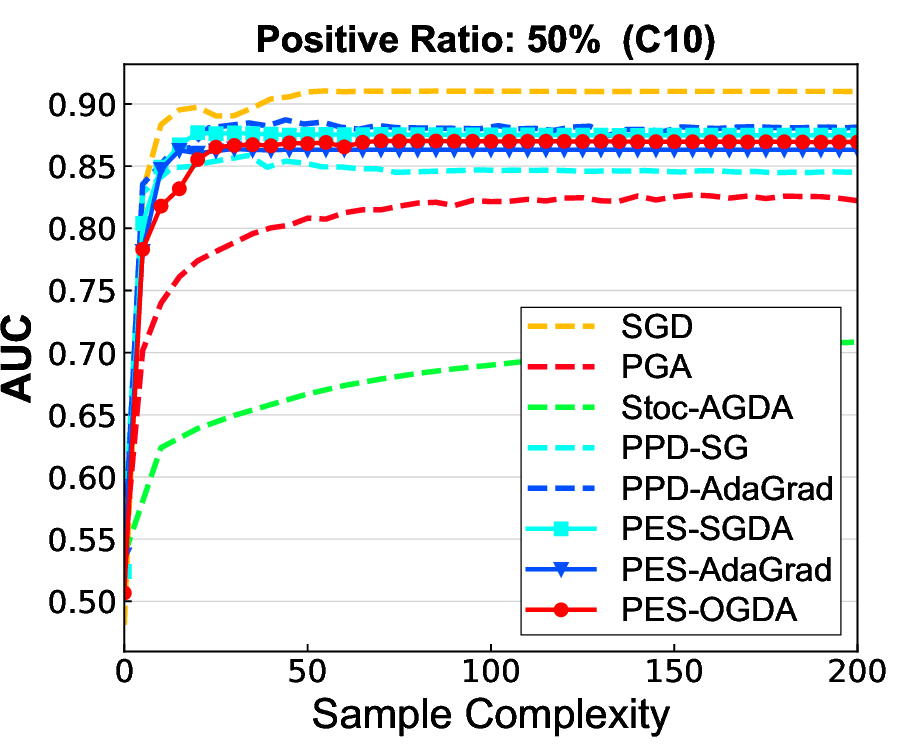}
\includegraphics[scale=0.225]{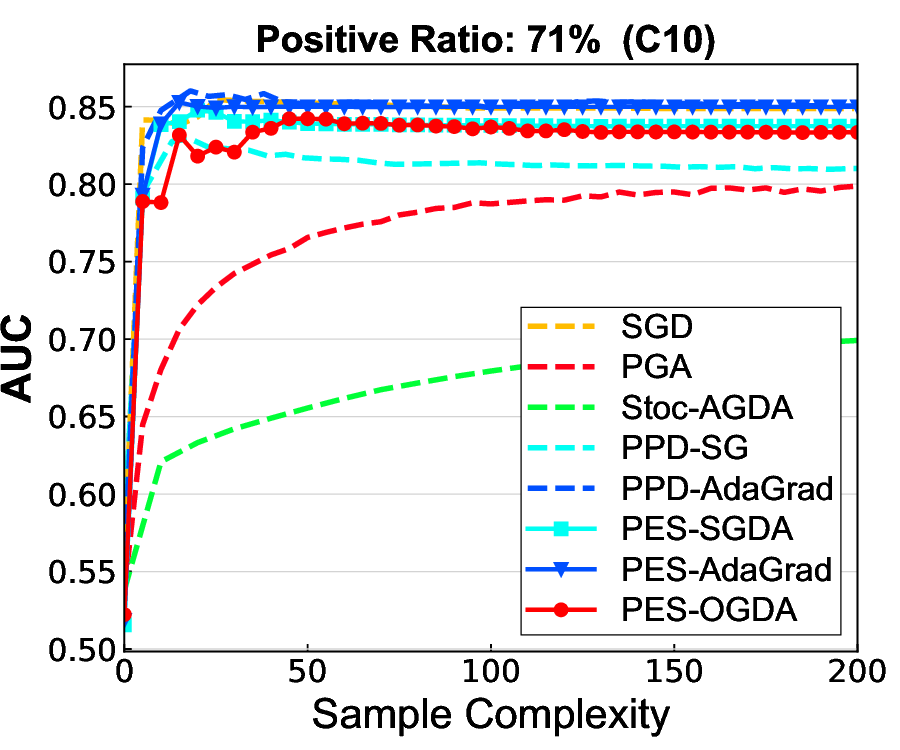}
\includegraphics[scale=0.225]{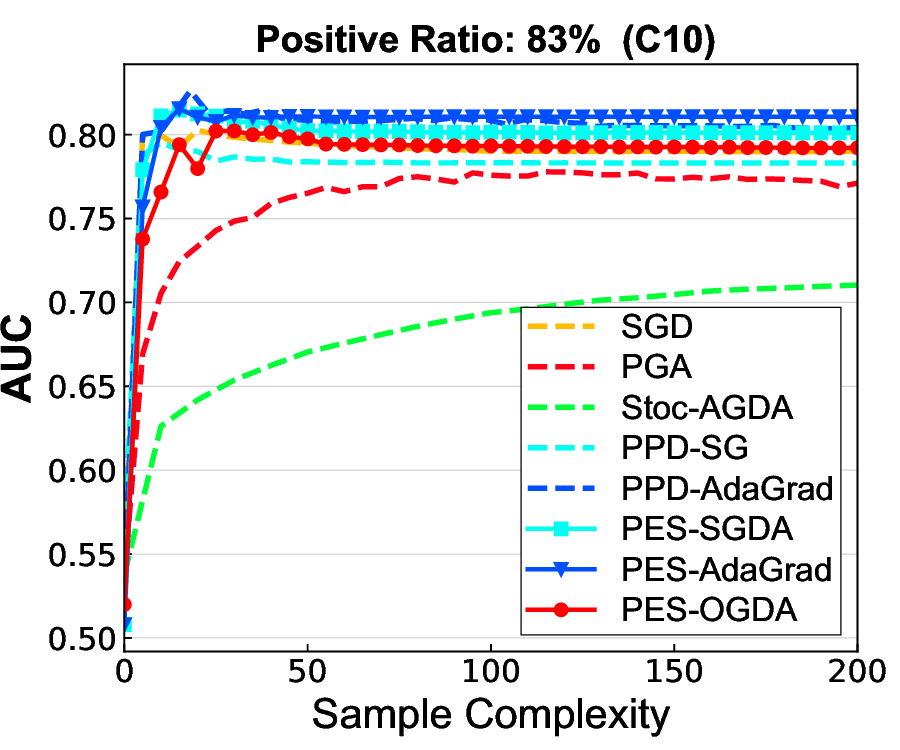}
\includegraphics[scale=0.225]{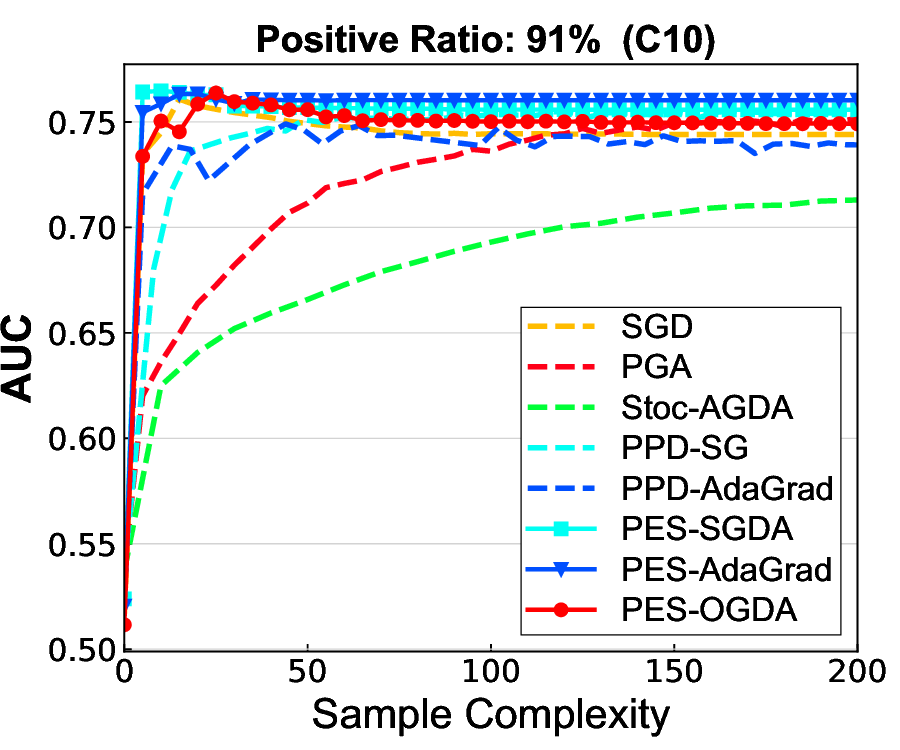}

\hspace{-0.1in}
\includegraphics[scale=0.225]{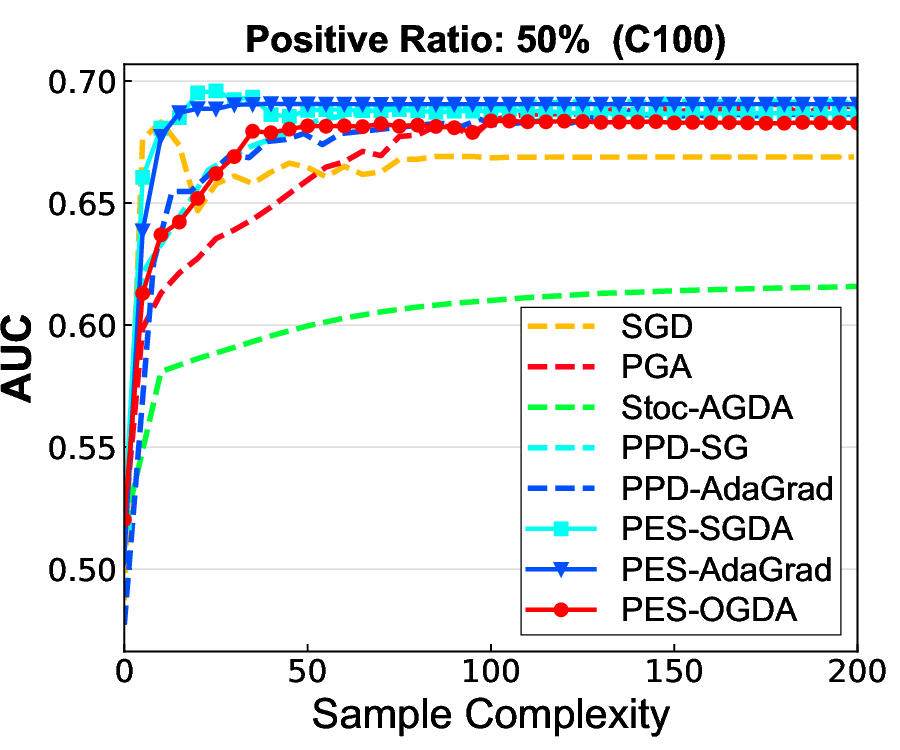}
\includegraphics[scale=0.225]{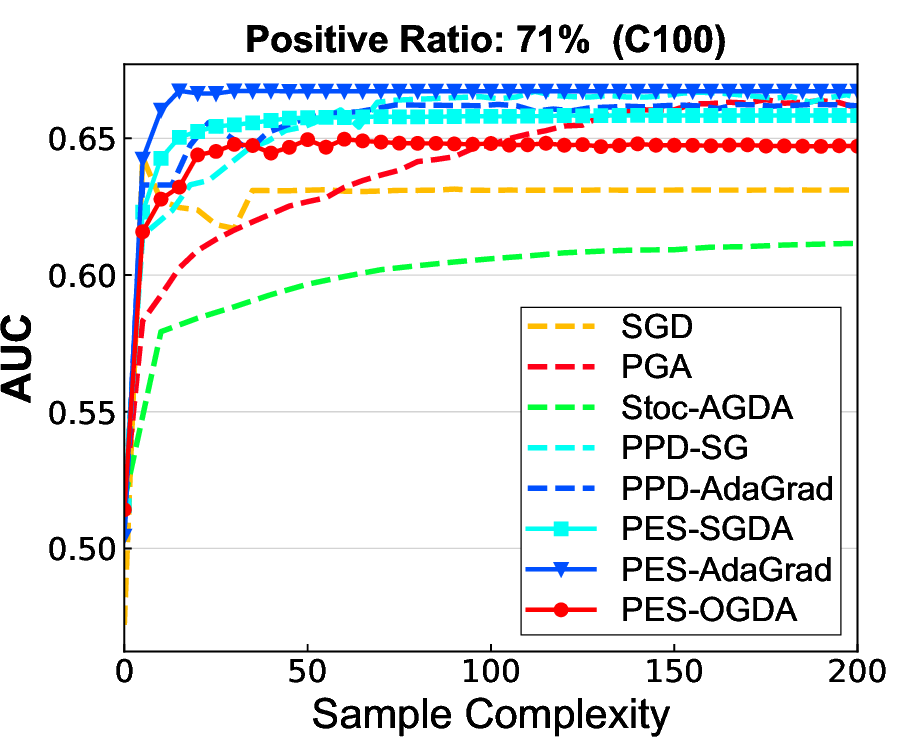}
\includegraphics[scale=0.225]{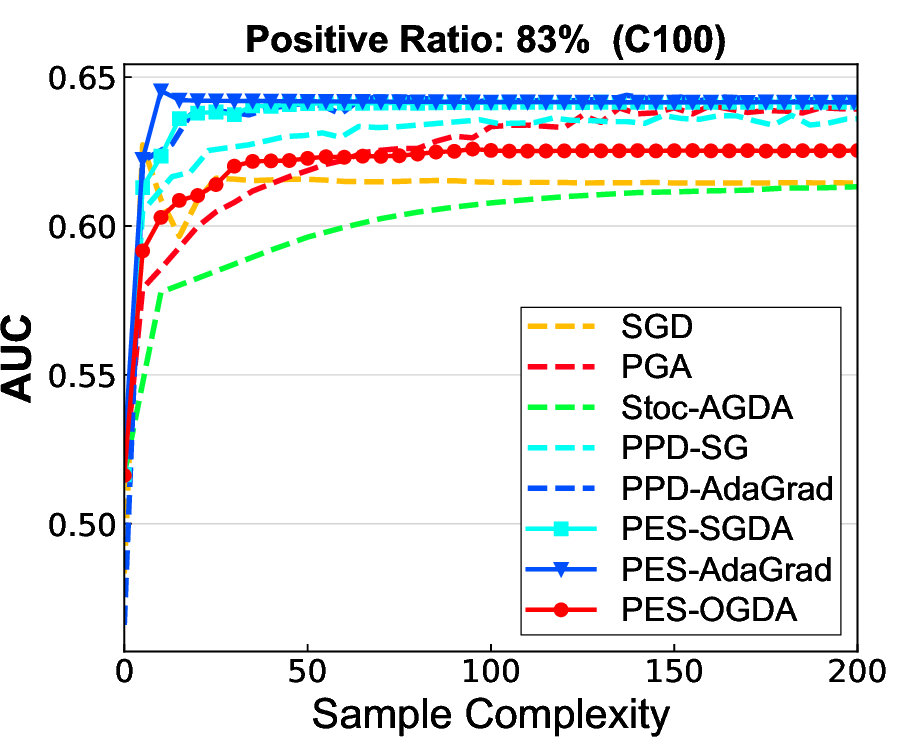}
\includegraphics[scale=0.225]{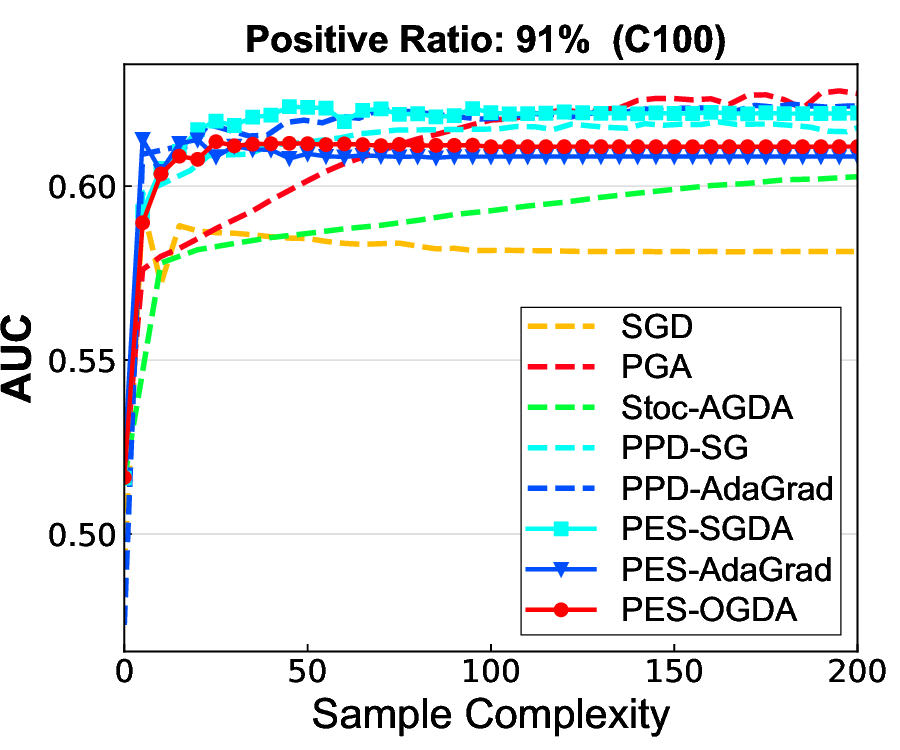} 
\vspace*{-0.1in}
\caption{Comparison of testing AUC on  Cat\&Dog, CIFAR10, CIFAR100.}
\label{fig:0}
\vspace*{-0.15in}
\end{figure*} 

\begin{figure*}[t]
\centering
\hspace{-0.1in}
\includegraphics[scale=0.35]{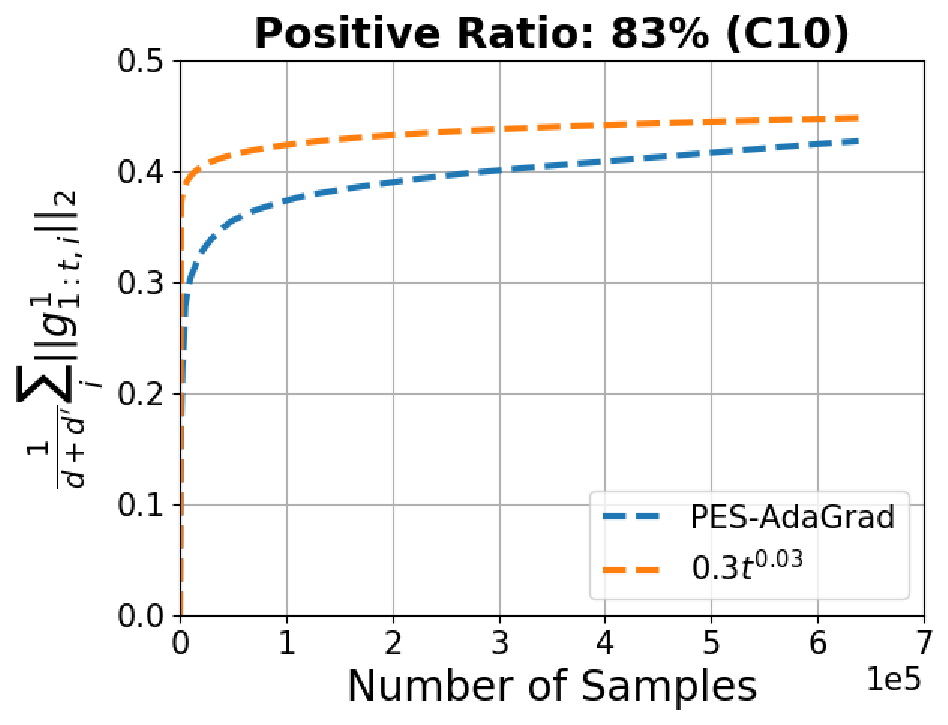} 
\includegraphics[scale=0.35]{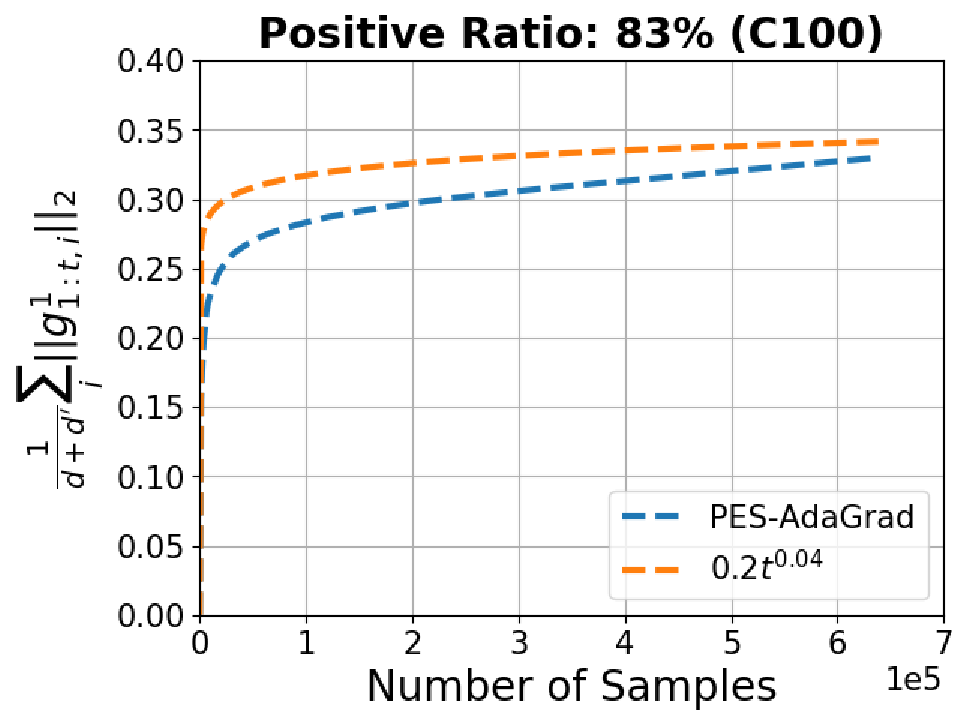} 
\caption{Verification of the Slow Growth Condition}  
\label{fig:3} 
\vspace*{-0.15in} 
\end{figure*}

The testing AUC curve of all algorithms are reported in Figure \ref{fig:0}, where the sample complexity indicates the number of samples used in the training up to 80K of stochastic first-order oracle calls. From the results, we can see that SGD works better (or similar to) than AUC-based methods on the balanced data (50\%). However, PES-SGDA and PES-AdaGrad generally outperform  SGD when the data is imbalanced,  and outperforms PGA and Stoc-AGDA in almost all cases. In addition, the proposed methods performs similarly sometimes better than PPD-SG/PPD-AdaGrad except on C100 (91\% positive ratio).
This is not surprising since PPD-SG/PPD-AdaGrad are designed for AUC maximization under the same PL condition by leveraging its structure and extra data samples for computing a restarted dual solution. In contrast, our algorithms directly use averaged dual solution for restarting.  When the positive ratio is $91\%$ on C100, we observe that PPD-AdaGrad performs better than our algorithms,  showing that using the extra data samples may help in the extreme imbalanced cases. We also observe that the Stoc-AGDA performs worst in all cases with $O(\frac{1}{t})$ stepsize. For our methods, PES-SGDA and PES-AdaGrad perform generally better than PES-OGDA. In Figure \ref{fig:3}, we verify the slow growth condition, i.e. $\|g^k_{1,T_k, i}\| \leq \delta T_k^{\alpha}$  used in the analysis of AdaGrad based algorithms, by plotting the $\frac{1}{d+d'}\sum\limits_{i} \|g^1_{1:t,i}\|_2$ versus the sample complexity. We can seen that the growth of the aggregate of stochastic gradients is slower than the order of $O(\sqrt{T})$.

\section{Conclusion}
In this paper, we have presented generic stochastic algorithms for solving non-convex and strongly concave min-max optimization problems. 
We established convergence for both the objective gap and the duality gap under PL conditions of the objective function for different stochastic updates.  The experiments on deep and non-deep learning tasks have demonstrated the effectiveness of our methods. 

\acks{The feedback provided by the anonymous reviewers is greatly valued. We also wish to acknowledge the support received from the NSF Career Award \#1844403, NSF Program \#2110545, and NSF-Amazon Joint Program \#2147253 for this work.}

\newpage
\bibliography{reference} 
\newpage
\appendix

\section{Convergence of Duality Gap by Stoc-AGDA Algorithm} 
\vspace*{-0.1in}  
To compare our algorithm with Stoc-AGDA in terms of convergence of duality gap, we derive Lemma  \ref{lem:agda} based on Theorem 3.3 of \citep{yang2020global}. 
We first present an auxiliary lemma which is an extension of the Danskin's theorem. 
\begin{lemma}[Corollary of Theorem 1 of \citep{bernhard1995theorem}]
In the \\
min-max problem, when $f(x,y)$ is strong concave in $y$ for any $x$ then the gradient of the function $P(x) =\max_{y\in \mathcal{Y}}$ is $\nabla P(x)=\nabla_x f(x,\hat{y}(x))$ where $\hat{y}(x) = \arg\max_{y\in \mathcal{Y}} f(x, y)$.
\label{lem:primal_gradient_danskin} 
\end{lemma}

Then the convergence of duality gap by Stoc-AGDA algorithm is given in next lemma. 
\begin{lemma}\label{lem:agda}
Supposes Assumption \ref{ass1}, \ref{ass2}, \ref{ass3} and \ref{ass4} hold.  
Stoc-AGDA would reach a $\epsilon$-duality gap by a stochastic first-order oracle call complexity of  $O\left(\frac{\ell^7}{\mu^2\mu_x\mu_y^5\epsilon}\right)$. 
\label{lem:implied_by_yang2020global}
\end{lemma}   
\begin{proof}
\citet{yang2020global} defines the measure as following potential function, 
\begin{align} 
\begin{split} 
P_t = E[P(x_t) - P(x_*)] + \frac{1}{10} E[P(x_t) - f(x_t, y_t)].
\end{split}
\end{align}
By Theorem 3.3 of \citep{yang2020global}, in Stoc-AGDA, $P_t \leq \hat{\epsilon}$ after $O\left(\frac{\ell^5}{\mu^2 \mu_y^4  \hat{\epsilon}}\right)$ stochastic first-order oracle calls.
It directly follows that the objective gap will be less than $\hat{\epsilon}$ after $O\left(\frac{\ell^5}{\mu^2 \mu_y^4 \hat{\epsilon}}\right)$ stochastic first-order oracle calls,  i.e.,
\begin{align} 
\begin{split} 
P(x_t) - P(x_*) \leq P_t \leq \hat{\epsilon}.
\end{split} 
\end{align} 

Besides, after $O\left(\frac{\ell^5}{\mu^2 \mu_y^4 \hat{\epsilon}}\right)$ stochastic first-order oracle calls, we also have 
\begin{align}
f(x_t, \hat{y}(x_t)) - f(x_t, y_t) = P(x_t) - f(x_t, y_t) \leq 10\hat{\epsilon},
\end{align}
where the equality holds by the Lemma \ref{lem:primal_gradient_danskin}. 
By the $\mu_y$-strong concavity of $f(x, \cdot)$, we have
\begin{align} 
\begin{split}
\|y_t - \hat{y}(x_t)\|^2 \leq \frac{f(x_t, \hat{y}(x_t)) - f(x_t, y_t)}{2\mu_y} \leq \frac{5\hat{\epsilon}}{\mu_y},
\end{split}
\end{align}
and 
\begin{align}
\begin{split}
&\|\hat{y}(x_t) - y_*\|^2 
\leq \frac{f(x_t, \hat{y}(x_t)) - f(x_t, y_*)}{2\mu_y} \\
&\leq \frac{f(x_t, \hat{y}(x_t)) - f(x_*, y_*) + f(x_*, y_*) - f(x_t, y_*)}{2\mu_y} \\
&\leq \frac{f(x_t, \hat{y}(x_t)) - f(x_*, y_*)}{2\mu_y} 
= \frac{P(x_t) - P(x_*)}{2\mu_y} \leq \frac{\hat{\epsilon}}{2\mu_y}.
\end{split}
\end{align}

Thus,
\begin{align}
\begin{split}
\|y_t - y_*\|^2 \overset{(a)}{\leq} 2\|y_t - \hat{y}(x_t)\|^2 + 2\|\hat{y}(x_t) - y_*\|^2 \leq \frac{11\hat{\epsilon}}{\mu_y},
\end{split}
\end{align}
where $(a)$ holds since $\|\mathbf{a} - \mathbf{b}\|^2 = \|\mathbf{a} - \mathbf{c} + \mathbf{c} - \mathbf{b}\|^2 \leq 2\|\mathbf{a} - \mathbf{c}\|^2 + 2\|\mathbf{c} - \mathbf{b}\|^2$.
Since $f(\cdot, \cdot)$ is $\ell$-smooth and $f(\cdot, y)$ satisfies $\mu_x$-PL condition for any $y$, we know $D(y) = \min\limits_{x'} f(x', y)$ is smooth with coefficient $ \ell + \frac{\ell^2}{\mu_x} \leq \frac{2\ell^2}{\mu_x}$
\citep{nouiehed2019solving,yang2020global}.
Thus, 
\begin{align}
\begin{split}
f(x_*, y_*) - f(\hat{x}(y_t), y_t)  = D(y_*) - D(y_t) 
\leq \frac{2 \ell^2}{2\mu_x} \|y_t - y_*\|^2 \leq \frac{11\ell^2\hat{\epsilon}}{\mu_x\mu_y}, 
\end{split} 
\end{align} 
where the first equality holds by Lemma A.5 of \citep{nouiehed2019solving}. 

Then we know the duality gap is 
\begin{align}
\begin{split}
f(x_t, \hat{y}(x_t)) - f(\hat{x}(y_t), y_t) 
&= f(x_t, \hat{y}(x_t)) - f(x_*, y_*) + f(x_*, y_*) - f(\hat{x}(y_t), y_t)  \\
&\leq \hat{\epsilon} + \frac{11\ell^2\hat{\epsilon}}{\mu_x \mu_y}.
\end{split} 
\end{align}
To make the duality gap less than $\epsilon$, we need $\hat{\epsilon} \leq O\left(\frac{\mu_x \mu_y \epsilon}{\ell^2}\right)$. 
Therefore, it takes $O\left(\frac{\ell^7}{\mu^2\mu_x\mu_y^5\epsilon}\right)$ stochastic first-order oracle calls to have a $\epsilon$-duality gap for the Algorithm Stoc-AGDA that has been proposed in \citep{yang2020global}. 
\end{proof} 
\section{Convergence Analysis of PES-SGDA} 
We present the convergence rate of primal gap and duality gap if SGDA update is used in Algorithm \ref{alg:subroutine}.
Since the proof is similar to the version with Option II: OGDA
as update, we include the proof in later sections together with the version using OGDA update.

\begin{theorem}
Consider Algorithm \ref{alg:main} that uses Option I: SGDA update in subroutine Algorithm \ref{alg:subroutine}.
Suppose Assumption \ref{ass1}, \ref{ass5}, \ref{ass3} hold. 
Assume $\E\|\nabla_x f(x, y; \xi)\|^2\leq B^2$ and \\
$\E\|\nabla_y f(x, y; \xi)\|^2\leq B^2$.
Take $\gamma = 2\rho$ and
denote $\hat{L} = L + 2\rho$ and $c = 4\rho+\frac{248}{53} \hat{L} \in O(L + \rho)$.
Define $\Delta_k = P(x_0^k) - P(x_*) + \frac{8\hat{L}}{53c}\emph{\text{Gap}}_k(x_0^k, y_0^k)$  and $\epsilon_0=\emph{\text{Gap}}(\bx_0, \by_0)$.
Then we can set $\eta_k = \eta_0  \exp(-(k-1)\frac{2\mu}{c+2\mu})\leq \frac{1}{\rho}$, $T_k = \left\lceil\frac{212 C_1}{\eta_0 \min\{\rho, \mu_y\}}\exp\left((k-1)\frac{2\mu}{c+2\mu}\right)\right\rceil$. 
After $K = \left\lceil\max\left\{\frac{c+2\mu}{2\mu}\log \frac{4\epsilon_0}{\epsilon},
\frac{c+2\mu}{2\mu} \log \frac{80 \eta_0 \hat{L}KB_2}{(c+2\mu)\epsilon} \right\}\right\rceil$ stages, we can have $\Delta_{K+1} \leq \epsilon$. 
The total stochastic first-order oracle call complexity is 
$\widetilde{O}\left(
\frac{(L + \rho)^2 B^2} {\mu^2 \min\{\rho,\mu_y\} \epsilon}\right)$. 
\label{thm:gda_primal}
\end{theorem}
\textbf{Remark.} The bounded stochastic gradient assumption i.e., $\E[\|\nabla_x f(x, y; \xi)\|^2]\leq B^2$ and $\E[\|\nabla_y f(x, y; \xi)\|^2]\leq B^2$ is only used for the analysis of our algorithm employing the SGDA update (Option I), and it is not used for other updates. 
It is notable that in min-max optimization it is an open question to get rid of the bounded stochastic gradient assumption for the vanilla SGDA updates in order to establish convergence bound for the duality gap. To the best of our knowledge, in the existing works over the gap convergence of stochastic min-max optimization that can achieve state-of-the-art complexity, they either use this bounded stochastic gradient assumption \citet{nemirovski2009robust,yan2020sharp}, or use some extra steps other than simple SGDA \citep{juditsky2011solving,zhao2022accelerated,hsieh2019convergence,zhao2019optimal,yang2020global}.

\begin{corollary}
Under the same setting as in Theorem \ref{thm:gda_primal} and suppose Assumption \ref{ass4} holds aw well.
To reach an $\epsilon$-duality gap, it takes a total stochastic first-order oracle call complexity of $\widetilde{O}\left(
\frac{(L + \rho)^2(\rho/\mu_x+1) B^2} {\mu^2 \min\{\rho,\mu_y\} \epsilon}\right)$. 
\label{cor:gda_duality}
\end{corollary}

\begin{theorem}
Suppose Assumption \ref{ass1}, \ref{ass4}, \ref{ass5} hold and $0<\rho \leq \frac{\mu}{8}$.
Assume $\E\|\nabla_x f(x, y; \xi)\|^2\leq B^2$ and $\E\|\nabla_y f(x, y; \xi)\|^2\leq B^2$.
Take $\gamma = 2\rho$.
Define $\Delta_k = 475(P(x_0^k) - P(x_*)) + 57\emph{\text{Gap}}_k(x_0^k, y_0^k)$  and $\epsilon_0=\emph{\text{Gap}}(\bx_0, \by_0)$.
Then we can set $\eta_k = \eta_0  \exp(-\frac{k-1}{16}) \leq \frac{1}{\rho}$, $T_k = \left\lceil\frac{768}{\eta_0 \min\{\mu/8, \mu_y\}}\exp\left(\frac{k-1}{16}\right)\right\rceil$. 
After $K = \left\lceil\max\left\{16\log \frac{1200\epsilon_0}{\epsilon}, 
16\log \frac{6000\eta_0  K B^2}{\epsilon} \right\}\right\rceil$ stages, we can have $\Delta_{K+1} \leq \epsilon$.  
The total stochastic first-order oracle call complexity is 
$\widetilde{O}\left( \frac{B^2}{\min\{\mu, \mu_y\} \epsilon}\right)$. 
\label{thm:gda_primal_rho<mu}
\end{theorem}

\begin{corollary} 
Under the same setting as in Theorem \ref{thm:gda_primal_rho<mu} and suppose Assumption \ref{ass4} holds as well. 
To reach an $\epsilon$-duality gap, it takes total stochastic first-order oracle call complexity of 
$\widetilde{O}\left(\frac{(\mu/\mu_x+1) B^2}{\min\{\mu, \mu_y\} \epsilon}\right)$.  
\label{cor:gda_duality_rho<mu} 
\end{corollary} 

\section{One Stage Analysis of PES-OGDA}
We need the following lemmas from \citep{nemirovski2004prox}.
\begin{lemma}[Lemma 3.1 of \citep{nemirovski2004prox}]
\label{nemi2004_lemma3.1}
For $z_0 \in \mathcal{Z}$, let $w_1 = \Pi_{z_0}(\zeta_1)$, $w_2 = \Pi_{z_0}(\zeta_2)$.
For any $z \in \mathcal{Z}$, 
\begin{align}
\begin{split}
\langle \zeta_2, w_1 - z \rangle \leq 
\frac{1}{2} \|z-z_0\|^2 - \frac{1}{2}\|w_2 -z\|^2 
- \frac{1}{2} \|w_1 - z_0\|^2 - \frac{1}{2} \|w_1 - w_2\|^2 + \|\zeta_1 - \zeta_2\|^2. \hfill
\end{split} 
\end{align} 
\end{lemma}

\begin{lemma}[Corollary 2 of \citep{juditsky2011solving}]
Let $\zeta_1, \zeta_2, ...$ be a sequence, we define a corresponding sequence $\{v_t \in \mathcal{Z} \}_{t=0}^{T}$ as 
\begin{align}
\begin{split} 
v_t = \Pi_{v_{t-1}} (\zeta_{t}), v_0 \in \mathcal{Z},
\end{split}
\end{align}
we have for any $u \in \mathcal{Z}$,
\begin{align}
\begin{split}
\sum\limits_{t=1}^{T} \left\langle \zeta_t, v_{t-1}-u  \right\rangle \leq \frac{1}{2}\|v_0 - u\|^2 +  \frac{1}{2}\sum\limits_{t=1}^{T}\|\zeta_t\|^2. 
\end{split}
\end{align}
\label{nemirovski:cor2}
\end{lemma}

Next we present the lemma that guarantees the converge of one call of Algorithm \ref{alg:subroutine} with Option II: OGDA update.
\begin{lemma}
\label{lem:ogda_sub}  
Suppose $f(x, y)$ is convex-concave and Assumption \ref{ass2} holds.
By running Algorithm \ref{alg:subroutine} with OGDA update and input $(f, x_0, y_0, \eta\leq \frac{1}{4\sqrt{3}\ell}, T)$, we have 
\begin{align} 
\begin{split} 
\E[f(\bar{x}, \hat{y}(\bar{x})) - f(\hat{x}(\bar{y}), \bar{y})]
&{\leq} \frac{1}{\eta T}\E(\|\hat{x}(\bar{y}) - x_{0}\|^2 + \|\hat{y}(\bar{x}) - y_{0}\|^2)  
+ 13 \eta\sigma^2.
\end{split} 
\end{align} 
\end{lemma}

\begin{proof}[Proof of Lemma \ref{lem:ogda_sub}]
Applying Lemma \ref{nemi2004_lemma3.1} with $z_0 = \Tilde{z}_{t-1}$,  
$\zeta_1 = \eta \G(z_{t-1}; \xi_{t-1})$, 
$\zeta_2 = \eta \G(z_t; \xi_t)$, 
and accordingly $w_1 = z_t$, $w_2 = \tilde{z}_t$, 
we get for any $z\in \mathcal{Z}$,
\begin{align}
\begin{split}
\langle \G(z_t;\xi_t), z_t - z\rangle 
\leq&  \frac{1}{2\eta} [\|z-\tilde{z}_{t-1}\|^2  
- \|\Tilde{z}_t-z\|^2] 
- \frac{1}{2\eta}[\|z_t - \Tilde{z}_{t-1}\|^2 
+ \|z_t - \Tilde{z}_{t} \|^2 ]\\ 
& + \eta \|\G(z_{t-1}; \xi_{t-1}) - \G(z_t; \xi_t)\|^2.
\end{split} 
\end{align} 

Taking average over $t = 1, ..., T$ and by the convexity of $f(x, y)$ in $x$,
we have for any $x \in \mathcal{X}$,
\begin{align}
\begin{split}
&\frac{1}{T} \sum\limits_{t=1}^{T} \langle \G(z_t;\xi_t),
z_t - z \rangle \leq 
\frac{\|z - \tilde{z}_0\|^2 }{2\eta T}
- \frac{1}{2\eta T} \sum\limits_{t=1}^{T} (\|z_t - \Tilde{z}_{t-1}\|^2 
+ \|z_t - \Tilde{z}_t\|^2 ) \\ 
&~~~~~~~~~~~~~~~~~~~~~~~~~~~~~~~~
+ \frac{\eta}{T}\sum\limits_{t=1}^{T} \|\G(z_{t-1};\xi_{t-1}) - \G(z_t; \xi_t)\|^2 \\ 
& \leq  \frac{\|z - z_0\|^2 }{2\eta T} 
- \frac{1}{2\eta T} \sum\limits_{t=1}^{T} (\|z_t - \Tilde{z}_{t-1}\|^2 
+ \|z_t - \Tilde{z}_t\|^2 ) 
+ \frac{3\eta }{T} \sum\limits_{t=1}^{T} 
\|F(z_{t-1}) - F(z_t)\|^2  \\
&~~~~~ + \frac{3\eta }{T} \sum\limits_{t=1}^{T} ( \|\G(z_{t}; \xi_{t}) - F(z_t)\|^2  
+ \|\G(z_{t-1};\xi_{t-1}) - F(z_{t-1})\|^2 ), 
\end{split}
\label{local:sgdae_grad_x_1}
\end{align} 
where the last inequality is due to $\left\|\sum\limits_{k=1}^{K} \mathbf{a}_k\right\|^2 \leq K\sum\limits_{k=1}^{K}\left\|\mathbf{a}_k\right\|^2 $. 
Note that
\begin{align} 
\begin{split} 
&\sum\limits_{t=1}^T (\|z_t - \Tilde{z}_{t-1}\|^2 + \|z_t - \Tilde{z}_t\|^2) 
=\sum\limits_{t=0}^{T-1} \|z_{t+1} - \Tilde{z}_t\|^2
 + \sum\limits_{t=1}^T \|z_t - \Tilde{z}_t\|^2 \\
& = \sum\limits_{t=1}^{T-1} \|z_{t+1} - \Tilde{z}_t\|^2 + \|z_1 - \Tilde{z}_0\|^2
 + \sum\limits_{t=1}^{T-1} \|z_t - \Tilde{z}_t\|^2  
 \geq \frac{1}{2}\sum\limits_{t=1}^{T-1} \|z_t - z_{t+1}\|^2 + \|z_1 - \Tilde{z}_0\|^2\\
& \geq \frac{1}{2}\sum\limits_{t=0}^{T-1} \|z_t - z_{t+1}\|^2 
= \frac{1}{2}\sum\limits_{t=1}^{T} \|z_{t-1} - z_{t}\|^2.
\end{split}
\label{local:yi_ineuqality_x}  
\end{align}

By the $\ell$-smoothness of $f(x, y)$, we have 
\begin{align*}
\begin{split}
\|F(z_{t-1}) - F(z_t)\|^2 = & 
\|\nabla_x f(x_t, y_t) - \nabla_x f(x_{t-1}, y_{t-1})\|^2 
+ \|\nabla_y f(x_t, y_t) - \nabla_y f(x_{t-1},  y_{t-1})\|^2 \\
\leq &
2 \|\nabla_x f(x_t, y_t) - \nabla_x f(x_t, y_{t-1})\|^2 + 2\|\nabla_x f(x_t, y_{t-1}) - \nabla_x f(x_{t-1}, y_{t-1})\|^2\\ 
&+  2 \|\nabla_y f(x_t, y_t) - \nabla_y f(x_t, y_{t-1})\|^2 + 2\|\nabla_y f(x_t, y_{t-1}) - \nabla_y f(x_{t-1}, y_{t-1})\|^2 \\
\leq & 2\ell^2 \|y_t - y_{t-1}\|^2 + 2\ell^2 \|x_t - x_{t-1}\|^2 + 2\ell^2 \|y_t - y_{t-1}\|^2 + 2\ell^2 \|x_t - x_{t-1}\|^2 \\
= & 4\ell^2 \|z_{t-1} - z_{t}\|^2.
\end{split} 
\end{align*} 

Denote $\Theta_{t} = F(z_t) - \G(z_t;\xi_t)$.
With the above two inequalities, (\ref{local:sgdae_grad_x_1}) becomes
\begin{align} 
\begin{split}
& \frac{1}{T} \sum\limits_{t=1}^{T} \langle \G(z_t; \xi_t),
z_t - z \rangle\\ 
& \leq \frac{\|z - z_0\|^2 }{2\eta T}
- \frac{1}{4\eta T} \sum\limits_{t=1}^{T} \|z_{t-1} - z_t\|^2
+ \frac{12 \eta \ell^2}{ T} \sum\limits_{t=1}^{T} \|z_{t-1} - z_{t}\|^2  
 + \frac{3\eta }{T} \sum\limits_{t=1}^{T} ( \|\Theta_{t}\|^2 
+ \|\Theta_{t-1}\|^2 ) \\ 
&\leq \frac{\|z - z_0\|^2 }{2\eta T}
+ \frac{3\eta }{T} \sum\limits_{t=1}^{T} ( \|\Theta_{t}\|^2 
+ \|\Theta_{t-1}\|^2 ),
\end{split} 
\label{local:sgdae_grad_x_2} 
\end{align} 
where the last inequality holds because $\eta \leq \frac{1}{4\sqrt{3}\ell}$.

Define a virtual sequence $\{\hat{z}_t \in \mathcal{X}\}_{t=0}^T$ as 
\begin{align}
\begin{split}
\hat{z}_t = \Pi_{\hat{z}_{t-1}}(\eta\Theta_t),  \hat{z}_0 = z_0. 
\end{split} 
\end{align} 
Applying Lemma \ref{nemirovski:cor2} with $\zeta_t =\eta \Theta_{t}= \eta( F(z_t) - \G(z_t; \xi_t)), v_t = \hat{z}_t$ and $u = z$, we have for any $z\in \mathcal{Z}$, 
\begin{align} 
\begin{split} 
\frac{1}{T} \sum\limits_{t=1}^T \langle \Theta_{t}, \hat{z}_{t-1} - z\rangle \leq \frac{1}{2\eta T} \|z_0 - z\|^2 
+ \frac{\eta}{2T} \sum\limits_{t=1}^{T} \|\Theta_{t}\|^2.  
\end{split}
\label{local:sgdae_grad_to_stoc_x} 
\end{align}
Using (\ref{local:sgdae_grad_x_2}) and  (\ref{local:sgdae_grad_to_stoc_x}), we get 
\begin{align}
\begin{split}
& \frac{1}{T} \sum\limits_{t=1}^T \langle F(z_t), z_t - z \rangle
=  \frac{1}{T} \sum\limits_{t=1}^T 
[\langle\G(z_t; \xi_t), z_{t} - z\rangle \!
+ \langle \Theta_{t}, z_t - z\rangle]\\ 
& = \frac{1}{T} \sum\limits_{t=1}^T \langle\G(z_t; \xi_t), z_{t} - z\rangle \!
+ \frac{1}{T} \sum\limits_{t=1}^T  \langle\Theta_{t}, z_{t}\! -\! \hat{z}_{t-1}\rangle 
 +\! \frac{1}{T} \sum\limits_{t=1}^T  \langle\Theta_{t}, \hat{z}_{t-1}\! - z\rangle  \\
& \leq \frac{1}{\eta T} \|z_0 - z\|^2 
+ \frac{\eta}{T} \sum\limits_{t=1}^{T}  \left(\frac{7}{2}\|\Theta_{x, t}\|^2 
+ 3\|\Theta_{x, t-1}\|^2 \right) +\! \frac{1}{T} \sum\limits_{t=1}^T \langle\Theta_{x,t}, x_{t}\! -\! \hat{x}_{t-1}\rangle.
\end{split} 
\label{local:sgdae_grad_x_before_expecation}
\end{align} 

Note 
\begin{align*}
\E[\langle\Theta_{t}, z_t - \hat{z}_{t-1} \rangle | z_t, \hat{z}_{t-1}, \Theta_{t-1}, ..., \Theta_{0}] = 0,
\end{align*}
and by Assumption \ref{ass2} 
\begin{align*}
\E[ \|\Theta_{t}\|^2|z_t,\tilde{z}_{t-1}, \Theta_{t-1}, ...,  \Theta_{0}] \leq 2\sigma^2.
\end{align*}
Thus, taking expectation on both sides of (\ref{local:sgdae_grad_x_before_expecation}), we get
\begin{small}
\begin{align}
\begin{split} 
\E\left[\frac{1}{T} \sum\limits_{t=1}^T \langle F(z_t), z_{t} - z\rangle \right] 
\leq& \frac{1}{\eta T} \E\left[ \|z_0 - z\|^2 \right]
+ 13\eta\sigma^2.
\end{split}
\label{local:sgdae_grad_x}
\end{align}
\end{small}

By the fact $f(x, y)$ is convex-concave,
\begin{align}
\begin{split} 
\E[f(\Bar{x}, y) - f(x, \Bar{y})] \leq & \E\left[\frac{1}{T}\sum\limits_{t=1}^{T} (f(x_t, y) - f(x, y_t))\right] \\
=& \E\left[\frac{1}{T}\sum\limits_{t=1}^{T} (f(x_t, y) - f(x_t, y_t) + f(x_t, y_t) - f(x, y_t))\right] \\
\leq & \E\left[\frac{1}{T}\sum\limits_{t=1}^{T} (
\langle-\nabla_y f(x_t, y_t), y_{t}\! -\! y\rangle  +
\langle\nabla_x f(x_t, y_t), x_{t}\! -\! x\rangle
)\right] \\
= & \E\left[\frac{1}{T}\sum\limits_{t=1}^{T}\langle F(z_t), z_t - z\rangle\right] 
\leq \frac{1}{\eta T}E[\|z_0 - z\|^2] + 13\eta \sigma^2. 
\end{split}
\label{local:sgdae_grad_y} 
\end{align} 
Then we can conclude by plugging in $z = (x, y) = (\hat{x}(\bar{y}), \hat{y}(\bar{x}))$. 
\end{proof}


\section{Proof of Theorem \ref{thm:ogda_primal} and  Theorem \ref{thm:gda_primal}} 
Before we prove these two theorems, we first present two lemmas from \citep{yan2020sharp} and we introduce Theorem \ref{thm:appendix_unify_3} that unifies the proof of  Theorem \ref{thm:ogda_primal} and  Theorem \ref{thm:gda_primal}.

\begin{lemma} [Lemma 1 of \citep{yan2020sharp}]
Suppose a function $h(x, y)$ is $\lambda_1$-strongly convex in $x$ and $\lambda_2$-strongly concave in $y$. 
Consider the following problem 
\begin{align*}
\min\limits_{x\in \mathcal{X}} \max\limits_{y\in \mathcal{Y}} h(x, y), 
\end{align*}
\label{lem:Yan1}
where $\mathcal{X}$ and $\mathcal{Y}$ are convex sets.
Denote $\hat{x}_h (y) = \arg\min\limits_{x'\in \mathcal{X}} h(x', y)$ 
and $\hat{y}_h(x) = \arg\max\limits_{y' \in \mathcal{Y}} h(x, y')$.
Suppose we have two solutions $(x_0, y_0)$ and $(x_1, y_1)$.
Then the following relation between variable distance and duality gap holds
\begin{align}
\begin{split}
\frac{\lambda_1}{4} \|\hat{x}_h(y_1)-x_0\|^2 + \frac{\lambda_2}{4}\|\hat{y}_h(x_1) - y_0\|^2 \leq& 
\max\limits_{y' \in \mathcal{Y}} h(x_0, y') - \min\limits_{x' \in \mathcal{X}}h(x', y_0) \\ 
&+ \max\limits_{y' \in \mathcal{Y}} h(x_1, y') - \min\limits_{x' \in \mathcal{X}} h(x', y_1).
\end{split}
\end{align}
\end{lemma} 

\begin{lemma}[Lemma 5 of \citep{yan2020sharp}]
We have the following lower bound for $\emph{\text{Gap}}_k(\bx_k, \by_k)$ 
\begin{align*} 
\emph{\text{Gap}}_k(\bx_k, \by_k) \geq \frac{3}{50}\emph{\text{Gap}}_{k+1}(x_0^{k+1}, y_0^{k+1}) + \frac{4}{5}(P(x_0^{k+1}) - P(x_0^k)), 
\end{align*} 
\label{lem:Yan5} 
\end{lemma}
where $x_0^{k+1} = \bx_k$ and $y_0^{k+1} = \by_k$.

We will introduce the following theorem that can unify the proof of Theorem \ref{thm:ogda_primal} and Theorem \ref{thm:gda_primal} since their have pretty similar forms of bounds in solving the subproblem.
\begin{theorem} 
Suppose Assumption \ref{ass1} and Assumption  \ref{ass3} hold. 
Assume we have a subroutine in the $k$-th stage of Algorithm \ref{alg:main} that can return $\bx_k, \by_k$ such that
\begin{align}
\begin{split}
\E[\emph{\text{Gap}}_k(\bx_k, \by_k)] 
\leq \frac{C_1}{\eta_k T_k} \E[\|\hat{x}_k(\bar{y}_k) - x_0^k\|^2 + \|\hat{y}_k(\bar{x}_k) - y_0^k\|^2]
+ \eta_k C_2,
\end{split} 
\label{equ:thm:appendix_unify_3_sub}
\end{align}
where $C_1$ and $C_2$ are constants corresponding to the specific subroutine. 
Take $\gamma = 2\rho$ and
denote $\hat{L} = L + 2\rho$ and $c = 4\rho+\frac{248}{53} \hat{L} \in O(L + \rho)$.
Define $\Delta_k = P(x_0^k) - P(x_*) + \frac{8\hat{L}}{53c}\emph{\text{Gap}}_k(x_0^k, y_0^k)$  and $\epsilon_0=\emph{\text{Gap}}(\bx_0, \by_0)$.
Then we can set $\eta_k = \eta_0  \exp(-(k-1)\frac{2\mu}{c+2\mu})$, $T_k = \left\lceil\frac{212 C_1}{\eta_0 \min\{\rho, \mu_y\}}\exp\left((k-1)\frac{2\mu}{c+2\mu}\right)\right\rceil$. 
After $K = \left\lceil\max\left\{\frac{c+2\mu}{2\mu}\log \frac{4\epsilon_0}{\epsilon},
\frac{c+2\mu}{2\mu} \log \frac{16\eta_0 \hat{L}KC_2}{(c+2\mu)\epsilon} \right\}\right\rceil$ stages, we can have $\Delta_{K+1} \leq \epsilon$. 
The total stochastic first-order oracle call complexity is 
$\widetilde{O}\left(\max\left\{\frac{(L+\rho)C_1 \epsilon_0}{\eta_0 \mu\min\{\rho, \mu_y\} \epsilon}, 
\frac{(L + \rho)^2 C_2} {\mu^2 \min\{\rho,\mu_y\} \epsilon} \right\}\right)$. 
\label{thm:appendix_unify_3}
\end{theorem}

\begin{proof}[Proof of Theorem \ref{thm:appendix_unify_3}] 
Since $f(x, y)$ is $\rho$-weakly convex in $x$ for any $y$, $P(x) = \max\limits_{y'\in \mathcal{Y}} f(x, y')$ is also $\rho$-weakly convex. 
Taking $\gamma = 2\rho$, we have
\begin{align}  
\begin{split} 
P(\bx_{k-1}) &\geq P(\bx_k) + \langle \nabla P(\bx_k), \bx_{k-1} - \bx_k\rangle - \frac{\rho}{2} \|\bx_{k-1} - \bx_k\|^2 \\ 
& = P(\bx_k) + \langle \nabla P(\bx_k) + 2 \rho (\bx_k - \bx_{k-1}), \bx_{k-1} - \bx_k\rangle + \frac{3\rho}{2} \|\bx_{k-1} - \bx_k\|^2 \\ 
& \overset{(a)}{=} P(\bx_k) + \langle \nabla P_k(\bx_k), \bx_{k-1} - \bx_k \rangle + \frac{3\rho}{2} \|\bx_{k-1} - \bx_k\|^2 \\ 
& \overset{(b)}{=}  P(\bx_k) - \frac{1}{2\rho} \langle \nabla P_k(\bx_k), \nabla P_k(\bx_k) - \nabla P(\bx_k) \rangle + \frac{3}{8\rho} \|\nabla P_k(\bx_k) - \nabla P(\bx_k)\|^2 \\ 
& = P(\bx_k) - \frac{1}{8\rho} \|\nabla P_k(\bx_k)\|^2
- \frac{1}{4\rho} \langle \nabla P_k(\bx_k), \nabla P(\bx_k)\rangle + \frac{3}{8\rho} \|\nabla P(\bx_k)\|^2,
\end{split} 
\label{local:P_weakly} 
\end{align} 
where $(a)$ and $(b)$ hold by the definition of $P_k(x)$.

Rearranging the terms in (\ref{local:P_weakly}) yields
\begin{align}
\begin{split}
P(\bx_k) - P(\bx_{k-1}) &\leq \frac{1}{8\rho} \|\nabla P_k(\bx_k)\|^2 + \frac{1}{4\rho}\langle \nabla P_k(\bx_k), \nabla P(\bx_k)\rangle - \frac{3}{8\rho} \|\nabla P(\bx_k)\|^2 \\
&\overset{(a)}{\leq} \frac{1}{8\rho} \|\nabla P_k(\bx_k)\|^2
+ \frac{1}{8\rho} (\|\nabla P_k(\bx_k)\|^2 + \|\nabla P(\bx_k)\|^2) - \frac{3}{8\rho} \|P(\bx_k)\|^2 \\
& = \frac{1}{4\rho}\|\nabla P_k(\bx_k)\|^2 - \frac{1}{4\rho}\|\nabla P(\bx_k)\|^2\\
& \overset{(b)}{\leq} \frac{1}{4\rho} \|\nabla P_k(\bx_k)\|^2 - \frac{\mu}{2\rho}(P(\bx_k) - P(x_*)),
\end{split} 
\end{align}
where $(a)$ holds by using $\langle \mathbf{a}, \mathbf{b}\rangle \leq \frac{1}{2}(\|\mathbf{a}\|^2 +  \|\mathbf{b}\|^2)$, and $(b)$ holds by the $\mu$-PL property of $P(x)$.

Thus, we have 
\begin{align}
\left(4\rho+2\mu\right) (P(\bx_k) - P(x_*)) - 4\rho (P(\bx_{k-1}) - P(x_*)) \leq \|\nabla P_k(\bx_k)\|^2. 
\label{local:nemi_thm_nabla_P_k_1} 
\end{align} 

Since $\gamma = 2\rho$, $f_k(x, y)$ is $\rho$-strongly convex in $x$ and $\mu_y$ strong concave in $y$.
Apply Lemma \ref{lem:Yan1} to $f_k$, we know that 
\begin{align} 
\frac{\rho}{4} \|\hat{x}_k(\by_k) - x_0^k\|^2 + \frac{\mu_y}{4} \|\hat{y}_k(\bx_k) - y_0^k\|^2 \leq {\text{Gap}}_k(x_0^k, y_0^k) + {\text{Gap}}_k(\bx_k, \by_k). 
\end{align}

By the setting of $\eta_k = \eta_0  \exp\left(-(k-1)\frac{2\mu}{c+2\mu}\right)$, and 
$T_k = \left\lceil\frac{212 C_1}{\eta_0 \min\{\rho, \mu_y\}} \exp  \left((k-1)\frac{2\mu}{c+2\mu}\right)\right\rceil$, we note that $\frac{C_1}{\eta_k T_k} \leq \frac{\min\{\rho,\mu_y\}}{212}$. 
Applying (\ref{equ:thm:appendix_unify_3_sub}), we have
\begin{align}  
\begin{split} 
&\E[{\text{Gap}}_k(\bx_k, \by_k)] 
\leq \eta_k C_2 
+ \frac{1}{53} \E\left[\frac{\rho}{4}\|\hat{x}_k(\by_k) - x_0^k\|^2 + \frac{\mu_y}{4}\|\hat{y}_k(\bx_k) - y_0^k\|^2 \right]
\\  
& \leq \eta_k C_2 + \frac{1}{53} \E\left[{\text{Gap}}_k(x_0^k, y_0^k) + {\text{Gap}}_k(\bx_k, \by_k)\right]. 
\end{split} 
\end{align} 

Since $P(x)$ is $L$-smooth and $\gamma = 2\rho$, then $P_k(x)$ is $\hat{L} = (L+2\rho)$-smooth. 
According to Theorem 2.1.5 of \citep{DBLP:books/sp/Nesterov04}, we have 
\begin{align}
\begin{split} 
& \E[\|\nabla P_k(\bx_k)\|^2] \leq 2\hat{L}\E(P_k(\bx_k) - \min\limits_{x\in \mathbb{R}^{d}} P_k(x)) \leq 2\hat{L}\E[{\text{Gap}}_k(\bx_k, \by_k)] \\
& = 2\hat{L}\E[4{\text{Gap}}_k(\bx_k, \by_k) - 3{\text{Gap}}_k(\bx_k, \by_k)] \\ 
&\leq 2\hat{L} \E \left[4\left(\eta_k C_2+  \frac{1}{53}\left({\text{Gap}}_k(x_0^k, y_0^k) + {\text{Gap}}_k(\bx_k, \by_k)\right)\right) - 3{\text{Gap}}_k(\bx_k,\by_k)\right] \\
& = 2\hat{L} \E \left[4\eta_k C_2 +   \frac{4}{53}{\text{Gap}}_k(x_0^k, y_0^k) - \frac{155}{53}{\text{Gap}}_k(\bx_k, \by_k)\right].
\end{split}  
\label{local:nemi_thm_P_smooth} 
\end{align}

Applying Lemma \ref{lem:Yan5} to  (\ref{local:nemi_thm_P_smooth}), we have
\begin{align*}
\begin{split} 
& \E[\|\nabla P_k(\bx_k)\|^2] \leq 2\hat{L} \E \bigg[4\eta_k C_k + \frac{4}{53}{\text{Gap}}_k(x_0^k, y_0^k) \\  
&~~~~~~~~~~~~~~~~~~~~~~~~~~~~~~~~~~~~~~~~  
- \frac{155}{53} \left(\frac{3}{50} {\text{Gap}}_{k+1}(x_0^{k+1}, y_0^{k+1}) + \frac{4}{5} (P(x_0^{k+1}) - P(x_0^k))\right) \bigg] \\
& = 2\hat{L}\E \bigg[4\eta_k C_2 \!+ \! \frac{4}{53}{\text{Gap}}_k(x_0^k, y_0^k) \!-\! \frac{93}{530}{\text{Gap}}_{k+1}(x_0^{k+1}, y_0^{k+1}) \!-\! 
\frac{124}{53} (P(x_0^{k+1}) - P(x_0^k)) \bigg]. 
\end{split} 
\end{align*}

Combining this with  (\ref{local:nemi_thm_nabla_P_k_1}), rearranging the terms, and defining a constant $c = 4\rho + \frac{248}{53}\hat{L} \in O(L+\rho)$, we get
\begin{align} 
\begin{split}
&\left(c + 2\mu\right)\E [P(x_0^{k+1}) - P(x_*)] + \frac{93}{265}\hat{L} \E[{\text{Gap}}_{k+1}(x_0^{k+1}, y_0^{k+1})] \\ 
&\leq \left(4\rho + \frac{248}{53} \hat{L}\right) \E[P(x_0^{k}) - P(x_*)] 
+ \frac{8\hat{L}}{53} \E[{\text{Gap}}_k(x_0^k, y_0^k)] 
+ 8 \eta_k \hat{L}C_2  \\ 
& \leq c \E\left[P(x_0^k) - P(x_*) + \frac{8\hat{L}}{53c} {\text{Gap}}_k(x_0^k, y_0^k)\right] + 8 \eta_k \hat{L}C_2.
\end{split}  
\end{align} 

Using the fact that $\hat{L} \geq \mu$,
\begin{align}
\begin{split}
(c+2\mu) \frac{8\hat{L}}{53c} = \left(4\rho + \frac{248}{53}\hat{L} + 2\mu\right)\frac{8\hat{L}}{53(4\rho + \frac{248}{53}\hat{L})} \leq \frac{8\hat{L}}{53} + \frac{16\mu \hat{L}}{248\hat{L}} \leq \frac{93}{265} \hat{L}. 
\end{split}
\end{align}

Then, we have
\begin{align}
\begin{split} 
&(c+2\mu)\E \left[P(x_0^{k+1}) - P(x_*) + \frac{8\hat{L}}{53c}{\text{Gap}}_{k+1}(x_0^{k+1}, y_0^{k+1})\right] \\ 
&\leq c \E \left[P(x_0^{k}) - P(x_*) 
+  \frac{8\hat{L}}{53c}{\text{Gap}}_{k}(x_0^{k},  y_0^{k})\right] 
+ 8 \eta_k \hat{L}C_2. 
\end{split}
\end{align}

Defining $\Delta_k = P(x_0^{k}) - P(x_*) +  \frac{8\hat{L}}{53c}{\text{Gap}}_{k}(x_0^{k}, y_0^{k})$, then
\begin{align}
\begin{split}
&\E[\Delta_{k+1}] \leq \frac{c}{c+2\mu} \E[\Delta_{k}] +  \frac{8\eta_k \hat{L} C_2}{c+2\mu}.
\end{split}
\end{align}

Using this inequality recursively, it yields
\begin{align} 
\begin{split}
& E[\Delta_{K+1}] \leq \left(\frac{c}{c+2\mu}\right)^K E[\Delta_1]
+ \frac{8 \hat{L} C_2}{c+2\mu} \sum\limits_{k=1}^{K} \left(\eta_k \left(\frac{c}{c+2\mu}\right)^{K+1-k} \right).
\end{split}
\end{align} 

By definition,
\begin{align*} 
\begin{split}
\Delta_1 &= P(x_0^1) - P(x_*) + \frac{8\hat{L}}{53c}{\text{Gap}}_1(x_0^1, y_0^1) \\
& = P(\bx_0) - P(x_*) + \left( f(\bx_0, \hat{y}_1(\bx_0)) +  \frac{\gamma}{2}\|\bx_0 - \bx_0\|^2 -  f(\hat{x}_1(\by_0), \by_0) - \frac{\gamma}{2}\|\hat{x}_1(\by_0) - \bx_0\|^2 \right) \\ 
& \leq \epsilon_0 + f(\bx_0, \hat{y}_1(\bx_0)) - f(\hat{x}(\by_0), \by_0) \leq 2\epsilon_0.
\end{split} 
\end{align*}
Using inequality $1-x \leq \exp(-x)$,
we have
\begin{align*}
\begin{split} 
&\E[\Delta_{K+1}] \leq \exp\left(\frac{-2\mu K}{c+2\mu}\right)\E[\Delta_1] + \frac{8 \eta_0 \hat{L} C_2} {c+2\mu} 
\sum\limits_{k=1}^{K}\exp\left(-\frac{2\mu K}{c+2\mu}\right) \\
&\leq 2\epsilon_0 \exp\left(\frac{-2\mu K}{c+2\mu}\right)
+ \frac{8 \eta_0 \hat{L} C_2 }{c+2\mu}  
K\exp\left(-\frac{2\mu K}{c+2\mu}\right). 
\end{split}  
\end{align*}  

To make this less than $\epsilon$, it suffices to make
\begin{align*}
\begin{split}
& 2\epsilon_0 \exp\left(\frac{-2\mu K}{c+2\mu}\right) \leq \frac{\epsilon}{2},\\
&\frac{8 \eta_0 \hat{L} C_2 }{c+2\mu} 
K\exp\left(-\frac{2\mu K}{c+2\mu}\right) \leq \frac{\epsilon}{2}.
\end{split} 
\end{align*}

Let $K$ be the smallest value such that $\exp\left(\frac{-2\mu K}{c+2\mu}\right) \leq \min \{ \frac{\epsilon}{4\epsilon_0}, 
\frac{(c+2\mu)\epsilon}{16 \eta_0 \hat{L} K C_2}\}$. 
We can set $K = \left\lceil\max\bigg\{\frac{c+2\mu}{2\mu}\log \frac{4\epsilon_0}{\epsilon}, 
\frac{c+2\mu}{2\mu}\log \frac{16 \eta_0 \hat{L} K C_2} {(c+2\mu)\epsilon} \bigg\}\right\rceil$. 
Then, the total stochastic first-order oracle call complexity is 
\begin{align*} 
\sum\limits_{k=1}^{K}T_k &\leq O\left( \frac{212 C_1}{\eta_0\min\{\rho,\mu_y\}}
\sum\limits_{k=1}^{K}\exp\left((k-1)\frac{2\mu}{c+2\mu}\right) \right)\\    
& \leq O\bigg(\frac{212C_1}{\eta_0\min\{\rho,\mu_y\}}  \frac{\exp(K\frac{2\mu}{c+2\mu}) -  1}{\exp(\frac{2\mu}{c+2\mu})-1}  \bigg)\\ 
& \overset{(a)}{\leq} \widetilde{O}   \left(\frac{cC_1}{\eta_0\mu\min\{\rho,\mu_y\}}  \max\left\{\frac{\epsilon_0}{\epsilon}, 
\frac{\eta_0 \hat{L} K C_2} 
{(c+2\mu)\epsilon} \right\}\right) \\
& \leq  \widetilde{O}\left(\max\left\{\frac{(L+\rho)C_1 \epsilon_0}{\eta_0 \mu\min\{\rho, \mu_y\} \epsilon}, 
\frac{(L + \rho)^2 C_2} {\mu^2 \min\{\rho,\mu_y\} \epsilon}\right\} \right), 
\end{align*} 
where $(a)$ uses the setting of $K$ and $\exp(x) - 1\geq x$, and $\widetilde{O}$ suppresses logarithmic factors. 
\end{proof} 

\begin{proof}[Proof of Theorem \ref{thm:ogda_primal}]
With the above theorem, Theorem \ref{thm:ogda_primal} directly follows. 
Noting Lemma \ref{lem:ogda_sub}, we can plug in $\eta_0 = \frac{1}{2\sqrt{2}\ell}$, $C_1 = 1$ and $C_2 = 13\sigma^2$ to Theorem \ref{thm:appendix_unify_3}.
\end{proof}

\begin{proof}[Proof of Theorem \ref{thm:gda_primal}]
We need the following lemma to bound the convergence of the subproblem at each stage,

\begin{lemma}[Lemma 4 of \citep{yan2020sharp}]
Suppose Assumption \ref{ass1} holds, \\$\E\|\nabla_x f(x_t, y_t; \xi_t)\|^2 \leq B^2$ and $\E\|\nabla_y f(x_t, y_t; \xi_t)\|^2 \leq B^2$.
Set $\gamma = 2\rho$.
By running Algorithm \ref{alg:main} with Option II: SGDA, it holds  for $k\geq 1$,
\begin{align*} 
\begin{split}
E[\emph{\text{Gap}}_k(\bx_k, \by_k)] \leq 5\eta_k B^2 + \frac{1}{T_k}\bigg\{\left(\frac{1}{\eta_k} + \frac{\rho}{2} \right)E[\|\hat{x}_k(\by_k) - x_0^k\|^2] + \frac{1}{\eta_k}E[\|\hat{y}_k(\bx_k) - y_0^k\|^2] \bigg\}.
\end{split} 
\end{align*}
\label{lem:Yan4}
\end{lemma}

Using this lemma , we can set $\gamma = 2\rho$ and 
$\eta_0 = \frac{1}{\rho}$. 
Then it follows that 
\begin{equation*} 
\begin{split}
E[\text{Gap}_k(\bx_k, \by_k)] \leq 5\eta_k B^2
+ 
\frac{2}{\eta_k T_k} \left(E[\|\hat{x}_k(\bar{y}_k) - x_0^k\|^2] + E[\|\hat{y}_k(\bar{x}_k) - y_0^k\|^2\right).
\end{split} 
\end{equation*} 
We plug in $\eta_0 \leq \frac{1}{\rho}$,  $C_1 = 2$ and $C_2 = 5B^2 $ to Theorem \ref{thm:appendix_unify_3}
and the conclusion follows.
\end{proof} 

\section{Proof of Theorem \ref{thm:ogda_primal_rho<mu} and Theorem \ref{thm:gda_primal_rho<mu}}
We first present a lemma by plugging in Lemma 8 of \citep{yan2020sharp}.
And then we a theorem that can unify the proof of Theorem  \ref{thm:ogda_primal_rho<mu} and Theorem \ref{thm:gda_primal_rho<mu}.
In the last, we prove Theorem \ref{thm:ogda_primal_rho<mu} and Theorem \ref{thm:gda_primal_rho<mu}. 

\begin{lemma}[Lemma 8 of \citep{yan2020sharp}]
Suppose $f(x, y)$ is $\frac{\mu}{8}$-weakly convex in $x$ for any $y$ and set $\gamma = \frac{\mu}{4}$. 
Thus, $f_k(x, y)$ is $\frac{\mu}{8}$-strongly convex in $x$. 
Then $\emph{\text{Gap}}_k(\bx_k, \by_k)$ can be lower bounded by the following inequalities 
\begin{align} 
\emph{\text{Gap}}_k(\bar{x}_k, \bar{y}_k) \geq \left(3 - \frac{2}{\alpha}\right) \emph{\text{Gap}}_{k+1}(x_0^{k+1}, y_0^{k+1}) - \frac{\mu \alpha}{8(1-\alpha)}\|x_0^{k+1} - x_0^{k}\|^2, (0<\alpha\leq 1),
\end{align}
and 
\begin{align}
\emph{\text{Gap}}_k(\bar{x}_k, \bar{y}_k) \geq P(x_0^{k+1}) - P(x_0^k) + \frac{\mu}{8} \|\bx_k-x_0^k\|^2, where ~ P(x) = \max\limits_{y' \in \mathcal{Y}} f(x, y').
\end{align} 
\label{lem:Yan8}
\end{lemma} 

\begin{theorem} 
Suppose $0<\rho \leq \frac{\mu}{8}$ 
and suppose Assumption \ref{ass1}, \ref{ass2}, \ref{ass5}, \ref{ass3} hold. 
Assume we have a subroutine in the $k$-th stage of Algorithm \ref{alg:main} that can return $\bx_k, \by_k$ such that
\begin{align}
\begin{split}
\E[\emph{\text{Gap}}_k(\bar{x}_k, \bar{y}_k)] 
\leq \frac{C_1}{\eta_k T_k} \E[\|x - x_0^k\|^2 + \|y - y_0^k\|^2]
+ \eta_k C_2,
\end{split} 
\label{equ:thm:appendix_unify_rho<mu_sub}
\end{align}
where $C_1$ and $C_2$ are constants corresponding to the specific subroutine. 
Take $\gamma = \frac{\mu}{4}$.
Define $\Delta_k = 475(P(x_0^k) - P(x_*)) + 57\emph{\text{Gap}}_k(x_0^k, y_0^k)$  and $\epsilon_0=\emph{\text{Gap}}(\bar{x}_0, \bar{y}_0)$.
Then we can set $\eta_k = \eta_0  \exp(-\frac{k-1}{16}) \leq \frac{1}{2\sqrt{2}\ell}$, $T_k = \left\lceil\frac{384C_1}{\eta_0 \min\{\mu/8, \mu_y\}}\exp\left(\frac{k-1}{16}\right)\right\rceil$. After \\$K = \left\lceil\max\left\{16\log \frac{1200\epsilon_0}{\epsilon},
16\log \frac{1200\eta_0  K C_2}{\epsilon} \right\}\right\rceil$ stages, we can have $\Delta_{K+1} \leq \epsilon$. 
The total stochastic first-order oracle call complexity is 
$\widetilde{O}\left(\max\left\{\frac{C_1 \epsilon_0}{\eta_0\min\{\mu, \mu_y\}\epsilon},  \frac{C_2}{\min\{\mu, \mu_y\} \epsilon}\right\}\right)$.  
\label{thm:appendix_unify_rho<mu}
\end{theorem}

\begin{proof}[Proof of Theorem \ref{thm:appendix_unify_rho<mu}]
We have the following relation between $P(x_0^k) - P(x_*)$ and ${\text{Gap}}_k(x_0^k, y_0^k)$,
\begin{align} 
\begin{split} 
&P(x_0^k) - P(x_*) = f(x_0^k, \hat{y}(x_0^k)) - f(x_*, y_*) \leq f(x_0^k, \hat{y}(x_0^k)) - f(x_*, y_0^k) \\ 
&= f(x_0^k, \hat{y}(x_0^k)) + \frac{\gamma}{2}\|x_0^k - x_0^k\|^2 - f(x_*, y_0^k) - \frac{\gamma}{2}\|x_* - x_0^k\|^2 + \frac{\gamma}{2} \|x_* - x_0^k\|^2\\
&= f_k(x_0^k, \hat{y}(x_0^k)) - f_k(x_*, y_0^k) + \frac{\gamma}{2}\|x_* - x_0^k\|^2 \\
& \leq \hat{f}_k(x_0^k, \hat{y}_k(x_0^k)) 
- f_k(\hat{x}_k(y_0^k), y_0^k) + \frac{\gamma}{2} \|x_* - x_0^k\|^2\\ 
&= {\text{Gap}}_k(x_0^k, y_0^k) + \frac{\gamma}{2}\|x_* - x_0^k\|^2\\
&\leq {\text{Gap}}_k(x_0^k, y_0^k) 
+ \frac{\gamma}{4\mu}(P(x_0^k) - P(x_*)),
\end{split}
\end{align}
where the first inequality holds by the Lemma \ref{lem:primal_gradient_danskin}, and the last inequality due to the $\mu$-PL condition of $P(x)$.
Since we take $\gamma = \frac{\mu}{4}$,  we know that $1 - \frac{\gamma}{4\mu} = \frac{15}{16}$.
Then it follows that
\begin{align}
P(x_0^k) - P(x_*) \leq \frac{16}{15}{\text{Gap}}_k(x_0^k, y_0^k).
\label{local:nemi_P_to_Gap_k}
\end{align}

Since $\rho < \frac{\mu}{8}$ and $\gamma = \frac{\mu}{4}$, we know that ${f}_k(x, y)$ is $\lambda_x = \frac{\mu}{8}$-strongly convex in $x$.
By the setting $\eta_k = \eta_0 \exp \left(-\frac{k-1}{16}\right), T_k = \left\lceil\frac{384C_1}{\eta_0\min\{\lambda_x, \mu_y\}}  \exp\left(\frac{k-1}{16}\right)\}\right\rceil$, we note that $\frac{C_1}{\eta_k T_k} \leq \frac{\min\{\lambda_x, \mu_y\}}{384}$. 
Applying \ref{equ:thm:appendix_unify_rho<mu_sub}, we have 
\begin{align}
\begin{split}
\E[{\text{Gap}}_k(\bar{x}_k, \bar{y}_k)] &{\leq} 
\eta_k C_2 
+ \frac{1}{96} \left(\frac{\lambda_x}{4}\E[\|\hat{x}_k(\by_k)-x_0^k\|^2]
+ \frac{\mu_y}{4} \E[\|\hat{y}_k(\bx_k) - y_0^k\|^2] \right) \\
&{\leq} \eta_k C_2
+ \frac{1}{96}E[{\text{Gap}}_k(x_0^k, y_0^k)]
+ \frac{1}{96}E[{\text{Gap}}_k(\bx_k, \by_k)],
\end{split}
\end{align} 
where the last inequality follows from Lemma 
\ref{lem:Yan1}.
Rearranging the terms, we have
\begin{align}
\begin{split}
\frac{95}{96} \E[{\text{Gap}}_k(\bar{x}_k, \bar{y}_k)] 
\leq \eta_k C_2 
+ \frac{1}{96}\E[{\text{Gap}}_k(x_0^k, y_0^k)] .
\end{split}
\label{local:nemi_hat_gap_k}
\end{align}
Since $\rho\leq \frac{\mu}{8}$, $f(x, y)$ is also $\frac{\mu}{8}$-weakly convex in $x$.
Then we use Lemma \ref{lem:Yan8} to lower bound the LHS of (\ref{local:nemi_hat_gap_k}) with $\alpha = \frac{5}{6}$,
\begin{align}
\begin{split} 
&\frac{95}{96} {\text{Gap}}_k(\bx_k, \by_k) =\frac{95}{576}{\text{Gap}}_k(\bx_k, \by_k) + \frac{475}{576} {\text{Gap}}_k(\bx_k, \by_k) \\
& \overset{(a)}{\geq} \frac{95}{576}\left(\frac{3}{5} {\text{Gap}}_{k+1}(x_0^{k+1}, y_0^{k+1}) -\frac{5}{8}\mu \|x_0^{k+1} - x_0^{k}\|^2\right) \\ 
&~~~~~~ + \frac{475}{576} (P(x_0^{k+1}) - P(x_*)) + \frac{475}{576}(P(x_*) - P(x_0^k)) + \frac{475}{576}\frac{\mu}{8}\|x_0^k - x_0^{k+1}\|^2\\
&= \frac{57}{576}{\text{Gap}}_{k+1}(x_0^{k+1}, y_0^{k+1}) + \frac{475}{576}(P(x_0^{k+1})-P(x_*)) \\
& ~~~~~~ - \frac{475}{576}\cdot \frac{15}{16}(P(x_0^k)-P(x_*)) - \frac{475}{576}\left(1-\frac{15}{16}\right)(P(x_0^k) - P(x_*)) \\
&\overset{(b)}{\geq} \frac{57}{576}{\text{Gap}}_{k+1}(x_0^{k+1}, y_0^{k+1}) + \frac{475}{576}(P(x_0^{k+1})-P(x_*)) \\
&~~~~~~ - \frac{475}{576}\cdot \frac{15}{16}(P(x_0^k)-P(x_*)) - \frac{475}{576}\cdot\frac{1}{15} {\text{Gap}}_k(x_0^k, y_0^k),
\end{split} 
\label{local:nemi_LHS_hat_gap_k}
\end{align} 
where $(a)$ uses Lemma \ref{lem:Yan8} and $(b)$ uses (\ref{local:nemi_P_to_Gap_k}). 
Combining (\ref{local:nemi_hat_gap_k}) and (\ref{local:nemi_LHS_hat_gap_k}), we get
\begin{align}
\begin{split}
&\E\left[\frac{475}{576}(P(x_0^{k+1})-P(x_*)) +  \frac{57}{576}{\text{Gap}}_{k+1}(x_0^{k+1}, y_0^{k+1}) \right] \\
&\leq \E\left[ \eta_k C_2 
+ \frac{475}{576}\cdot \frac{15}{16}(P(x_0^k)-P(x_*)) +  \frac{475}{576}\cdot\frac{1}{15} {\text{Gap}}_k(x_0^k, y_0^k)
+\frac{1}{96} {\text{Gap}}_k (x_0^k, y_0^k) \right]\\
&\leq \eta_k C_2 
+ \frac{15}{16}\E\left[\frac{475}{576}(P(x_0^k)-P(x_*)) + \frac{57}{576}{\text{Gap}}_k (x_0^k, y_0^k) \right].
\end{split} 
\end{align} 
Defining $\Delta_k = 475(P(x_0^k)-P(x_*)) + 57{\text{Gap}}_k (x_0^k, y_0^k)$, we have 
\begin{align}
\E[\Delta_{k+1}] \leq 600 \eta_k C_2
+ \frac{15}{16} \E[\Delta_k] 
\leq \exp\left(-1/16\right)\E[\Delta_{k}] + 600 \eta_k C_2,
\end{align} 
and 
\begin{align*}
\begin{split}
&\Delta_1 = 475(P(x_0^1) - P(x_*)) + 57{\text{Gap}}_1(x_0^1, y_0^1) \\
& = 475(P(\bx_0) - P(x_*)) + 57\left( f(\bx_0, \hat{y}_1(\bx_0)) + \frac{\gamma}{2}\|\bx_0 - \bx_0\|^2 - f(\hat{x}_1(\by_0), \by_0) - \|\hat{x}_1(\by_0) - \bx_0\|^2 \right) \\
& \leq 475\epsilon_0 + 57 \left( f(\bx_0, \hat{y}_1(\bx_0)) - f(\hat{x}(\by_0), \by_0) \right)  \leq 600 \epsilon_0.
\end{split} 
\end{align*} 

Thus, 
\begin{align}
\begin{split}
\E[\Delta_{K+1}] &\leq \exp\left(-K/16\right)\Delta_1 + 600 C_2 \sum\limits_{k=1}^{K} \eta_k \exp\left(-(K+1-k)/(16)\right) \\
&= \exp\left(-K/16\right)\Delta_1 + 600C_2\sum\limits_{k=1}^{K} \left(\eta_0  \exp\left(-K/16\right)\right) \\  
&\leq 600\epsilon_0 \exp\left(-K/16\right) + 600 \eta_0 C_2 K\exp\left(-K/16\right).
\end{split} 
\end{align} 
To make this less than $\epsilon$, we just need to make  
\begin{align*} 
\begin{split} 
&600\epsilon_0 \exp\left(-K/16\right) \leq \frac{\epsilon}{2}, \\
&600 \eta_0 C_2 K \exp\left(-K/16\right) \leq \frac{\epsilon}{2}.
\end{split} 
\end{align*} 
Let $K$ be the smallest value such that $\exp\left(\frac{-K}{16}\right) \leq  \min\{\frac{\epsilon}{1200\epsilon_0}, \frac{\epsilon}{1200\eta_0 C_2 K}\}$.  
We can set $K = \left\lceil\max\left\{ 16\log\left(\frac{1200\epsilon_0}{\epsilon}\right) , 16\log\left(\frac{1200\eta_0 
C_2 K}{ \epsilon}\right) \right\}\right\rceil$.  
Then the total stochastic first-order oracle call complexity is
\begin{align}
\begin{split}
\sum\limits_{k=1}^{K} T_k
&\leq O\left(\frac{384 C_1}{\eta_0\min\{\lambda_x, \mu_y\}}  \sum\limits_{k=1}^{K} \exp\left(\frac{k-1}{16}\right) \right) \\ 
&\leq {O}\left( \frac{384C_1}{\eta_0 \min\{\lambda_x, \mu_y\}} \frac{\exp\left(\frac{K}{16}\right) - 1}{\exp\left(\frac{1}{16}\right) - 1}\right) \\ 
&\leq \widetilde{O}\left(\max\left\{\frac{C_1 \epsilon_0}{\eta_0\min\{\mu, \mu_y\}\epsilon},  \frac{KC_2}{\min\{\mu, \mu_y\} \epsilon}\right\}\right)\\ 
&\leq \widetilde{O}\left(\max\left\{\frac{C_1  \epsilon_0}{\eta_0\min\{\mu, \mu_y\}\epsilon},  \frac{C_2}{\min\{\mu, \mu_y\} \epsilon}\right\}\right).
\end{split}  
\end{align} 
\end{proof}
\vspace*{-0.3in}  
\begin{proof}[Proof of Theorem \ref{thm:ogda_primal_rho<mu}]
Plugging in Theorem  \ref{thm:appendix_unify_rho<mu} with $\eta_0 = \frac{1}{2\sqrt{2}\ell}$, $C_1 = 1$ and $C_2 = 5B^2$, we get the conclusion.
\end{proof}
\vspace*{-0.25in}  
\begin{proof}[Proof of Theorem \ref{thm:gda_primal_rho<mu}]
We can plug in $\eta_0 = \frac{1}{\rho}$,  $C_1 = 2$ and $C_2 = 5B^2 $ to Theorem \ref{thm:appendix_unify_3}.
And the conclusion follows.
\end{proof}
\vspace*{-0.3in}  

\section{Analysis of PES-AdaGrad}
In this section, we analyze AdaGrad in solving the strongly convex-strongly concave problem.
Define $\|u\|_H = \sqrt{u^T H u}$, $\psi_0(z) = 0$, $\psi_{T}^{*}$ to be the conjugate of $\frac{1}{\eta} \psi_{T}$, i.e., $\psi_{t}^{*}(z) = \sup\limits_{z'\in \mathcal{Z}} \{\langle z, z'\rangle - \frac{1}{\eta}\psi_{t}(z')\}$.  
We first present a supporting lemma, 
\begin{lemma}
For a sequence $\zeta_1, \zeta_2, ...$, define a sequence $\{u_t \in \mathcal{Z}\}_{t=0}^{T+1}$ as
\begin{align}
\begin{split}
u_{t+1} = \arg\min\limits_{u\in \mathcal{X}} \frac{\eta}{t} \sum\limits_{\tau=1}^{t} \langle \zeta_{\tau}, u\rangle +\frac{1}{t} \psi_{t}(u) , u_0 = z_0,
\end{split}
\end{align} 
where $\psi_{t}(\cdot)$ is defined in Algorithm \ref{alg:subroutine} with Option III: AdaGrad.
Then for any $u\in \mathcal{Z}$, 
\begin{align}
\begin{split} 
\sum\limits_{t=1}^{T}\langle \zeta_{t} , u_t - u\rangle \leq \frac{1}{\eta}\psi_{T}(u) + \frac{\eta}{2} \sum\limits_{t=1}^{T}\|\zeta_{t}\|^2_{\psi_{t-1}^{*}}.
\end{split} 
\end{align} 
\label{lem:ada_sequence} 
\end{lemma}

\begin{proof}[Proof of Lemma \ref{lem:ada_sequence}]
\begin{align}
\begin{split} 
\sum\limits_{t=1}^{T} \langle \zeta_{t}, u_t - u\rangle 
& = \sum\limits_{t=1}^{T}\langle \zeta_{t}, u_t \rangle 
- \sum\limits_{t=1}^{T}\langle \zeta_{t}, u \rangle 
- \frac{1}{\eta}\psi_{T}(u) + \frac{1}{\eta} \psi_{T}(u) \\ 
& \leq \frac{1}{\eta} \psi_{T} (x) 
+ \sum\limits_{t=1}^{T}\langle \zeta_{t}, u_t \rangle + \sup\limits_{u\in \mathcal{Z}} 
\bigg\{\left\langle -\sum\limits_{t=1}^{T} \zeta_{t}, u \right\rangle -\frac{1}{\eta} \psi_{T}(u) \bigg\}\\ 
& = \frac{1}{\eta} \psi_{T}(u) 
+ \sum\limits_{t=1}^{T}\langle \zeta_{t}, u_t \rangle 
+ \psi_{T}^{*}
\left(-\sum\limits_{t=1}^{T} \zeta_{t}\right).
\end{split} 
\label{local:ada_sequence_50} 
\end{align}

Note that 
\begin{small} 
\begin{align} 
\begin{split} 
\psi_{T}^{*}\left(-\sum\limits_{t=1}^{T} \zeta_{t}\right) 
& \overset{(a)}{=} \left\langle -\sum\limits_{t=1}^{T} \zeta_{t}, u_{T+1} \right\rangle 
 - \frac{1}{\eta} \psi_{T} (u_{T+1}) 
 \overset{(b)}{\leq} \left\langle -\sum\limits_{t=1}^{T} \zeta_{t}, u_{T+1} \right\rangle 
 - \frac{1}{\eta} \psi_{T-1} (u_{T+1}) \\ 
& \leq \sup\limits_{u\in \mathcal{Z}}\bigg\{\left\langle -\sum\limits_{t=1}^{T}\zeta_{t}, u\right\rangle - \frac{1}{\eta}\psi_{T-1}(u)\bigg\}  
= \psi_{T-1}^{*}\left(-\sum\limits_{t=1}^{T} \zeta_{t} \right) \\
& \overset{(c)}{\leq} \psi_{T-1}^{*}
\left(-\sum\limits_{t=1}^{T-1} \zeta_{t}\right) 
+ \left\langle -\zeta_{T}, \nabla \psi^{*}_{T-1}
\left(-\sum\limits_{t=1}^{T-1} \zeta_{t}\right) \right\rangle + \frac{\eta}{2}\|\zeta_{T}\|^2_{\psi^{*}_{T-1}},
\end{split}
\label{local:ada_sequence_51}
\end{align}
\end{small}
where $(a)$ holds due to the updating rule, $(b)$ holds since $\psi_{t+1}(u)\geq \psi_t(u)$, $(c)$ uses the fact that $\psi_t(u)$ is 1-strongly convex w.r.t. $\|\cdot\|_{\psi_t} = \|\cdot\|_{H_t}$ and hence $\psi_{t}^{*}(\cdot)$ 
is $\eta$-smooth w.r.t. $\|\cdot\|_{\psi_{t}^*}=\|\cdot\|_{(H_{t})^{-1}}$.

Noting $\nabla \psi_{T-1}^{*}\left(-\sum\limits_{t=1}^{T-1}
\zeta_{t}\right) = u_{T}$ and adding $\sum\limits_{t=1}^{T}\langle \zeta_{t}, u_t\rangle$ to both sides of (\ref{local:ada_sequence_51}), 
\begin{align} 
\begin{split}  
\sum\limits_{t=1}^{T}\langle \zeta_{t}, u_t \rangle + \psi_{T}^{*}
\left(-\sum\limits_{t=1}^{T} \zeta_{t}\right) 
\leq \sum\limits_{t=1}^{T-1} \langle \zeta_{t} , u_t\rangle 
+ \psi_{T-1}^{*}\left(-\sum\limits_{t=1}^{T-1} \zeta_{t}\right) 
+ \frac{\eta}{2}\|\zeta_{T}\|^2_{\psi^{*}_{T-1}}.
\end{split}
\label{local:ada_sequence_52} 
\end{align}
Using (\ref{local:ada_sequence_52}) recursively and noting that $\psi_{0}(u) = 0$, we have
\begin{align} 
\begin{split} 
\sum\limits_{t=1}^{T}\langle \zeta_{t}, u_t\rangle + \psi_{x, T}^{*}\left(-\sum\limits_{t=1}^{T} \zeta_{t}\right) \leq \frac{\eta}{2} \sum\limits_{t=1}^{T} \|\zeta_{t}\|^2_{\psi_{t-1}^{*}}.
\end{split} 
\label{local:ada_sequence_53}
\end{align}

Combining (\ref{local:ada_sequence_50}) and (\ref{local:ada_sequence_53}), we have
\begin{align*}
\begin{split} 
\sum\limits_{t=1}^{T}\langle \zeta_{t} , u_t - u\rangle \leq \frac{1}{\eta}\psi_{T}(u) + \frac{\eta}{2} \sum\limits_{t=1}^{T}\|\zeta_{t}\|^2_{\psi_{t-1}^{*}}.
\end{split} 
\end{align*}
\end{proof} 
\begin{lemma}
Suppose $f(x, y)$ is convex-concave.
And also assume $\|\G_{t}\|_{\infty} \leq \delta$. \\
Set $T = M \left\lceil \max\{ \frac{\delta + \max_i \|g_{1:T, i}\|}{m},
m\sum\limits_{i=1}^{d+d'} \|g_{1:T, i}\|\}\right\rceil$.
By running Algorithm \ref{alg:subroutine} with Option III: AdaGrad, with input ($f, x_0, y_0, \eta, T$), we have
\begin{align} 
\begin{split}
E[\emph{\text{Gap}}(\bar{x}, \bar{y})] &\leq  \frac{m}{\eta M}(\|z - z_0\|^2) + \frac{4 \eta}{m M}. 
\end{split} 
\end{align}
\label{lem:adagrad} 
\end{lemma} 
\begin{proof} 
Applying Lemma \ref{lem:ada_sequence} with $\zeta_t = \G_{t}$ and $u_t = z_t$, for any $z\in \mathcal{Z}$, 
\begin{align}
\begin{split} 
\sum\limits_{t=1}^{T}\langle \G_{t} , z_t - z\rangle \leq \frac{1}{\eta}\psi_{T}(z) + \frac{\eta}{2} \sum\limits_{t=1}^{T}\|\G_{t}\|^2_{\psi_{t-1}^{*}}.
\end{split} 
\end{align} 

By Lemma 4 of \citep{duchi2011adaptive}, we know that $\sum\limits_{t=1}^{T}\|\G_{t}\|^2_{\psi_{t-1}^{*}} \leq 2\sum\limits_{i=1}^{d+d'} \|g_{1:T, i}\|$.
Hence, for any $z \in \mathcal{Z}$
\begin{align} 
\begin{split} 
\sum\limits_{t=1}^{T}\langle \G_{t} , z_t - z\rangle &\leq 
\frac{1}{\eta} \psi_{T}(z) +
\eta\sum\limits_{i=1}^{d+d'} \|g_{1:T, i}\|_2 \\
& = \frac{\delta \|z_0 - z\|^2}{2\eta} + \frac{\langle z_0 - z, \text{ diag }(s_{T})(z_0 - z) \rangle}{2\eta} 
+\eta\sum\limits_{i=1}^{d+d'} \|g_{1:T, i}\| \\
& \leq \frac{\delta + \max_i \|g_{1:T, i}\|}{2\eta} \|z_0 - z\|^2 + \eta\sum\limits_{i=1}^{d+d'} 
\|g_{1:T, i}\|.
\end{split} 
\label{local:first_part_55} 
\end{align}

Then, we define the following auxiliary sequence $\{\hat{z}_t \in \mathcal{Z}\}_{t=0}^{T}$,
\begin{align}
\begin{split}
\hat{z}_{t+1} = \arg\min\limits_{z\in \mathcal{Z}} \frac{\eta}{t}\sum\limits_{\tau=1}^{t}
\langle F(z_t) - \G(z_t;\xi_t)), z \rangle
+ \frac{1}{t} \psi_{t}(z), \hat{z}_0 = z_0.
\end{split} 
\end{align} 

Denote $\Theta_{t} = F(z_t) - \G(z_t; \xi_t)$.
Applying Lemma \ref{lem:ada_sequence} with $\zeta_t = \Theta_{t}$ and $u_t = \hat{z}_t$,
we have
\begin{align}
\begin{split} 
\sum\limits_{t=1}^{T}\langle \Theta_{t} , \hat{z}_t - z\rangle &\leq \frac{1}{\eta}\psi_{T}(z) + \frac{\eta}{2} \sum\limits_{t=1}^{T}\|\Theta_{t}\|^2_{\psi_{t-1}^{*}} \\
&\leq \frac{\delta + \max_i \|g_{1:T, i}\|}{2\eta} \|z_0 - z\|^2
+ \frac{\eta}{2} \sum\limits_{t=1}^{T}\|\Theta_{t}\|^2_{\psi_{t-1}^{*}}.
\end{split} 
\label{local:ada_aux_5}
\end{align} 

To deal with the last term in the above inequality, we have in expectation that
\begin{align}
\begin{split} 
&\E\left[\sum\limits_{t=1}^{T}\|\Theta_{t}\|^2_{\psi_{t-1}^{*}} \right]
= \sum\limits_{t=1}^{T} \E\left[\|\Theta_{t}\|^2_{\psi_{t-1}^{*}} \right] \\
&= \sum\limits_{t=1}^{T}  (\E\left[\|\G(z_t;\xi_t)\|^2_{\psi^*_{t-1}}\right] - \|F(z_t)\|_{\psi^*_{t-1}}^2) \\
&\leq \E\left[\sum\limits_{t=1}^{T}\|\G(z_t;\xi_t)\|^2_{\psi_{t-1}^{*}} \right] 
\leq 2\E\left[\sum\limits_{i=1}^{d+d'} \|g_{1:T, i}\| \right],
\end{split} 
\label{local:ada_V^2_to_E^2}
\end{align} 
where the second equality uses the fact that $\E[\G(z_t;\xi_t)] = F(z_t)$ and the last inequality uses Lemma 4 of \citep{duchi2011adaptive}. 

Thus,
\begin{align}
\begin{split}
& \E\left[\frac{1}{T}\sum\limits_{t=1}^{T} \langle F(z_t), z_t - z\rangle\right] 
= \E\left[\frac{1}{T}\sum\limits_{t=1}^{T} \langle \G(z_t;\xi_t), z_t - z\rangle\right] + 
\E\left[\frac{1}{T}\sum\limits_{t=1}^{T} \langle \Theta_{t},  z_t - z\rangle\right] \\
& = \E\left[\frac{1}{T}\sum\limits_{t=1}^{T} \langle \G(z_t;\xi_t), z_t - z\rangle\right] + 
\E\left[\frac{1}{T}\sum\limits_{t=1}^{T} \langle \Theta_{t},  z_t - \hat{z}_t\rangle\right] 
+ \E\left[\frac{1}{T}\sum\limits_{t=1}^{T} \langle \Theta_{t},  \hat{z}_t - z\rangle\right] \\ 
& \overset{(a)}{\leq} E\left[\frac{\delta + \max_i \|g_{1:T, i}\|}{\eta T} \|z_0 - z\|^2\right] + 2\frac{\eta}{T} \E\left[\sum\limits_{i=1}^{d} \|g_{1:T, i}\| \right] 
+ \E\left[\frac{1}{T}\sum\limits_{t=1}^{T} \langle \Theta_{t}, z_t - \hat{z}_t \rangle\right]\\
& = \E\left[\frac{\delta + \max_i \|g_{1:T, i}\|}{\eta T} \|z_0 - z\|^2\right] + 2\frac{\eta}{T} \E\left[\sum\limits_{i=1}^{d+d'} \|g_{1:T, i}\| \right],
\end{split} 
\label{local:ada_grad_x} 
\end{align} 
where the last equality holds because $\E[\langle \Theta_{t}, z_t - \hat{z}_t\rangle| z_{t}, \hat{z}_t, \Theta_{t-1}, ..., \Theta_{0}] = 0$, and $(a)$ uses (\ref{local:first_part_55}), (\ref{local:ada_aux_5}) and (\ref{local:ada_V^2_to_E^2}).
Then for any $x\in \mathcal{X}$ and $y\in \mathcal{Y}$,
\begin{align}
\begin{split}
\E[f(\Bar{x}, y) - f(x, \Bar{y})] \leq & 
\E\left[\frac{1}{T} \sum\limits_{t=1}^T (f(x_t, y) - f(x, y_t)) \right] \\
= & \E\left[\frac{1}{T} \sum\limits_{t=1}^T (f(x_t, y) - f(x_t, y_t) + f(x_t, y_t) - f(x, y_t))\right] \\ 
\leq & \E\left[\frac{1}{T} \sum\limits_{t=1}^T ( \langle - \nabla_y f(x_t, y_t), y_t - y\rangle + \langle \nabla_x f(x_t, y_t), x_t - x\rangle )\right] \\
=& \E \left[ \frac{1}{T}\sum\limits_{t=1}^{T} \langle F(z_t), z_t - z\rangle \right]\\
\overset{(a)}{\leq} & \E\left[\frac{\delta + \max_i \|g_{1:T, i}\|}{\eta T} \|z_0 - z\|^2\right] + 2\frac{\eta}{T} \E\left[\sum\limits_{i=1}^{d+d'} \|g_{1:T, i}\| \right] \\
\overset{(b)}{\leq} & \frac{m}{\eta M}\E[\|x_0 - x\|^2 + \|y_0 -y\|^2] + \frac{4 \eta}{m M},
\end{split} 
\end{align}
where $(a)$ uses (\ref{local:ada_grad_x}), and the last inequality is due to $T \!=\! M \!\left\lceil\max\{ \frac{\delta + \max_i \|g_{1:T, i}\|}{m}, 
m\sum\limits_{i=1}^{d+d'} \|g_{1:T, i}\|
\} \right\rceil$.
Then we can conclude by plugging in $(x, y) = (\hat{x}(\bar{y}), \hat{y}(\bar{x}))$. 
\end{proof} 

Now we formally restate the Theorem \ref{thm:ada_primal} as:
\begin{theorem}[Formal version of Theorem \ref{thm:ada_primal}] 
Suppose Assumption \ref{ass1}, \ref{ass5},  \ref{ass3} hold. Let
$g^k_{1:T_k}$ denote the cumulative matrix of gradients in $k$-th stage. 
Suppose $\|g^k_{1:T_k, i}\|_2 \leq \delta T^{\alpha}_k$ and with $\alpha\in(0,1/2]$. 
Then by setting parameters appropriately, 
$\gamma = 2\rho$, $m = 1/\sqrt{d+d'}$, $\eta_k = 2\eta_0 \exp\left(-\frac{(k-1)}{2}\frac{2\mu}{c+2\mu}\right)$,  $M_k = \frac{212m}{\eta_0 \min(\ell, \mu_y)} \exp\left(\frac{k-1}{2}\frac{2\mu}{c+2\mu}\right)$, and
$T_k =  \bigg\lceil M_k \max\left\{\frac{\delta + \max_i\|g^k_{1:T_k, i}\|_2}{2m},\right. \\  \left.m\sum\limits_{i=1}^{d+d'}\|g^k_{1:T_k, i}\|_2 \right\}\bigg\rceil$, and after $K =\left\lceil \max \left\{ \frac{c+2\mu}{2\mu}\log\left(\frac{4\epsilon_0}{\epsilon}\right), \frac{c+2\mu}{2\mu} \log \left(\frac{16\eta^2_0 \hat{L}  \min(\rho, \mu_y) K}{53 m^2(c+2\mu)\epsilon}\right) \right\} \right\rceil$ stages, we have 
PES-AdaGrad has the total stochastic first-order oracle call complexity of \\
$\widetilde{O}\left( \left(
\frac{\delta^2 (L+\rho)^2 (d+d')}{\mu^2 \min\{\rho, \mu_y\} \epsilon}\right)^{\frac{1}{2(1-\alpha)}} \right)$ in order to have $\E[\Delta_{K+1}] \leq \epsilon$, { where $\Delta_{k}$ is defined as  in Theorem \ref{thm:1}}. 
\end{theorem}
\begin{proof} 
By analysis in proof of Theorem \ref{thm:ogda_primal}, we have the following inequalities that do not depend on the  optimization algorithm 
\begin{align} 
\begin{split} 
\left(1+\frac{\mu}{2\rho}\right)(P(\bx_k) - P(x_*))
- (P(\bx_{k-1}) - P(x_*)) \leq \frac{1}{4\rho}\|\nabla P_k(\bx_k)\|^2,
\end{split} 
\label{local:ada:smooth_weakly}
\end{align} 
\begin{align} 
\begin{split} 
{\text{Gap}}_k(\bx_k, \by_k) 
\geq \frac{3}{50}{\text{Gap}}_{k+1}(x_0^{k+1}, y_0^{k+1})
+ \frac{4}{5}(P(x_0^{k+1}) - P(x_0^k)),
\end{split}
\label{local:G:20_ada}
\end{align}
and 
\begin{align}
\begin{split}
\frac{\rho}{4}\|\hat{x}_k(\by_k) - x_0^k\|^2 
+ \frac{\mu_y}{4} \|\hat{y}_k(\bx_k) - y_0^k\|^2
\leq {\text{Gap}}_k(x_0^k, y_0^k) 
+ {\text{Gap}}_k(\bx_k, \by_k).
\end{split}
\end{align} 
Set $m = 1/\sqrt{d+d'}$, $\eta_k = \eta_0  \exp\left(-\frac{(k-1)}{2}\frac{2\mu}{c+2\mu}\right)$, 
$M_k = \left\lceil\frac{212m }{\eta_0 \min\{\rho, \mu_y\}} 
\exp\left( \frac{(k-1)}{2} \frac{2\mu}{c+2\mu}\right)\right\rceil$.
Note that, 
\begin{align}
\begin{split}
T_k =& \left\lceil M_k \max\left\{ \frac{\delta + \max_i \|g_{1:T, i}\|}{m}, 
m\sum\limits_{i=1}^{d+d'} \|g_{1:T, i}\|
\right\} \right\rceil 
\leq  2\sqrt{d+d'}\delta M_k T_k^{\alpha}. 
\end{split}
\end{align}
Thus, $T_k \leq (2\sqrt{d+d'}\delta M_k) ^{\frac{1}{1-\alpha}}$. 
Noting $\frac{m}{\eta_k M_k} \leq \frac{\min\{\rho, \mu_y\}}{212}$, we can plug in Lemma \ref{lem:adagrad} as
\begin{align}
\begin{split}
\E[{\text{Gap}}_k(\bx_k, \by_k)] \leq 
\E\left[\frac{4\eta_k}{m M_k}\right]
+ \frac{1}{53}\E\left[{\text{Gap}}_k(x_0^k, y_0^k) 
+ {\text{Gap}}_k(\bx_k, \by_k) \right].
\end{split}
\end{align} 

\begin{align*} 
\begin{split}
& \E[\|\nabla P_k(\bx_k)\|^2] \leq 2\hat{L}\E[P_k(\bx_k) - \min\limits_{x\in \mathbb{R}^d}P_k(x)] 
\leq 2\hat{L} \E[{\text{Gap}}_k(\bx_k, \by_k)] \\
& = 2\hat{L}\E[4{\text{Gap}}_k(\bx_k, \by_k) 
    - 3{\text{Gap}}_k(\bx_k, \by_k)] \\
& \leq 2\hat{L} \E\left[ 4\left(\frac{4\eta_k}{m M_k}
+ \frac{1}{53}\left({\text{Gap}}_k(x_0^k, y_0^k) 
+ {\text{Gap}}_k(\bx_k, \by_k) \right) \right)
- 3{\text{Gap}}_k(\bx_k, \by_k)\right] \\ 
& = 2\hat{L} \E \left[ 16 \frac{\eta_k}{m M_k} 
+  \frac{4}{53}{\text{Gap}}_k(x_0^k, y_0^k) 
- \frac{155}{53} {\text{Gap}}_k(\bx_k, \by_k)
\right]\\ 
&\leq 2\hat{L} \E \left[ 16\frac{\eta_k}{m M_k} +  \frac{4}{53}{\text{Gap}}_k(x_0^k, y_0^k) 
-  \frac{93}{530}{\text{Gap}}_{k+1}(x_0^{k+1}, y_0^{k+1})
- \frac{124}{53}(P(x_0^{k+1}) - P(x_0^k)) \right], 
\end{split}
\end{align*}
where the last inequality uses (\ref{local:G:20_ada}).
Combining this with (\ref{local:ada:smooth_weakly}) and arranging terms, with a constant $c = 4\rho + \frac{248}{53}\hat{L}$, we have
\begin{align} 
\begin{split} 
& (c + 2\mu) \E[P(x_0^{k+1}) - P(x_*)]
+ \frac{93\hat{L}}{265}\E[{\text{Gap}}_{k+1}(x_0^{k+1}, y_0^{k+1})] 
\\ 
&\leq c\E[ P(x_0^k) - P(x_*)] + \frac{8\hat{L}}{53}\E [{\text{Gap}}_k(x_0^k, y_0^k)] + \frac{32 \eta_k \hat{L}}{m M_k}.
\end{split} 
\end{align}
Then using the fact that $\hat{L} \geq \mu$, by similar analysis as in proof of Theorem \ref{thm:ogda_primal}, we have
\begin{align} 
\begin{split}
&(c+2\mu)\E \left[P(x_0^{k+1}) - P(x_*) + \frac{8\hat{L}}{53c} {\text{Gap}}_{k+1}(x_0^k, y_0^k)\right] \\
&\leq c \E\left[P(x_0^k) - P(x_*)  
+ \frac{8\hat{L}}{53c}{\text{Gap}}_k(x_0^k, y_0^k) \right] 
+ \frac{32 \eta_k \hat{L}}{m M_k}. 
\end{split}
\end{align} 

Defining $\Delta_k = P(x_0^k) - P(x_*) + \frac{8\hat{L}}{53c}{\text{Gap}}_k(x_0^k, y_0^k)$  and $\epsilon_0={\text{Gap}}(x_0, y_0)$, then
\begin{align} 
\begin{split}
\E[\Delta_{k+1}] \leq \frac{c}{c+2\mu} \E[\Delta_k] 
+ \frac{32 \eta_k \hat{L}}{(c+2\mu) m M_k}.  
\end{split}
\end{align}

Noting $\Delta_1 \leq 2\epsilon_0$ and $(1-x) \leq \exp(-x)$, 
\begin{small}
\begin{align} 
\begin{split} 
&\E [\Delta_{k+1}]  \leq  \left(\frac{c}{c+2\mu}\right)^K \E[\Delta_1 ]
+ \frac{32\hat{L}}{(c+2\mu)m}  
\sum\limits_{k=1}^{K}\frac{\eta_k}{M_k} \left( \frac{c}{c+2\mu} \right)^{K+1-k}\\
& \leq 2\epsilon_0 \exp\left(\frac{-2\mu K}{c+2\mu}\right) 
  + \frac{32\hat{L} \eta_0^2 \min(\rho, \mu_y)}{212m^2 (c+2\mu)} \sum\limits_{k=1}^{K}
  \exp\left((k-1) \frac{2\mu}{c+2\mu} \right)
  \exp\left(-\frac{2\mu (K+1-k)}{c+2\mu}\right) \\ 
& \leq 2\epsilon_0 \exp\left(\frac{-2\mu K}{c+2\mu}\right) 
  + \frac{8\eta^2_0 \hat{L}\min(\rho, \mu_y)}{53m^2 (c+2\mu)} K   \exp\left(-\frac{2\mu K}{c+2\mu}\right). 
\end{split}
\end{align}
\end{small}

To make this less than $\epsilon$, we just need to make
\begin{align}
\begin{split}
&2\epsilon_0 \exp\left(\frac{-2\mu K}{c+2\mu}\right)  \leq \frac{\epsilon}{2}, \\
&\frac{8\eta^2_0 \hat{L}\min(\rho, \mu_y)}{53 m^2 (c+2\mu)} K  
\exp\left(-\frac{2\mu K}{c+2\mu}\right)  \leq \frac{\epsilon}{2}.
\end{split}
\end{align} 
Let $K$ be the smallest integer such that  
$\exp\left(\frac{-2\mu K}{c+2\mu}\right) \leq \min\{\frac{\epsilon}{4\epsilon_0}, 
\frac{53 m^2 (c+2\mu)\epsilon}{16\eta_0^2 
\hat{L}\min(\rho, \mu_y)K} 
\}$.
We can set $K = \left\lceil\max \left\{ \frac{c+2\mu}{2\mu}\log\left(\frac{4\epsilon_0}{\epsilon}\right), 
\frac{c+2\mu}{2\mu} \log \left(\frac{16\eta^2_0 \hat{L}
\min(\rho , \mu_y) K}{53 m^2 (c+2\mu)\epsilon}\right)
\right\}\right\rceil$.  
Recall 
\begin{align}
\begin{split}
T_k \leq  (2 \sqrt{d+d'} \delta M_k)^{\frac{1}{1-\alpha}}
\leq  \left[
\frac{424 \delta}{\eta_0 \min\{\rho, \mu_y\}}\exp\left(\frac{(k-1)}{2} \frac{2\mu}{c+2\mu}\right)\right]^{\frac{1}{1-\alpha}}.
\end{split}
\end{align}

Then the total number of stochastic first-order oracle calls is 
\begin{align*}
\begin{split}
&\sum\limits_{k=1}^{K} T_k \leq {O}
\left( \sum\limits_{k=1}^{K} 
\left[\frac{\delta}{\eta_0 
\min\{\rho, \mu_y\}} \exp\left(\frac{(k-1)}{2} \frac{2\mu}{c+2\mu}\right) \right]^{\frac{1}{1-\alpha}} \right) \\
&\leq {O} \left( \sum\limits_{k=1}^{K} 
\left(\frac{\delta}{\eta_0 \min\{\rho, \mu_y\}}\right)^{\frac{1}{1-\alpha}} \exp\left(\frac{k-1}{2(1-\alpha)} \frac{2\mu}{c+2\mu}\right) \right)\\
& \leq {O} \left(
\left(\frac{\delta}{\eta_0 \min\{\rho, \mu_y\}}\right)^{\frac{1}{1-\alpha}}
\frac{\exp\left(K\frac{2\mu}{2(1-\alpha)(c+2\mu)} - 1\right)} {\exp\left(\frac{2\mu}{2(1-\alpha)(c+2\mu)}\right) - 1} \right)\\
& \overset{(a)}{\leq}  {O} \left(
\left(\frac{\delta}{\eta_0 \min\{\rho, \mu_y\}}\right)^{\frac{1}{1-\alpha}}
\left(\frac{c+2\mu}{2\mu}\right)^{\frac{1}{2(1-\alpha)}}  
\left(\max\left\{\frac{4\epsilon_0}{\epsilon}, \frac{16\eta_0^2 \hat{L} \min(\rho, \mu_y) K}
{53\epsilon m^2 ( c+\mu)} \right\}\right)^{\frac{1}{2(1-\alpha)}} 
\right) \\
&\leq \widetilde{O}\left( \left(\max\left\{
\frac{\delta^2c}{\eta_0^2\mu(\min\{\rho, \mu_y\})^2},
\frac{\delta^2 \hat{L}c (d+d')}{\mu^2 \min\{\rho, \mu_y\} \epsilon} \right\} \right)^{\frac{1}{2(1-\alpha)}} \right) \\
&\leq \widetilde{O}\left( \left(
\frac{\delta^2 (L+\rho)^2 (d+d')}{\mu^2 \min\{\rho, \mu_y\} \epsilon} \right)^{\frac{1}{2(1-\alpha)}} \right),
\end{split} 
\end{align*} 
where $(a)$ uses the inequality that $\exp(ax) - 1 \geq x^a$ for any $0< a <1$ and $x > 0$, noting that $0< \frac{2\mu}{c+2\mu} < 1$ and $\frac{1}{2(1-\alpha)} > 0$. 
\end{proof}

\section{More Analysis on PES-AdaGrad} 
We have already shown in Theorem \ref{thm:ada_primal} about the convergence of primal gap for our Algorithm with Option III: Adagrad update.
In this section, we show a corollary about the convergence of duality gap based on Theorem \ref{thm:ada_primal}.
What is more, in parallel with our analysis on Option II: OGDA update, we show some convergence results under the condition that $\rho\leq \frac{\mu}{8}$.

\begin{corollary}
Under same setting as in Theorem \ref{thm:ada_primal} and suppose Assumption \ref{ass4} holds as well.
To reach an $\epsilon$-duality gap, the total stochastic first-order oracle call complexity  is $\widetilde{O}\left(   \left(\left(\frac{\rho}{\mu_x}+1\right) 
\frac{\delta^2 (L+\rho)^2 (d+d')}{\mu^2 \min\{\rho, \mu_y\} \epsilon}\right)^{\frac{1}{2(1-\alpha)}} 
\right)$.
\label{cor:ada_dual} 
\end{corollary}

\begin{theorem}
Suppose Assumption \ref{ass1}, \ref{ass5}, \ref{ass4}, hold and $\rho\leq \frac{\mu}{8}$. 
Define a constant $c = 4\rho + \frac{248}{53}\hat{L} \in O(L + \rho)$.
$g^k_{1:T_k}$ denotes the cumulative matrix of gradients $g_{1:T}$ in $k$-th stage. 
Suppose $\|g^k_{1:T_k, i}\|_2 \leq \delta T^{\alpha}_k$ and with $\alpha\in(0,1/2]$. 
Then by setting $\gamma = 2\rho$, $m = 1/\sqrt{d+d'}$, $\eta_k = 2\eta_0 \exp\left(-\frac{(k-1)}{2}\frac{2\mu}{c+2\mu}\right)$,  $M_k = \frac{212m}{\eta_0 \min(\ell, \mu_y)} \exp\left(\frac{k-1}{2}\frac{2\mu}{c+2\mu}\right)$, and $T_k =  \left\lceil M_k \max\left\{\frac{\delta + \max_i\|g^k_{1:\tau, i}\|_2}{2m}, m\sum\limits_{i=1}^{d+d'}\|g^k_{1:\tau, i}\|_2 \right\}\right\rceil$, and after $K = \bigg\lceil\max \left\{ \frac{c+2\mu}{2\mu}\log\left(\frac{4\epsilon_0}{\epsilon}\right), \right. \\ 
\left.\frac{c+2\mu}{2\mu} \log \left(\frac{16\eta^2_0 \hat{L}
\min(\rho, \mu_y) K}{53 m^2(c+2\mu)\epsilon}\right) 
\right\} \bigg\rceil$ stages, we have  
$\Tilde{O}\left( \left( 
\frac{\delta^2 (d+d') }{\min\{\mu, \mu_y\} \epsilon}  
\right)^{\frac{1}{2(1-\alpha)}} \right)$.
\label{thm:ada_primal_rho<mu}
\end{theorem}

\begin{corollary} 
Under same setting as in Theorem \ref{thm:ada_primal_rho<mu} and suppose Assumption \ref{ass4} holds as well.
To reach an $\epsilon$-duality gap, the total stochastic first-order oracle call complexity  is 
$\Tilde{O}\left(\left(  \left(\frac{\mu}{\mu_x}+1\right) 
\frac{\delta^2 (d+d') }{\min\{\mu, \mu_y\} \epsilon} 
\right)^{\frac{1}{2(1-\alpha)}} \right)$.
\label{cor:ada_dual_rho<mu} 
\end{corollary}

\begin{proof}[Proof of Theorem \ref{thm:ada_primal_rho<mu} ]
By analysis in the Proof of Theorem  \ref{thm:ogda_primal_rho<mu}, we know that when $\rho < \frac{\mu}{8}$ and $\gamma = \frac{\mu}{4}$,
\begin{align}
\begin{split}
P(x_0^k) - P(x_*) \leq \frac{16}{15}{\text{Gap}}_k(x_0^k, y_0^k)
\end{split}
\end{align}
and ${f}_k(x, y)$ is $\lambda_x = \frac{\mu}{8}$-strongly convex in $x$.

Set $m = 1/\sqrt{d+d'}$, $\eta_k = \eta_0  \exp\left(-\frac{(k-1)}{32}\right)$, 
$M_k = \frac{384 m }{\eta_0 \min\{\lambda_x, \mu_y\}} 
\exp\left( \frac{(k-1)}{32} \right)$.
Note that,
\begin{align}
\begin{split}
T_k =&  \left\lceil M_k \max\left\{ \frac{\delta + \max_i \|g_{1:T, i}\|}{m}, 
m\sum\limits_{i=1}^{d+d'} \|g_{1:T, i}\|
\right\} \right\rceil 
\leq  2\sqrt{d+d'} \delta M_k T_k^{\alpha}.
\end{split}
\end{align}
Thus, $T_k \leq (2\sqrt{d+d'} \delta M_k) ^{\frac{1}{1-\alpha}}$. 
Since $\frac{m}{\eta_k M_k} \leq \frac{\min\{\lambda_x, \mu_y\}}{384}$, we can apply Lemma \ref{lem:adagrad} to get
\begin{align}
\begin{split}
\E[{\text{Gap}}_k(\bx_k, \by_k)] &\leq \frac{4\eta_k}{m M_k} 
+ \frac{1}{96} \left(\frac{\lambda_x}{4}\E[\|\hat{x}_k(\by_k) - x_0^k\|^2] + \frac{\mu_y}{4}\E[\|\hat{y}_k(\bx_k) - y_0^k\|^2]\right) \\ 
&\leq \frac{4\eta_k}{m M_k} 
+ \frac{1}{96} \E[{\text{Gap}}_k(x_0^k, y_0^k)]
+ \frac{1}{96} \E[{\text{Gap}}_k(\bx_k, \by_k)],
\end{split}    
\end{align}
where the last inequality follows from Lemma \ref{lem:Yan1}. Rearranging terms, we have 
\begin{align} 
\begin{split} 
\frac{95}{96} \E[{\text{Gap}}_k (\bx_k, \by_k)] 
\leq \frac{4\eta_k}{m M_k} 
+ \frac{1}{96} \E[{\text{Gap}}_k(x_0^k, y_0^k)].
\end{split}
\label{local:ada_rho<mu_gap_1}
\end{align}

Since $\rho \leq \frac{\mu}{8}$, $f(x, y)$ is also  $\frac{\mu}{8}$-weakly convex in $x$.
Then, similar to the analysis in proof of Theorem \ref{thm:ogda_primal_rho<mu}, we use Lemma \ref{lem:Yan8} to lower bound the LHS of 
(\ref{local:ada_rho<mu_gap_1}) with $\alpha = \frac{5}{6}$,
\begin{align} 
\begin{split} 
\frac{95}{96} {\text{Gap}}_k(\bx_k, \by_k) \geq& 
\frac{57}{576}{\text{Gap}}_{k+1}(x_0^{k+1}, y_0^{k+1})
+ \frac{475}{576}(P(x_0^{k+1}) - P(x_*)) \\
& - \frac{475}{576}\cdot \frac{15}{16} (P(x_0^k) - P(x_*))
- \frac{475}{576} \cdot   \frac{1}{15}{\text{Gap}}_k(x_0^k, y_0^k)
\end{split}
\label{local:ada_rho<mu_gap_2}
\end{align}
Combining (\ref{local:ada_rho<mu_gap_1}) and (\ref{local:ada_rho<mu_gap_2}), we get
\begin{align*} 
\begin{split} 
&\E\left[\frac{475}{576}(P(x_0^{k+1}) - P(x_*))  
+ \frac{57}{576} {\text{Gap}}_{k+1}(x_0^{k+1}, y_0^{k+1}) \right]\\ 
&\leq \frac{4\eta_k}{m M_k} 
+ \frac{475}{576} \cdot \frac{15}{16} \E[P(x_0^k) - P(x_*)]
+ \frac{475}{576}\cdot \frac{1}{15} \E[{\text{Gap}}_k(x_0^k, y_0^k) ]
+ \frac{1}{96} \E[{\text{Gap}}_k(x_0^k, y_0^k)]\\
&\leq \frac{4\eta_k}{m M_k} 
+ \frac{15}{16} \E \left[\frac{475}{576} (P(x_0^k) - P(x_*))
+ \frac{57}{576} {\text{Gap}}_k(x_0^k, y_0^k) \right].
\end{split}
\end{align*}
Defining $\Delta_k = 475(P(x_0^k) - P(x_*)) + 57{\text{Gap}}_k(x_0^k, y0^k)$  and $\epsilon_0={\text{Gap}}(x_0, y_0)$, we have
\begin{align}
\begin{split}
\E[\Delta_{k+1}] \leq \frac{15}{16} \E[\Delta_k] + \frac{4\eta_k}{m M_k} \leq \exp\left(-\frac{1}{16}\right)\E[\Delta_k] +  \frac{4\eta_k}{m M_k} 
\end{split}
\end{align}
and $\Delta_1 \leq 600\epsilon_0$.
Thus, 
\begin{align}
\begin{split}
\E[\Delta_{K+1}] \leq &  \exp\left(-\frac{K}{16}\right)\Delta_1
+ \frac{4}{m}\sum\limits_{k=1}^{K} \frac{\eta_k}{M_k} \exp\left(-\frac{K+1-k}{16}\right) \\
= & \exp\left(-\frac{K}{16}\right)\Delta_1 
+ \frac{\eta_0^2 \min\{\lambda_x, \mu_y\}}{96 m^2} 
\sum\limits_{k=1}^{K} \exp\left(-\frac{K}{16}\right) \\
\leq & 600\epsilon_0 \exp\left(-\frac{K}{16}\right) 
+ \frac{\eta_0^2 \min\{\lambda_x, \mu_y\}}{96 m^2} 
K \exp\left(-\frac{K}{16}\right).
\end{split}
\end{align} 
To make this less than $\epsilon$, we just need to make
\begin{align}
\begin{split}
&600\epsilon_0 \exp\left(-\frac{K}{16}\right)  \leq \frac{\epsilon}{2}, \\
&\frac{\eta_0^2 \min\{\rho, \mu_y\}}{96 m^2} 
K \exp\left(-\frac{K}{16}\right) \leq \frac{\epsilon}{2}.
\end{split}
\end{align}
Let $K$ be the smallest integer such that $\exp\left(\frac{-K}{16}\right) \leq \min\{\frac{\epsilon}{1200\epsilon_0}, 
\frac{48m^2 \epsilon}{\eta_0^2 \min\{\rho, \mu_y\}K}\}$. 
Recall $T_k \leq  (2\sqrt{d+d'}\delta M_k) ^{\frac{1}{1-\alpha}}  =  \left(\frac{768\delta}{\eta_0
\min\{\lambda_x, \mu_y\}} \exp\left(\frac{k-1}{32} \right)\right)^{\frac{1}{1-\alpha}}$.  
Then the total stochastic first-order oracle call complexity is
\begin{align*}
\begin{split}
\sum\limits_{k=1}^{K} T_k\leq & O 
\left(\sum\limits_{k=1}^{K} \left[\frac{\delta}{\eta_0 \min\{\lambda_x, \mu_y\}} \exp \left(\frac{k-1}{32}\right) \right]^{\frac{1}{1-\alpha}} \right) \\
\leq & O\left( \sum\limits_{k=1}^{K} 
\left(\frac{\delta}{\eta_0 \min\{\lambda_x,  \mu_y\}}\right)^{\frac{1}{1-\alpha}} 
\exp \left(\frac{k-1}{32(1-\alpha)}\right) \right)\\ 
\leq& O\left(\left(\frac{\delta}{\eta_0 \min\{\lambda_x,  \mu_y\}}\right)^{\frac{1}{1-\alpha}} 
\frac{\exp\left(\frac{K}{2(1-\alpha)\cdot 16}\right) - 1}
{\exp\left(\frac{1}{2(1-\alpha) \cdot 16}\right) - 1} \right) \\
\leq & \Tilde{O} \left(\left(\frac{\delta}{\eta_0 \min\{\lambda_x,  \mu_y\}}\right)^{\frac{1}{1-\alpha}} \left( \max\left\{ 
\frac{\epsilon_0}{\epsilon},
\frac{\eta_0^2 \min\{\lambda_x, \mu_y\} K}{m^2 \epsilon} \right\}\right)^{\frac{1}{2(1-\alpha)}} \right)  \\
\leq & \Tilde{O}\left( \left(  \max\left\{
\frac{\delta^2 \epsilon_0}{\eta_0^2 (\min\{\mu, \mu_y\})^2\epsilon}
,
\frac{\delta^2 (d+d') }{\min\{\mu, \mu_y\} \epsilon} \right\}
\right)^{\frac{1}{2(1-\alpha)}} \right) \\
\leq & \Tilde{O}\left( \left(
\frac{\delta^2 (d+d') }{\min\{\mu, \mu_y\} \epsilon}
\right)^{\frac{1}{2(1-\alpha)}} \right). 
\end{split} 
\end{align*}
\end{proof}

\section{Proof of Corollary \ref{cor:ogda_duality},
\ref{cor:gda_duality}, \ref{cor:ada_dual}
} 
\begin{proof}  
Let $(x_*, y_*)$ denote a saddle point solution of $\min\limits_{x\in \mathbb{R}^d} \max\limits_{y\in \mathcal{Y}} f(x, y)$.

Note that $x_0^{K+1} = x_{K}, y_0^{K+1} = \by_K$.
Suppose we have $\E[{\text{Gap}}_{K+1}(x_0^{K+1}, y_0^{K+1})] \leq \hat{\epsilon}$ after $K$ stages.
Noting $\gamma = 2\rho$,
$f_k(x, y)$ is $\rho$-strongly convex and $\mu_y$-strongly concave.
By Lemma \ref{lem:Yan1}, we know that
\begin{align} 
\begin{split} 
\E[\|\hat{x}_{K+1}(y_0^{K+1}) - x_0^{K+1}\|^2] \leq \frac{4}{\rho} 2 \E[{\text{Gap}}_{K+1}(x_0^{K+1}, y_0^{K+1})] \leq \frac{8\hat{\epsilon}}{\rho}. 
\end{split} 
\end{align} 

Since $\nabla_x {f}_{K+1}(\hat{x}_{K+1}(y_0^{K+1}), y_0^{K+1}) = \nabla_x f(\hat{x}_{K+1}(y_0^{K+1}), y_0^{K+1}) + \gamma (\hat{x}_{K+1}(y_0^{K+1}) - x_0^{K+1}) = 0$, we have 
\begin{align*}
\begin{split}
\E[\|\nabla_x f(\hat{x}_{K+1}(y_0^{K+1}), y_0^{K+1})\|^2] = \gamma^2 \E[\|\hat{x}_{K+1}(y_0^{K+1}) - x_0^{K+1}\|^2] \leq 32\rho\hat{\epsilon} 
\end{split}
\end{align*}

Using the $\mu_x$-PL condition of $f(\cdot, y_0^{K+1})$ in $x$,
\begin{align*} 
\begin{split} 
\E \left[ f(\hat{x}_{K+1}(y_0^{K+1}), y_0^{K+1}) - f(\hat{x}(y_0^{K+1}),  y_0^{K+1}) \right] 
 \leq \E\left[ \frac{\|\nabla_x f(\hat{x}_{K+1}(y_0^{K+1}), y_0^{K+1})\|^2}{2\mu_x} \right] \leq \frac{16\rho\hat{\epsilon}}{\mu_x}.
\end{split}
\end{align*}

Hence,
\begin{small} 
\begin{align*}
\begin{split} 
&\E \left[{\text{Gap}}(x_0^{K+1}, y_0^{K+1})\right] = \E\left[ f(x_0^{K+1}, \hat{y}(x_0^{K+1})) - f(\hat{x}(y_0^{K+1}), y_0^{K+1}) \right] \\ 
& = \E\bigg\{ f(x_0^{K+1}, \hat{y}(x_0^{K+1}))+\frac{\gamma}{2} \|x_0^{K+1} - x_0^{K+1}\|^2 - f(\hat{x}_{K+1}(y_0^{K+1}), y_0^{K+1}) - \frac{\gamma}{2} \|\hat{x}_{K+1}(y_0^{K+1}) - x_0^{K+1}\|^2 \\
&~~~+ f(\hat{x}_{K+1}(y_0^{K+1}), y_0^{K+1})- f(\hat{x}(y_0^{K+1}), y_0^{K+1}) + \frac{\gamma}{2} \|\hat{x}_{K+1}(y_0^{K+1}) - x_0^{K+1}\|^2 \bigg\} \\ 
& = \E[{\text{Gap}}_{K+1}(x_0^{K+1}, y_0^{K+1})] +  \E[f(\hat{x}_{K+1}(y_0^{K+1}), y_0^{K+1})- f(\hat{x}(y_0^{K+1}), y_0^{K+1})] \\
&~~~ + \frac{\gamma}{2} \E[\|\hat{x}_{K+1}(y_0^{K+1}) - x_0^{K+1}\|^2] \\
& \leq \hat{\epsilon} +  \frac{16 \rho \hat{\epsilon}}{\mu_x} + 8\hat{\epsilon} \leq O\left(\frac{\rho\hat{\epsilon}}{\mu_x} +  \hat{\epsilon}\right).
\end{split} 
\end{align*} 
\end{small}

To have ${\text{Gap}}(x_0^{K+1}, y_0^{K+1})  \leq \epsilon$, we need $\hat{\epsilon} \leq O\left(\left(\frac{\rho}{\mu_x} + 1\right)^{-1}\epsilon\right)$. 
Plug $\hat{\epsilon}$ into Theorem \ref{thm:ogda_primal},  \ref{thm:ada_primal}, \ref{thm:gda_primal},
we can prove Corollary \ref{cor:ogda_duality},  \ref{cor:ada_dual}, \ref{cor:gda_duality}, respectively. 
\end{proof} 
\vspace{-0.2in} 
\section{Proof of Corollary \ref{cor:ogda_dual_rho<mu}, \ref{cor:gda_duality_rho<mu}, \ref{cor:ada_dual_rho<mu}}
\begin{proof}  
Let $(x_*, y_*)$ denote a saddle point solution of  $\min\limits_{x\in \mathbb{R}^d} \max\limits_{y\in \mathcal{Y}} f(x, y)$ and $x_{K+1}^* = \min\limits_{x\in \mathbb{R}^d} P_k(x)$.
Note that $x_0^{K+1} = x_{K}, y_0^{K+1} = \by_K$.
Suppose $\E[{\text{Gap}}_{K+1}(x_0^{K+1}, y_0^{K+1})] \leq \hat{\epsilon}$ after $K$ stages.

By the setting $\rho \leq \frac{\mu}{8}$ and $\gamma = \frac{\mu}{4}$,
$f_k(x, y)$ is $\frac{\mu}{8}$-strongly convex and $\mu_y$-strongly concave. 
By Lemma \ref{lem:Yan1}, we know that 
\begin{align} 
\begin{split} 
\E[\|\hat{x}_{K+1}(y_0^{K+1}) - x_0^{K+1}\|^2] \leq \frac{64}{\mu} \E[{\text{Gap}}_{K+1}(x_0^{K+1}, y_0^{K+1})] \leq \frac{64\hat{\epsilon}}{\mu}. 
\end{split} 
\end{align} 
Since $\nabla_x {f}_{K+1}(\hat{x}_{K+1}(y_0^{K+1}), y_0^{K+1}) = \nabla_x f(\hat{x}_{K+1}(y_0^{K+1}), y_0^{K+1}) + \gamma (\hat{x}_{K+1}(y_0^{K+1}) - x_0^{K+1}) = 0$, we have 
\begin{align}
\begin{split}
\E[\|\nabla_x f(\hat{x}_{K+1}(y_0^{K+1}), y_0^{K+1})\|^2] = \gamma^2 \E[\|\hat{x}_{K+1}(y_0^{K+1}) - x_0^{K+1}\|^2] \leq 4\mu\hat{\epsilon}. 
\end{split} 
\end{align}

Using the $\mu_x$-PL condition of $f(\cdot, y_0^{K+1})$ in $x$,
\begin{align*} 
\begin{split} 
\E \left[ f(\hat{x}_{K+1}(y_0^{K+1}), y_0^{K+1}) - f(\hat{x}(y_0^{K+1}),  y_0^{K+1}) \right] 
 \leq \E\left[ \frac{\|\nabla_x f(\hat{x}_{K+1}(y_0^{K+1}), y_0^{K+1})\|^2}{2\mu_x} \right] \leq  \frac{2\mu\hat{\epsilon}}{\mu_x}. 
\end{split}
\end{align*}

Hence,
\begin{align*}
\begin{split} 
&\E \left[{\text{Gap}}(x_0^{K+1}, y_0^{K+1})\right] = \E\left[ f(x_0^{K+1}, \hat{y}(x_0^{K+1})) - f(\hat{x}(y_0^{K+1}), y_0^{K+1}) \right] \\ 
& = \E\bigg\{ f(x_0^{K+1}, \hat{y}(x_0^{K+1}))+\frac{\gamma}{2} \|x_0^{K+1} - x_0^{K+1}\|^2 - f(\hat{x}_{K+1}(y_0^{K+1}), y_0^{K+1}) - \frac{\gamma}{2} \|\hat{x}_{K+1}(y_0^{K+1}) - x_0^{K+1}\|^2 \\
&~~~+ f(\hat{x}_{K+1}(y_0^{K+1}), y_0^{K+1})- f(\hat{x}(y_0^{K+1}), y_0^{K+1}) + \frac{\gamma}{2} \|\hat{x}_{K+1}(y_0^{K+1}) - x_0^{K+1}\|^2 \bigg\} \\ 
& = \E[{\text{Gap}}_{K+1}(x_0^{K+1}, y_0^{K+1})] +  \E[f(\hat{x}_{K+1}(y_0^{K+1}), y_0^{K+1})- f(\hat{x}(y_0^{K+1}), y_0^{K+1})] \\
&~~~ + \frac{\gamma}{2} \E[\|\hat{x}_{K+1}(y_0^{K+1}) - x_0^{K+1}\|^2] \\ 
& \leq \hat{\epsilon} +  \frac{2 \mu \hat{\epsilon}}{\mu_x} + 8\hat{\epsilon} \leq O\left(\frac{\mu\hat{\epsilon}}{\mu_x} +   \hat{\epsilon}\right).
\end{split} 
\end{align*} 

To have ${\text{Gap}}(x_0^{K+1}, y_0^{K+1})  \leq \epsilon$, we need $\hat{\epsilon} \leq O\left(\left(\frac{\mu}{\mu_x} +  1\right)^{-1}\epsilon\right)$. 
Plug $\hat{\epsilon}$ into Theorem\ref{thm:2}, \ref{thm:gda_primal_rho<mu}, \ref{thm:ada_primal_rho<mu}
we can prove Corollary \ref{cor:ogda_dual_rho<mu},  \ref{cor:gda_duality_rho<mu}, \ref{cor:ada_dual_rho<mu}, respectively.
\end{proof} 

\section{Analysis of Option IV: PES-Storm} 
In this section, we present the formal version of Theorem \ref{thm:storm_primal_informal} in the Theorem \ref{thm:storm_primal} and show its proof. 
Denote $d_t = (v_t, u_t)$, where the component $v_t$ is corresponding to primal variable $x$ and the component $u_t$ is corresponding to dual variable. Also denote $\eta = (\eta^x, \eta^y)$, $a=(a_x, a_y)$.

\subsection{Auxiliary Lemmas}
In this subsection, we show some lemmas that are needed to prove Theorem \ref{thm:storm_primal}. 

\begin{lemma}
In Algorithm \ref{alg:subroutine} with Option IV: Storm. 
setting $0<\eta^x \leq \frac{1}{2L}$, we have 
\begin{equation}
\begin{split} 
P(x_{t+1}) - P(x_{t}) \leq -\frac{\eta^x}{4} \|v_t\|^2 + \eta^x \ell^2\|\hat{y}(x_t)-y_t\|^2 + \eta^x \|\nabla_x f(x_t, y_t) - v_t\|^2. 
\end{split} 
\end{equation} 
\label{lem:storm_P}
\end{lemma} 

\begin{proof}
Using the $L$-smoothness of $P(x) = \max\limits_{y' \in \mathcal{Y}} f(x, y')$, 
\begin{equation*} 
\begin{split} 
P(x_{t+1}) &\leq P(x_t) + \langle \nabla P(x_t), x_{t+1}-x_t\rangle
 + \frac{L}{2}\|x_{t+1} - x_t\|^2 \\
&= P(x_t) + \langle \nabla P(x_t) - \nabla_x f(x_t, y_t), x_{t+1}-x_t\rangle + \langle \nabla_x f(x_t, y_t) - v_t, x_{t+1} - x_t\rangle \\
& ~~~ + \langle v_t, x_{t+1} - x_t\rangle + \frac{L (\eta^x)^2}{2} \|v_t\|^2 \\
& \leq P(x_t) + \eta^x \|\nabla P(x_t) - \nabla_x f(x_t, y_t)\|^2
+\frac{1}{4 \eta^x} \|x_{t+1} - x_t\|^2 \\
& ~~~ + \eta^x \| \nabla_x f(x_t, y_t) - v_t \|^2 + \frac{1}{4\eta^x} \|x_{t+1}-x_t\|^2+ \langle v_t, x_{t+1} - x_t\rangle + \frac{L (\eta^x)^2}{2} \|v_t\|^2 \\
& = P(x_t) + \eta^x \|\nabla P(x_t) - \nabla_x f(x_t, y_t)\|^2
    + \frac{\eta^x}{4}\|v_t\|^2 + \eta^x \| \nabla_x f(x_t, y_t) - v_t \|^2 \\
&~~~  + \frac{\eta^x}{4}\|v_t\|^2 
    - \eta^x \|v_t\|^2 + \frac{L (\eta^x)^2}{2} \|v_t\|^2 \\
&\leq P(x_t) + \eta^x \ell^2 \|y_t - \hat{y}(x_t)\|^2 
+ \eta^x \| \nabla_x f(x_t, y_t) - v_t \|^2 
- \frac{\eta^x}{4}\|v_t\|^2,
\end{split} 
\label{primal_recursion} 
\end{equation*} 
where the last inequality uses the setting $\eta^x \leq \frac{1}{2L}$. 
\end{proof}
\begin{lemma} 
\label{lem:storm_var_recur}
In Algorithm \ref{alg:subroutine} with Option IV: Storm, setting $0< a_x, a_y < 1$, we have 
\begin{equation*} 
\begin{split} 
&\E\|\nabla_x f(x_{t+1}, y_{t+1}) - v_{t+1}\|^2 \\ 
&\leq (1-a_x)\E\|\nabla_x f(x_t, y_t) - v_t \|^2 +   8(1-a_x)^2\ell^2 (\|x_{t+1} - x_t\|^2 + \|y_{t+1}-y_t\|^2) 
+ 2a_x^2 \sigma^2,
\end{split}
\end{equation*}
and
\begin{equation*}
\begin{split} 
&\E\|\nabla_y f(x_{t+1}, y_{t+1}) - u_{t+1}\|^2 \\ 
&\leq (1-a_y)\E\|\nabla_y f(x_t, y_t) - u_t \|^2 +   8(1-a_y)^2\ell^2 (\|x_{t+1} - x_t\|^2 + \|y_{t+1}-y_t\|^2) 
+ 2a_y^2 \sigma^2. 
\end{split} 
\end{equation*}
\end{lemma}
\begin{proof} 
By the update rule of $v$, we get
\begin{equation*} 
\begin{split}
&\E\|\nabla_x f(x_{t+1}, y_{t+1}) - v_{t+1}\|^2 \\
&=\E\| \nabla_x f(x_{t+1}, y_{t+1}; \xi_{t+1}) + (1-a_x) v_t -(1-a_x) \nabla_x f(x_{t}, y_t; \xi_{t+1}) - \nabla_x f(x_{t+1}, y_{t+1}) \|^2  \\ 
&\leq \E\|(1-a_x) (v_t - \nabla_x f(x_t, y_t)) 
+(1-a_x) [\nabla_x f(x_t, y_t) - \nabla_x f(x_t, y_t;\xi_{t+1})] \\
&~~~~~~~~~ - [\nabla_x f(x_{t+1}, y_{t+1}) - 
\nabla_x f(x_{t+1}, y_{t+1}; \xi_{t+1})\|^2]  \\ 
&= \E\|(1-a_x) (v_t - \nabla_x f(x_t, y_t)) \|^2 
\\
&~~~ + \E \| (1-a_x)[\nabla_x f(x_t, y_t) - \nabla_x f(x_t, y_t;\xi_{t+1})] 
- [\nabla_x f(x_{t+1}, y_{t+1}) - 
\nabla_x f(x_{t+1}, y_{t+1}; \xi_{t+1})]\|^2 \\
&= \E\|(1-a_x) (v_t - \nabla_x f(x_t, y_t)) \|^2 
\\
&~~~ + \E \| (1-a_x)[\nabla_x f(x_t, y_t) - \nabla_x f(x_t, y_t;\xi_{t+1})] 
- (1-a_x)[\nabla_x f(x_{t+1}, y_{t+1}) - 
\nabla_x f(x_{t+1}, y_{t+1}; \xi_{t+1})] \\
&~~~~~~~~~~ - a_x [\nabla_x f(x_{t+1}, y_{t+1}) - 
\nabla_x f(x_{t+1}, y_{t+1}; \xi_{t+1})] \|^2 \\ 
\end{split} 
\end{equation*}
Then using $\E[\nabla_x f(x_t, y_t)-\nabla_x f(x_t, y_t; \xi_{t+1})] = 0$ and 
$\E[\nabla_x f(x_{t+1}, y_{t+1})-\nabla_x f(x_{t+1}, y_{t+1}; \xi_{t+1})] = 0$, we continue the above inequality as 
\begin{equation*} 
\begin{split}
&\E\|\nabla_x f(x_{t+1}, y_{t+1}) - v_{t+1}\|^2  \leq (1-a_x)\E\|v_t - \nabla_x f(x_t, y_t) \|^2 \\
&~~~ + 2 (1-a_x)^2 \E \|[\nabla_x f(x_t, y_t) - \nabla_x f(x_t, y_t;\xi_{t+1})]  - [\nabla_x f(x_{t+1}, y_{t+1}) - 
\nabla_x f(x_{t+1}, y_{t+1}; \xi_{t+1})] \\
&~~~+ 2a_x^2 \E\| \nabla_x f(x_{t+1}, y_{t+1}) - 
\nabla_x f(x_{t+1}, y_{t+1}; \xi_{t+1})  \|^2 \\
&\leq (1-a_x)\E\|v_t - \nabla_x f(x_t, y_t) \|^2 
+ 4(1-a_x)^2 \E\|\nabla_x f(x_t, y_t) - \nabla_x f(x_{t+1}, y_{t+1})\|^2 \\
&~~~+ 4(1-a_x)^2 \E\|\nabla_x f(x_t, y_t; \xi_{t+1}) - \nabla_x f(x_{t+1}, y_{t+1}; \xi_{t+1})\|^2 + 2a_x^2\sigma^2 \\ 
& \leq (1-a_x)\E\|\nabla_x f(x_t, y_t) - v_t \|^2 + 8(1-a_x)^2\ell^2 (\|x_{t+1} - x_t\|^2 + \|y_{t+1}-y_t\|^2) 
+ 2a_x^2 \sigma^2.
\end{split} 
\label{x_storm}
\end{equation*}  

By similar analysis on $y$-side, we have 
\begin{equation*}
\begin{split} 
&\E\|\nabla_y f(x_{t+1}, y_{t+1}) - u_{t+1}\|^2 \\ 
&\leq (1-a_y)\E\|\nabla_y f(x_t, y_t) - u_t \|^2 +   8(1-a_y)^2\ell^2 (\|x_{t+1} - x_t\|^2 + \|y_{t+1}-y_t\|^2) 
+ 2a_y^2 \sigma^2. 
\end{split}
\label{y_storm} 
\end{equation*}
\end{proof}

The next lemma follows from Lemma 18 of \citep{DBLP:journals/corr/abs-2008-08170}. We include the proof for the sake of completeness. 
\begin{lemma} 
In Algorithm \ref{alg:subroutine} with Option IV, setting $\eta^y \leq \min\{1, \frac{1}{6\ell} \}, \lambda = \frac{1}{6\ell}$, we have 
\begin{equation*}
\begin{split}
\|y_{t+1} - \hat{y}(x_{t+1})\|^2 \leq&  (1-\frac{\mu_y \eta^y \lambda}{4}) \|y_t - \hat{y}(x_t)\|^2 - \frac{3\eta^y \lambda^2}{4} \|u_t\|^2
+ \frac{5\eta^y\lambda}{\mu_y} \|\nabla_y f(x_t, y_t) - u_t\|^2 \\
&+ \frac{5 \ell^2(\eta^x)^2}{\eta^y \lambda \mu_y^3} \|v_t\|^2 .  
\end{split} 
\end{equation*} 
\label{lem:storm_y_star_y}
\end{lemma}

\begin{proof}
Using $\mu_y$-strong concavity of $f(x, y)$ in $y$,
\begin{equation}
\begin{split} 
f(x_t, y) &\leq f(x_t, y_t) + \langle \nabla_y f(x_t, y_t), y-y_t\rangle - \frac{\mu_y}{2} \|y-y_t\|^2 \\ 
&=f(x_t, y_t) + \langle u_t, y-\tilde{y}_{t+1}\rangle 
+ \langle \nabla_y f(x_t, y_t) - u_t, y-\ty_{t+1}\rangle \\
&~~~ + \langle \nabla_y f(x_t, y_t), \ty_{t+1}-y_t\rangle - \frac{\mu_y}{2} \|y - y_t\|^2.  
\end{split}    
\end{equation} 

Using $\ell$-smoothness of $f(x, y)$,
\begin{equation}
\begin{split}
-f(x_t, \ty_{t+1}) \leq -f(x_t, y_t) - \langle \nabla_y f(x_t, y_t), \ty_{t+1} - y_t\rangle + \frac{\ell}{2} \|\ty_{t+1} - y_t\|^2.  
\end{split}    
\end{equation}  

Adding the above two inequalities, we get 
\begin{equation}
\begin{split}
&f(x_t, y) - f(x_t, \ty_{t+1}) \leq \\
&\langle u_t, y-\ty_{t+1}\rangle 
+ \langle \nabla_y f(x_t, y_t) - u_t, y-\ty_{t+1}\rangle - \frac{\mu_y}{2} \|y - y_t\|^2 + \frac{\ell}{2} \|\ty_{t+1} - y_t\|^2. 
\end{split}
\end{equation}

Note that the update of $y$ is
\begin{equation}
\begin{split}
&\Tilde{y}_{t+1} = \mathcal{P}_{\mathcal{Y}}(y_t + \lambda u_t), \\
&y_{t+1} = y_t + \eta^y (\Tilde{y}_{t+1} - y_t), 
\end{split}
\end{equation}
where $\lambda = \frac{1}{6\ell}$.  
Since $\ty_{t+1} = \mathcal{P}_{\mathcal{Y}}(y_t+\lambda u_t) = \arg\min_{y\in \mathcal{Y}} \frac{1}{2} \|y-y_t - \lambda u_t \|^2  $ and $\frac{1}{2}\|y - y_t - \lambda u_t\|^2$ is convex in $y$, we have  
\begin{equation} 
\langle \ty_{t+1}-y_t-\lambda u_t, y-\ty_{t+1}\rangle \geq 0, y\in \mathcal{Y}. 
\end{equation}

Then we get
\begin{equation*}
\begin{split} 
&\langle u_t, y - \ty_{t+1}\rangle \leq \frac{1}{\lambda} 
\langle \ty_{t+1}-y_t, y-\ty_{t+1} \rangle 
= \frac{1}{\lambda} 
\langle \ty_{t+1}-y_t, y_t-\ty_{t+1} \rangle + \frac{1}{\lambda} \langle \ty_{t+1}-y_t, y-y_{t} \rangle \\
& = - \frac{1}{\lambda} \| \ty_{t+1}-y_t \|^2 +  \frac{1}{\lambda} \langle \ty_{t+1}-y_t, y-y_{t} \rangle.  
\end{split} 
\end{equation*} 

Thus, 
\begin{equation}
\begin{split}
f(x_t, y) - f(x_t, \ty_{t+1})  & \leq - \left(\frac{1}{\lambda} - \frac{\ell}{2} \right) \|\ty_{t+1}-y_t\|^2 +  \frac{1}{\lambda} \langle \ty_{t+1}-y_t, y-y_{t} \rangle \\ 
&~~~ + \langle \nabla_y f(x_t, y_t) - u_t,  y-\ty_{t+1}\rangle - \frac{\mu_y}{2} \|y - y_t\|^2. 
\end{split} 
\end{equation}  

Plugging in $y = \hat{y}(x_t)$, 
\begin{equation}
\begin{split}
0\leq f(x_t, \hat{y}(x_t)) - f(x_t, \ty_{t+1}) & \leq - \left(\frac{1}{\lambda} - \frac{\ell}{2}\right) \|\ty_{t+1}-y_t\|^2 + \frac{1}{\lambda} \langle \ty_{t+1}-y_t, \hat{y}(x_t)-y_{t} \rangle \\
&~~~ + \langle \nabla_y f(x_t, y_t) - u_t, \hat{y}(x_t) - \ty_{t+1}\rangle - \frac{\mu_y}{2} \|\hat{y}(x_t) - y_t\|^2.
\end{split} 
\end{equation}  

By $y_{t+1} = y_t + \eta^y (\Tilde{y}_{t+1} - y_t)$, we have 
\begin{equation*}
\begin{split}
&\|y_{t+1} - \hat{y}(x_t)\|^2 = \|y_t + \eta^y (\ty_{t+1}-y_t) - \hat{y}(x_t)\|^2 \\
&=\|y_t - \hat{y}(x_t)\|^2 + 2 \eta^y\langle \ty_{t+1}-y_t, y_{t} -  \hat{y}(x_t) \rangle + (\eta^y)^2 \|\ty_{t+1} - y_t\|^2 \\ 
& \leq \|y_t  - \hat{y}(x_t)\|^2 + (\eta^y)^2 \|\ty_{t+1} - y_t\|^2 
-\eta^y (2 - \ell\lambda) \|\ty_{t+1} - y_t\|^2 \\
&~~~ + 2 \eta^y \lambda \langle \nabla_y f(x_t, y_t) - u_t,  \hat{y}(x_t) - \ty_{t+1}\rangle - \mu_y \eta^y \lambda \|\hat{y}(x_t) - y_t\|^2 \\ 
&\leq \|y_t  - \hat{y}(x_t)\|^2 - (2\eta^y -(\eta^y)^2 - \ell \lambda \eta^y) \|\ty_{t+1}-y_t\|^2 \\ 
&~~~ +2\eta^y \lambda \left[ \frac{2}{\mu_y} \|\nabla_y f(x_t, y_t) - u_t\|^2 + \frac{\mu_y}{8}\|\hat{y}(x_t) - y_{t+1}\|^2 \right] - {\mu_y \eta^y \lambda} \|\hat{y}(x_t) - y_t\|^2 \\  
& \leq (1-{\mu_y \eta^y \lambda}) \|y_t  - \hat{y}(x_t)\|^2 - (2\eta^y-(\eta^y)^2-\ell\lambda\eta^y) \|\ty_{t+1} - y_t\|^2 + \frac{\eta^y \mu_y \lambda}{2}\|\hat{y}(x_t) - y_t\|^2 \\ 
&~~~ + \frac{\eta^y \mu_y \lambda}{2}\|y_t - \ty_{t+1}\|^2 +  \frac{4\eta^y \lambda}{\mu_y} \|\nabla_y f(x_t, y_t) - u_t\|^2 \\ 
&\leq (1-\frac{\mu_y \eta^y \lambda}{2})\|y_t-\hat{y}(x_t)\|^2 -  (2\eta^y - (\eta^y)^2 - \ell\lambda \eta^y - \frac{\eta^y \mu_y \lambda}{2}) \|\ty_{t+1} - y_t\|^2 + \frac{4\eta^y \lambda}{\mu_y} \|\nabla_y f(x_t, y_t) - u_t\|^2 \\ 
&\leq (1-\frac{\mu_y \eta^y \lambda}{2})\|y_t-\hat{y}(x_t)\|^2 - \frac{3}{4} \eta^y \|\ty_{t+1} - y_t\|^2 + \frac{4\eta^y \lambda}{\mu_y} \|\nabla_y f(x_t, y_t) - u_t\|^2 \\
&=(1-\frac{\mu_y \eta^y \lambda}{2})\|y_t-\hat{y}(x_t)\|^2 - \frac{3 \eta^y \lambda^2}{4} \|u_t\|^2 + \frac{4\eta^y  \lambda}{\mu_y} \|\nabla_y f(x_t, y_t) - u_t\|^2,
\end{split}    
\end{equation*}  
where the last inequality holds because $\mu_y\leq \ell, \eta^y \leq \min\{1, \frac{1}{6\ell}\}$.
Using  the above inequalities, we get
\begin{equation*}
\begin{split}
&\|y_{t+1} - \hat{y}(x_{t+1})\|^2 = \|y_{t+1} - \hat{y}(x_t) + \hat{y}(x_t) - \hat{y}(x_{t+1})\|^2 \\
&\leq (1+\frac{\eta^y \mu_y \lambda}{4})\|y_{t+1} - \hat{y}(x_t)\|^2 + (1+\frac{4}{\eta^y \mu_y \lambda})\|\hat{y}(x_t) - \hat{y}(x_{t+1})\|^2 \\ 
&\leq (1+\frac{\eta^y \mu_y \lambda}{4})\|y_{t+1} - \hat{y}(x_t)\|^2 
+(1+\frac{4}{\eta^y \mu_y \lambda}) \frac{\ell^2}{\mu_y^2} \|x_{t+1} - x_t\|^2 \\  
&\leq (1-\frac{\mu_y \eta^y \lambda}{2}) (1+\frac{\eta^y \mu_y \lambda}{4}) \|y_t - \hat{y}(x_t)\|^2 - \frac{3\eta^y\lambda^2}{4} \|u_t\|^2 \\ 
&~~~+ (1+\frac{\eta^y \mu_y \lambda}{4}) \frac{4\eta^y \lambda }{\mu_y} \|\nabla_y f(x_t, y_t) - u_t\|^2 +
(1+\frac{4}{\eta^y \mu_y \lambda}) \frac{\ell^2}{\mu_y^2} (\eta^x)^2 \|v_t\|^2 \\
&\leq (1-\frac{\mu_y \eta^y \lambda}{4}) \|y_t - \hat{y}(x_t)\|^2 - \frac{3\eta^y \lambda^2}{4} \|u_t\|^2
+ \frac{5\eta^y\lambda}{\mu_y} \|\nabla_y f(x_t, y_t) - u_t\|^2 
+ \frac{5 \ell^2 (\eta^x)^2}{\eta^y \lambda \mu_y^3} \|v_t\|^2,
\end{split}    
\label{y_recursion} 
\end{equation*}  
where the second inequality is because $\hat{y}(\cdot)$ is $\frac{\ell}{\mu_y}$-Lipshitz \citep{lin2019gradient}. 
\end{proof} 

The following lemma analyze the convergence of one stage in PES-Storm.

\begin{lemma}
By setting $\eta^x = \frac{\mu_2^2}{1000 \ell^2} \eta^y, a_x = \frac{800\ell}{\mu_y} (\eta^y)^2, a_y = \frac{800\ell}{\mu_y}  (\eta^y)^2 $, $\eta^y \leq O(\frac{1}{30}\sqrt{\frac{\mu_y}{\ell}})$ to ensure $0< a_x, a_y < 1$, one stage of Algorithm \ref{alg:subroutine} with Option IV: Storm returns an solution $(x_\tau, y_\tau)$ such that 
\begin{equation*}
\begin{split}
&\E\|y_\tau - \hat{y}(x_\tau)\|^2 + \frac{\eta^y}{\eta^x} \E[\|\nabla_x f(x_\tau, y_\tau) - v_\tau\|^2]  
+  \frac{\eta^y}{\eta^x} \E[\|\nabla_y f(x_\tau, y_\tau) - v_\tau\|^2] + \frac{1}{8} \E\|v_\tau\|^2 ]\\
&\leq \frac{\Gamma_1 - \Gamma_{T+1}}{\eta^x T} + \frac{4C \ell (\eta^y)^3 \sigma^2 }{\mu_y \eta^x}, 
\end{split}
\end{equation*}
where $C=1600$ and $\tau$ is sampled from $1, ..., T$. 
\end{lemma}

\begin{proof}
Defining a Lyapunov function as in \citep{DBLP:journals/corr/abs-2008-08170}, 
\begin{equation*} 
\Gamma_t = P(x_t) + \frac{\mu_y}{\ell} 
\left( 9\ell^2 \|y_t - \hat{y}(x_t)\|^2 + \frac{1}{\eta^y} \|\nabla_x f(x_t, y_t) - v_t\|^2 + \frac{1}{\eta^y} \|\nabla_y f(x_t, y_t) - u_t\|^2 \right). 
\end{equation*}

Then we have
\begin{equation*}
\begin{split}
&\Gamma_{t+1} - \Gamma_t \\
&= P(x_{t+1})-P(x_t) 
+ \frac{9\mu_y}{\ell} \ell^2 (\|y_{t+1}-\hat{y}(x_{t+1})\|^2 - \|y_t - \hat{y}(x_t)\|^2) \\
&~~~+ \frac{\mu_y}{\ell} \left(\frac{1}{\eta^y} \|\nabla_x f(x_{t+1}, y_{t+1}) - v_{t+1}\|^2-\frac{1}{\eta^y}  \|\nabla_x f(x_t, y_t) - v_t\|^2 \right) \\
&~~~ + \frac{\mu_y}{\ell} \left(\frac{1}{\eta^y} \|\nabla_y f(x_{t+1}, y_{t+1}) - u_{t+1}\|^2-\frac{1}{\eta}  \|\nabla_y f(x_t, y_t) - u_t\|^2 \right) \\ 
&\leq -\frac{\eta^x}{4} \|v_t\|^2 +  \eta^x \ell^2 \|\hat{y}(x_t)-y_t\|^2 
+ \eta^x \|\nabla_x f(x_t, y_t)-v_t\|^2 \\  
&~~~ + \frac{9\mu_y}{\ell} \ell^2 \left(-\frac{\mu_y\eta^y\lambda}{4}\|y_t - \hat{y}(x_t)\|^2 - \frac{3\eta^y \lambda^2 }{4} \|u_t\|^2 + \frac{5\eta^y \lambda}{\mu_y} \|\nabla_y f(x_t, y_t) - u_t\|^2 +  \frac{5\ell^2(\eta^x)^2}{\eta^y \lambda \mu_y^3} \|v_t\|^2 \right) \\ 
&~~~ - \frac{\mu_y a_x}{ 16 \ell \eta^y} \E[\|\nabla_x f(x_t, y_t)-v_t\|^2] 
+ \frac{\mu_y \ell^2}{2 \ell \eta^y} ((\eta^x)^2 \E[\|v_t\|^2] + (\eta^y)^2\lambda^2 \E[\|u_t\|^2]) + \frac{\mu_y a_x^2 \sigma^2}{8 \ell \eta^y} \\  
&~~~ -\frac{\mu_y a_y}{16 \ell \eta^y} \E[\|\nabla_y f(x_t, y_t)-u_t\|^2] 
+ \frac{\mu_y \ell^2}{2 \ell \eta^y} ((\eta^x)^2 \E[\|v_t\|^2] + (\eta^y)^2 \lambda^2 \E[\|u_t\|^2]) + \frac{\mu_y a_y^2 \sigma^2}{8 \ell \eta^y} \\
&\leq (-\frac{9\mu_y^2 \ell \lambda \eta^y}{4} + \eta^x\ell^2) \|y_t -  \hat{y}(x_t)\|^2  \\ 
& +  (\eta^x - \frac{\mu_y a_x}{16 \ell\eta^y} )  \E[\|\nabla_x f(x_t, y_t) - v_t\|^2] 
+ (45 \eta^y \ell \lambda - \frac{\mu_y a_y }{16 \ell \eta^y})  \E[\|\nabla_y f(x_t, y_t)-u_t\|^2] \\  
& - (\frac{\eta^x}{4} - \frac{\mu_y \ell (\eta^x)^2}{\eta^y} 
- \frac{45 \ell^3 (\eta^x)^2}{\eta^y \lambda \mu_y^2} ) \E \|v_t\|^2 
+ ( \mu_y \ell \eta^y  - 9\mu_y \ell \frac{3\eta^y}{4})\lambda^2 \E[\|u_t\|^2] + \frac{2\mu_y a_x^2 \sigma^2}{\ell\eta^y} + \frac{2\mu_y a_y^2 \sigma^2}{\ell \eta^y},    \\ 
\end{split}    
\end{equation*} 
where the first inequality uses Lemma \ref{lem:storm_P}, Lemma \ref{lem:storm_var_recur}, Lemma \ref{lem:storm_y_star_y}. 

Taking $\eta^x = \frac{\mu_2^2}{1000 \ell^2} \eta^y, a_x = \frac{800\ell}{\mu_y} (\eta^y)^2, a_y = \frac{800\ell}{\mu_y}  (\eta^y)^2 $, $\eta^y \leq O(\sqrt{\frac{\mu_y}{\ell}})$ to ensure $0< a_x, a_y < 1$,  we get  
\begin{equation}
\begin{split}
\Gamma_{t+1} - \Gamma_t \leq&  - \eta^x  
\|y_t - \hat{y}(x_t)\|^2 - \eta^y \E[\|\nabla_x f(x_t, y_t) - v_t\|^2]  \\
& - \eta^y \E[\|\nabla_y f(x_t, y_t) - v_t\|^2] 
- \frac{1}{8} \eta^x \E\|v_t\|^2 + \frac{4 C \ell (\eta^y)^3 \sigma^2}{\mu_y},
\end{split} 
\end{equation}
where $C=1600$. 

Thus, 
\begin{equation}
\begin{split}
&\eta^x \|y_t - \hat{y}(x_t)\|^2 +\eta^y \E[\|\nabla_x f(x_t, y_t) - v_t\|^2]  
+  \eta^y \E[\|\nabla_y f(x_t, y_t) - v_t\|^2] + \frac{1}{8} \eta^x \E\|v_t\|^2 \\
&\leq \Gamma_t - \Gamma_{t+1} + \frac{4 C \ell (\eta^y)^3 \sigma^2 }{\mu_y}. 
\end{split}
\end{equation}

Taking average over $t=1,.., T$,
\begin{equation}
\begin{split}
&\frac{1}{T} \sum\limits_{t=1}^{T} [
\|y_t - \hat{y}(x_t)\|^2 + \frac{\eta^y}{\eta^x} \E[\|\nabla_x f(x_t, y_t) - v_t\|^2]  
+  \frac{\eta^y}{\eta^x} \E[\|\nabla_y f(x_t, y_t) - v_t\|^2] + \frac{1}{8} \E\|v_t\|^2 ] \\
&\leq \frac{\Gamma_1 - \Gamma_{T+1}}{\eta^x T} + \frac{4C \ell (\eta^y)^3 \sigma^2 }{\mu_y \eta^x} .
\end{split}
\end{equation}

Randomly sample $\tau$ from $1, ..., T$, we obtain
\begin{equation*}
\begin{split}
&\E\|y_\tau - \hat{y}(x_\tau)\|^2 + \frac{\eta^y}{\eta^x} \E[\|\nabla_x f(x_\tau, y_\tau) - v_\tau\|^2]  
+  \frac{\eta^y}{\eta^x} \E[\|\nabla_y f(x_\tau, y_\tau) - v_\tau\|^2] + \frac{1}{8} \E\|v_\tau\|^2 ] \\
&\leq \frac{\Gamma_1 - \Gamma_{T+1}}{\eta^x T} + \frac{4C \ell (\eta^y)^3 \sigma^2 }{\mu_y \eta^x}. 
\end{split}
\end{equation*} 
\end{proof}

Theorem \ref{thm:storm_primal_informal} is formally restated as follows: 
\begin{theorem}[Formal version of Theorem \ref{thm:storm_primal_informal}] 
Assume Assumption \ref{ass1}, \ref{ass2}, \ref{ass3}, \ref{ass6} hold. Define a constant $\epsilon_1 = \frac{ C \ell^2 \sigma^2 }{2\mu \mu^2_y}$ and $\epsilon_k = \epsilon_1/2^k$, where $C=1600$. 
By setting $\eta^y_{k} = \min\{\frac{1}{30} \sqrt{\frac{\mu_y}{\ell}},   \sqrt{\frac{\mu\mu_y^3\epsilon_k}{320C \ell^3 \sigma^2 }}\}$, $\eta^x_{k} = \frac{\mu_y^2}{1000\ell^2} \eta^y_{k}$,
$T_k = O\left(\max\{ \frac{1}{\mu \eta^x_{k}}, \frac{\mu_y^3}{\ell^3 \eta^x_{k} \eta^y_{k}}\} \right)$,
after $K=O(log(\epsilon_1/\epsilon))$ stages, $\E[P(\bx_k) - P(x_*)] \leq \epsilon$. The total stochastic first-order oracle call complexity is 
$\widetilde{O} \left( \frac{\ell^{7/2}}{\mu^{3/2} \mu_y^{7/2} \epsilon^{1/2}} +  \frac{\ell^2}{\mu \mu_y^2 \epsilon}\right)$.
\label{thm:storm_primal}
\end{theorem} 

\begin{proof}
Without loss of generality, let us assume that the initialization of the first stage $P(x_0^1) - P(x_*) = \epsilon_0 \geq \frac{C \ell^2 \sigma^2 }{2\mu \mu^2_y}$, i.e. 
$\sqrt{\frac{\mu \mu_y^3 \epsilon_0}{320 C\ell^3 \sigma^2}} > \frac{1}{30}\sqrt{\frac{\mu_y}{\ell} }$. 
The case where $\epsilon_0 \leq \frac{320 C \ell^2 \sigma^2 }{\mu \mu^2_y}$ can be simply covered by our proof.
Then denote $\epsilon_1 = \frac{ C \ell^2 \sigma^2 }{2 \mu \mu^2_y}$ and $\epsilon_k = \epsilon_1/2^k$. 

Let's consider the first stage, we have initialization such that $P(x_0) - P(x_*) = \epsilon_0$ and $\E[\|\nabla_x f(x_0, y_0) - v_0\|^2 + \|\nabla_y f(x_0, y_0) - u_0\|^2]\leq \sigma^2$.

We bound the error of the first stage's output as follows
\begin{equation*}  
\begin{split}  
&\E\ell^2 \|\by_1 - \hat{y}(\bx_1)\|^2 + \frac{\eta^y_1}{\eta^x_1} \E[\|\nabla_x f(\bx_1, \by_1) - \bar{v}_1\|^2]  
+  \frac{\eta^y_1}{\eta^x_1} \E[\|\nabla_y f(\bx_1, \by_1) - \bar{v}_1\|^2] + \frac{1}{8} \E\|\bar{v}_1\|^2 ] \\
&\leq \frac{\E[P(x_{0})] - P(x_*)}{\eta^x{1} T_1} + 
\frac{\mu_y \ell^2 \|y_{0}-\hat{y}(x_{0})\|^2 }{\ell \eta^x_{1} T_1}
+ \frac{\mu_y}{\ell \eta^y_{1} \eta^x_{1} T_1} \|\nabla_x f(\bx_1, \by_1) - \bar{v}_1\|^2 \\
&~~~~ 
+ \frac{\mu_y}{\ell \eta^y_{1} \eta^x_{1} T_1} \|\nabla_y f(\bx_1, \by_1) - \bar{u}_1\|^2 
+ \frac{4C \ell (\eta^y_{1})^3 \sigma^2 }{\mu_y \eta^x_{1}} \\ 
&\leq \frac{\mu \epsilon_1}{16}, 
\end{split}
\end{equation*}
where the last inequality is by the setting $\eta^x_{1} = \frac{\mu_y^2}{1000 \ell^2} \eta^y_{1}$, $\eta^y_{1} = \sqrt{\frac{\mu_y}{\ell}}$ and \\
$T_1 = O\left( \max\{\frac{\epsilon_0}{\eta^x_{1} \mu \epsilon_1}, \frac{\mu_y \sigma^2}{\mu \eta^x_{1} \eta^y_{1} \epsilon_1}, \frac{\mu_y\ell D}{\eta^x_{1} \mu \epsilon_1} \} \right)$, which is in the order of a constant and where $D$ denotes the diameter of $\mathcal{Y}$.  This result implies that
\begin{equation} 
\begin{split} 
& \E\ell^2\|\by_1 - \hat{y}(\bx_1)\|^2 \leq \frac{\mu \epsilon_1}{16}, \\
& \E[\|\nabla_x f(x_\tau, y_\tau) - v_\tau\|^2]  
+   \E[\|\nabla_y f(x_\tau, y_\tau) - v_\tau\|^2] \leq \frac{\mu \mu_y^2 \epsilon_1}{16 \ell^2},   \\
& \E\|\bar{v}_1\|^2 \leq \frac{\mu\epsilon_1}{2}. 
\end{split}
\end{equation}

Using the $\mu$-PL condition of $P(x)$,
\begin{equation*} 
\begin{split} 
P(\bx_1) - P(x_*) &\leq \frac{1}{2\mu} \|\nabla P(\bx_1)\|^2
= \frac{1}{2\mu} \|\nabla P(\bx_1) - \nabla_x f(\bx_1, \by_1) + \nabla_x f(\bx_1, \by_1) - \bar{v}_1 + \bar{v}_1\|^2 \\
&\leq \frac{1}{2\mu} (3\ell^2 \|\by_1 - \hat{y}(\bx_1)\|^2 + 3\|\nabla_x f(x_t, y_t) - v_t\|^2 + 3\|v_t\|^2) \leq \epsilon_1,
\end{split}
\end{equation*}
where the second inequality has used $\nabla P(x) = \nabla f(x, \hat{y}(x))$, which is by the Lemma \ref{lem:primal_gradient_danskin}.

Starting from the second stage, we will prove by induction.
Suppose the initialization of $k$-th stage ($k\geq 2$) satisfies 
$\E[P(\bx_{k-1}) - P(x_*)] \leq \epsilon_{k-1}$, 
$\E[\|\nabla_x f(\bx_{k-1}, \by_{k-1}) - v_{k-1}\|^2 + \|\nabla_y f(\bx_{k-1}, \by_{k-1}) - u_{k-1}\|^2] \leq \frac{\mu \mu_y^2 \epsilon_{k-1}}{\ell^2}$. 
The error of the output of $k$-th stage can be bounded as 
\begin{equation}
\begin{split}
&\E\ell^2 \|\by_k - \hat{y}(\bx_k)\|^2 + \frac{\eta^y_{k}}{\eta^x_{k}} \E[\|\nabla_x f(\bx_k, \by_k) - v_k\|^2]  
+  \frac{\eta^y_{k}}{\eta^x_{k}} \E[\|\nabla_y f(\bx_k, \by_k) - v_k\|^2] + \frac{1}{8} \E\|v_k\|^2 ] \\
&\leq \frac{\E[P(\bx_{k-1})] - P(x_*)}{\eta^x_{k} T_k} + 
\frac{\mu_y \ell^2 \|\by_{k-1}-\hat{y}(\bx_{k-1})\|^2 }{\ell \eta^x_{k} T_k}
+ \frac{\mu_y}{\ell \eta^y_{k} \eta^x_{k} T_k} \|\nabla_x f(\bx_k, \by_k) - v_k\|^2 \\
&~~~~ 
+ \frac{\mu_y}{\ell \eta^y_{k} \eta^x_{k} T_k} \|\nabla_y f(\bx_k, \by_k) - u_k\|^2 
+ \frac{4C \ell (\eta^y_{k})^3 \sigma^2 }{\mu_y \eta^x_{k}} \\
& 
\leq \frac{\epsilon_{k-1}}{\eta^x_{k} T_k} + \frac{\mu_y \ell^2 \mu \epsilon_{k-1} }{\ell \eta^x_{k} T_k}
+ \frac{\mu \mu_y^3 \epsilon_{k-1}}{\ell^3 \eta^y_{k} \eta^x_{k} T_k}
+ \frac{\mu \mu_y^3 \epsilon_{k-1}}{\ell^3 \eta^y_{k} \eta^x_{k} T_k } 
+ \frac{4C \ell (\eta^y_{k})^3 \sigma^2 }{\mu_y \eta^x_{k}} \\ 
& \leq \frac{\mu \epsilon_k}{16},  
\end{split}
\end{equation}
where the last inequality is due to the setting $\eta^y_{k} = \sqrt{ \frac{\mu \mu_y^3 \epsilon_k}{320 C \ell^3 \sigma^2} }$, $\eta^x_{k} = \frac{\mu_y^2}{1000 \ell^2} \eta^y_{k}$, $T_k=O\left(\max\left\{\frac{1}{\mu  \eta^x_{k}}, \frac{\mu_y^3}{\ell^3 \eta^x_{k} \eta^y_{k}} \right\}  \right)$. 
 
Similar to in the analysis of first stage, this result implies that
\begin{equation} 
\begin{split} 
& \E[\ell^2\|\by_k - \hat{y}(\bx_k)\|^2] \leq \frac{\mu \epsilon_{k-1}}{16}, \\ 
& \E[\|\nabla_x f(\bx_k, \by_k) - v_k\|^2]  
+   \E[\|\nabla_y f(\bx_k, \by_k) - v_k\|^2] \leq \frac{\mu \mu_y^2 \epsilon_k}{16 \ell^2},   \\
& \E\|v_k\|^2 \leq \frac{\mu\epsilon_k}{2}.  
\end{split}
\end{equation}
Using the $\mu$-PL condition of $P(x)$, we obtain 
\begin{equation}
\begin{split}
P(\bx_k) - P(x_*) &\leq \frac{1}{2\mu} \|\nabla P(\bx_k)\|^2
= \frac{1}{2\mu} \|\nabla P(\bx_k) - \nabla_x f(\bx_k, \by_k) + \nabla_x f(\bx_k, \by_k) - v_k + v_k\|^2 \\
&\leq \frac{1}{2\mu} (3\ell^2 \|\by_k - \hat{y}(\bx_k)\|^2 + 3\|\nabla_x f(\bx_k, \by_k) - v_k\|^2 + 3\|v_k\|^2) \leq \epsilon_k . 
\end{split}
\end{equation}

By induction we know that after $K = 1+\log(\epsilon_1/\epsilon)$ stages, $P(\bx_K) - P(x_*) \leq 0$. 
Total complexity is  
\begin{equation}
\begin{split} 
\sum\limits_{k=1}^{K} T_k 
&= O\left(  \sum\limits_{k=2}^K \left( \frac{1}{\mu \eta^x_{k}} + \frac{ \mu_y^3}{\ell^3 \eta^x_{k} \eta^y_{k}} \right)  \right) \\ 
&= O\left(  \sum\limits_{k=2}^K \left( \frac{\ell^2}{\mu \mu_y^2 \sqrt{\mu \mu_y^3 \epsilon_k / \ell^3} } + \frac{ \mu_y^3 \ell^2}{ \mu_y^2 \mu \mu_y^3 \epsilon_k } \right)  \right) \\ 
&= \widetilde{O} \left( \frac{\ell^{7/2}}{\mu^{3/2} \mu_y^{7/2} \epsilon^{1/2}} +  \frac{\ell^2}{\mu \mu_y^2 \epsilon}\right).   
\end{split}
\end{equation}

\end{proof}

In the following corollary, we analyze the convergence of duality gap by PES-Storm.

\begin{corollary}
Under the same setting as in Theorem \ref{thm:storm_primal} and suppose Assumption \ref{ass6} as well. To achieve $\E[\emph{\text{Gap}}(\bx_K, \by_K)] \leq \epsilon$, the total number of stochastic first-order oracle calls is $\widetilde{O}\left(\frac{\ell^{9/2}}{\mu^{3/2}\mu_x^{1/2}\mu_y^{9/2} \epsilon^{1/2}} + \frac{\ell^4}{\mu \mu_x \mu_y^3 \epsilon}\right)$.
\label{cor:storm_duality}
\end{corollary}

\begin{proof}
Assume after $K$ stages, we have the output $\bx_K,\by_K$ such that 
\begin{equation}
\begin{split}
P(\bx_K) - P(x_*) &\leq \frac{1}{2\mu} \| \nabla P(\bx_K)\|^2  \\
&\leq \frac{1}{2\mu}\left[ 3\ell^2 \|\by_K - \hat{y}(\bx_K)\|^2 + 3\|\nabla_x f(\bx_K, \by_K) - v_K\|^2 + 3\|v_K\|^2 \right]\\ 
&\leq  \hat{\epsilon},
\end{split} 
\end{equation} 
and 
\begin{equation}
\E\|\by_K - \hat{y}(\bx_K)\|^2 \leq \frac{\mu\hat{\epsilon}}{16\ell^2}. 
\end{equation}

Hence, by the strong concavity,
\begin{equation} 
\begin{split}
\|\hat{y}(\bx_K) - y_*\|^2 &\leq \frac{f(\bx_K, \hat{y}(\bx_K)) - f(\bx_K, \by_K)}{2\mu_y} \\
&\leq \frac{f(\bx_K, \hat{y}(\bx_K)) - f(x_*, y_*) + f(x_*, y_*) - f(\bx_K, \by_K)}{2\mu_y} \\
&\leq  \frac{f(\bx_K, \hat{y}(\bx_K)) - f(x_*, y_*)}{2\mu_y} \\
&\leq \frac{P(\bx_K) - P(x_*)}{2\mu_y} \\
&\leq \frac{\hat{\epsilon}}{2\mu_y}.   
\end{split}
\end{equation}

Thus, 
\begin{equation}
\begin{split}
\|\by_K - y_*\|^2 \leq 2\|\by_K - \hat{y}(\bx_K)\|^2 + 2\|\hat{y}(\bx_K) - y_*\|^2 
\leq \frac{\hat{\epsilon}}{\mu_y}.  
\end{split}
\end{equation}

And the dual function $D(y) = \min_{x'} f(x', y)$ is $\ell + \frac{\ell^2}{\mu_x}\leq \frac{2\ell^2}{\mu_x}$, where $\mu_x$ is the $x$-side PL condition coefficient. 
Therefore, we have
\begin{equation}
\begin{split} 
f(x_*, y_*) - f(\hat{x}(\by_K), \by_K) = D(y_*) - D(\by_K) \leq \frac{2\ell^2}{\mu_x} \|\by_K - y_*\|^2 \leq \frac{\ell^2 \hat{\epsilon}}{\mu_x \mu_y}. 
\end{split}
\end{equation}

Then we know the duality gap is 
\begin{align}
\begin{split}
f(\bx_K, \hat{y}(\bx_K)) - f(\hat{x}(\by_K), \by_K) 
&= f(\bx_K, \hat{y}(\bx_K)) - f(x_*, y_*) + f(x_*, y_*) - f(\hat{x} (\by_K), \by_K)  \\
&\leq \hat{\epsilon} + \frac{\ell^2\hat{\epsilon}}{\mu_x \mu_y}. 
\end{split} 
\end{align}
To make the duality gap less than $\epsilon$, we need $\hat{\epsilon} \leq O(\frac{\mu_x\mu_y \epsilon}{\ell^2})$. Therefore, it takes\\ $\widetilde{O}\left(\frac{\ell^{9/2}}{\mu^{3/2}\mu_x^{1/2}\mu_y^{9/2} \epsilon^{1/2}} + \frac{\ell^4}{\mu \mu_x \mu_y^3 \epsilon}\right)$ to have a  $\epsilon$-duality gap.
\end{proof}

\section{Justification of PL condition}  
In this section, we show the analysis of cases where the $x$-side PL condition can hold, and show the properties that follow from the $x$-side PL condition.
We need to introduce a auxiliary lemma as follows.
\begin{lemma}[Corallary 5.1 of \citep{li2018calculus}]
\label{lem:guoyin_PL}
Suppose $h(x) = g(Ax)$, where $A$ is a matrix and $g(\cdot)$ is strongly convex, then $h(x)$ satisfies a $\mu$-PL condition. 
\end{lemma}

\subsection{Proof of Lemma \ref{lem:pl_auc}} 
Here we prove the Lemma \ref{lem:pl_auc} which justifies the PL condition for the non-convex AUC maximization problem.
\begin{proof}[Proof of Lemma \ref{lem:pl_auc}] 
Before diving into analyzing the min-max formulation of the AUC maximization problem, we investigate the property of the optimal solution to a AUC maximization problem. Reconstruct a data set  $\{(\a_1, c_1), ..., (\a_i, c_i), ..., (\a_n, c_n)\}$ where $c_i=i$ if $b_i=1$ and $c_i=0$ if $b_i=-1$. Consider the problem
\begin{equation}
\begin{split}
\min_{\w} F_1(\w) := \sum\limits_{i=1}^{n} (h(\w; \a_i) - b_i)^2. 
\end{split}
\end{equation}
By Theorem 1 and Theorem 3 of \cite{allen2018convergence}, we know that  for $\w_* = \arg\min F_1(\w)$, $\|\w_*-\w_0\|_2 \leq \omega$ and 
$F_1(\w_*) = 0$ where $\omega = O(\frac{\log m}{\sqrt{m}})$ and $\w_0$ is a random initialization. Then we know that $\w_*$ is a optimal solution to problem (\ref{prob:auc_square}) as well with the optimal objective to be 0. Therefore, $\w_*$ is also a optimal solution of the Problem (\ref{auc_min-max-1}). 

Then let us consider the min-max formulation of the AUC maximization problem. 
For the $n$ input data points, the problem (\ref{auc_min-max-1}) can be written as 
\begin{equation} 
\label{auc_min-max-2} 
\min\limits_{(\w, s, r)} \max\limits_{\alpha\in \mathbb{R}} f(\w, s, r, y)=\frac{1}{n} \sum\limits_{i=1}^{n} F(\mathbf{w}, s, r, y, \z_i). 
\end{equation}

From Section 12 and Section 13 of \citep{allen2018convergence}, we know that $h(\w; \a) \leq O(\log m)$. Then by a similar analysis of Lemma 7 and Lemma 8 of \citep{guo2020communication}, it holds that $\max\{|s|, |r|, |y|\} \leq O(\log m)$.   
    
By Theorem 5 of \citep{allen2018convergence}, for $\|\w - \w_0\|\leq \omega$, with probability at least  $1-e^{-\widetilde{\Omega}(m\omega^{2/3}\Tilde{L})}$, 
\begin{equation} 
h(\w;\a) = h(\w_0;\a) + \langle\nabla h(\w_0;\a), \w-\w_0\rangle \pm \widetilde{O}(\Tilde{L}^3 \omega^{4/3} \sqrt{m}). 
\end{equation}   

Then for any fixed $y$ and for $\|\w-\w_0\|\leq \omega$, with probability at least  $1-e^{-\widetilde{\Omega}(m\omega^{2/3}\Tilde{L})}$, 
\begin{equation}   
\begin{split}      
&f(x, y) = f(\w, s, r, y) \\
&=\frac{1}{n}\sum\limits_{i=1}^{n} 
\bigg[ (1-p)(h(\w;\a_i)-s)^2 \I_{[b_i=1]} +p(h(\w;\a_i)-r)^2\I_{[b_i=-1]} \\ 
&~~~~~~~~~~~~~  
+2(1+y)(ph(\w;\a_i)\I_{[b_i=-1]}-(1-p)h(\w;\a_i)\I_{[b_i=1]})  -p(1-p)y^2 \bigg] \\
&=\frac{1}{n}\sum\limits_{i=1}^{n} 
\bigg[ (1-p)(h(\w_0;\a_i) + \langle\nabla h(\w_0; \a_i), \w-\w_0 \rangle + \widetilde{O}(\tilde{L}^3 \omega^{4/3} \sqrt{m} )-s)^2 \I_{[b_i=1]} \\
&~~~~~~~~~~~~
+p(h(\w_0; \a_i) + \langle \nabla h(\w_0; \a_i), \w-\w_0\rangle +\widetilde{O}(\Tilde{L}^3\omega^{4/3}\sqrt{m}) -r)^2 \I_{[b_i=-1]} \\
&~~~~~~~~~~~~
+2(1+y)p (h(\w_0;\a_i) + \langle\nabla h(\w_0; \a_i), \w-\w_0 \rangle + \widetilde{O}(\tilde{L}^3 \omega^{4/3} \sqrt{m} )) \I_{[b_i=-1]} \\   
&~~~~~~~~~~~~
-2(1+y)(1-p)  (h(\w_0;\a_i) + \langle\nabla h(\w_0; \a_i), \w-\w_0 \rangle + \widetilde{O}(\tilde{L}^3 \omega^{4/3} \sqrt{m} )) \I_{[b_i=1]} \\  
&~~~~~~~~~~~~ - p(1-p)y^2 
\bigg]. 
\end{split} 
\label{eq:ntk_approx}
\end{equation} 
Then, 
\begin{equation}
\begin{split}
\hat{y}(x) = &\frac{1}{(1-p)n} \sum\limits_{i=1}^{n} (h(\w_0; \a_i) + \langle \nabla h(\w_0; \a_i), \w - \w_0\rangle) \I_{[b_i=-1]} \\
&-  \frac{1}{p n} \sum\limits_{i=1}^{n} (h(\w_0; \a_i) + \langle \nabla h(\w_0; \a_i), \w - \w_0\rangle)\I_{[b_i=1]} + \widetilde{O}(\tilde{L}^3 \omega^{4/3} \sqrt{m} ).
\end{split}
\end{equation}

Thus,
\begin{equation}
\begin{split}
&P(x) = \max_{y} f(x, y) = \frac{1}{n}\sum\limits_{i=1}^{n}\bigg[
(1-p)(\langle \nabla h(\w_0;\a_i),\w-\w_0\rangle - s)^2 \I_{[b_i=1]} \\
&~~~~~~~~~~~~ +p(\langle \nabla h(\w_0;\a_i),\w-\w_0\rangle-r)^2\I_{[b_i=-1]} \\
& + \frac{1}{p(1-p)} \big(p\langle  \frac{1}{n}\sum\limits_{j=1}^{n} \nabla h(\w_0;\a_j), \w-\w_0\rangle\I_{[b_j=-1]}  - (1-p)\langle \nabla \frac{1}{n}\sum\limits_{j=1}^{n}  h(\w_0;\a_j), \w-\w_0 \rangle \I_{[b_j=1]}\big)^2  
\bigg]\\
& + \widetilde{O}(\Tilde{L}^3 \omega^{4/3} \sqrt{m}) \\
& = \|Hx - c\|^2 + \widetilde{O}(\Tilde{L}^3 \omega^{4/3} \sqrt{m}) 
\end{split} 
\end{equation} 
where $H \in \mathbb{R}^{(n+1)\times 3}$ and $c\in \mathbb{R}^{n+1}$. For $0\leq i \leq n$, the $i$-th row is $H_i = (\sqrt{1-p}\nabla h(\w_0; \a_i), -1, 0)$ if $b_i = 1$ and $H_i=(\sqrt{p}\nabla h(\w_0; \a_i), 0, -1)$ if $b_i=-1$; and the $i$-th element of $c$ is $c_i=\sqrt{1-p} \langle \nabla h(\w_0; \a_i), \w_0\rangle $.
The last row of $H$ is $(\frac{1}{\sqrt{p(1-p)}}(p\frac{1}{n} \sum\limits_{i=1}^{n}\nabla h(\w_0; \a_i)\I_{[b_i=-1]} - (1-p) \frac{1}{n} \sum\limits_{i=1}^{n}\nabla h(\w_0; \a_i)\I_{[b_i=1]}), 0, 0)$ and the last element of $c$ is $\frac{1}{\sqrt{p(1-p)}}\langle(p\frac{1}{n} \sum\limits_{i=1}^{n}\nabla h(\w_0; \a_i)\I_{[b_i=-1]} - (1-p) \frac{1}{n} \sum\limits_{i=1}^{n}\nabla h(\w_0; \a_i)\I_{[b_i=1]}), \w_0\rangle$. 

With (\ref{eq:ntk_approx}) and Lemma \ref{lem:guoyin_PL}, we know that for some $\mu>0$ and any $y$, i.e., 
\begin{equation}
\begin{split}
2\mu(P(x) - P(x_*))\leq \|\nabla P(x)\|^2 + \widetilde{O}(\Tilde{L}^3 \omega^{4/3} \sqrt{m}). 
\end{split}
\end{equation}
Since $\omega = O(\frac{\log m}{\sqrt{m}})$, by the choice of $m$, we know that 
\begin{equation}
\begin{split}
\|\nabla P(x)\|^2 \geq 2\mu(P(x) - P(x_*) - \epsilon). 
\end{split} 
\end{equation}

\end{proof}

\subsection{An Example of $x$-side-PL-Strongly-Concave Problem}
\begin{lemma}  
\label{lem:x_side_PL_case}
Let $x, y \in \mathbb{R}$, $f(x, y) = \frac{1}{2}x^2 + \sin^2 x \sin^2 y - 2y^2$. We have that $f(x, y)$ satisfies a $x$-side $\frac{1}{12}$ PL condition, is $2$-strongly concave in $y$ and has a saddle point $(0, 0)$. 
\end{lemma}
\begin{proof}
For any $x$,
\begin{equation}
    \nabla_y^2 f(x, y) = 2\sin^2 x \cos (2y) - 4 \in [-6, -2]. 
\label{eq:hessian_x_boun}
\end{equation}
Thus, $f(x, y)$ is $2$-strongly concave in $y$.

For any $y$, we know that $\hat{x}(y) = 0$, and 
\begin{equation}
|\nabla_x^2 f(x, y)| = |1 + 2\cos (2x)\sin^2 y|^2 \leq 3, 
\end{equation}
which together with (\ref{eq:hessian_x_boun}) implies that $f(x, y)$ is $6$-smooth. 

We also get
\begin{equation}
\frac{|\nabla_x f(x, y)|}{|x - \hat{x}(y)|} = \frac{|x + \sin(2x)\sin^2 y|}{|x|} \geq \frac{1}{2}, 
\end{equation}
which together with the $6$-smoothness we know that $f(x, y)$ satisfies a $x$-side $\frac{1}{12}$-PL condition by Appendix A of \citep{karimi2016linear}. 

Also it is easy to verify that 
\begin{equation}
f(0, y) \leq f(0, 0) \leq f(x, 0),
\end{equation}
therefore $(0, 0)$ is a saddle point.

\end{proof}

\subsection{Existence of a saddle point} 
\begin{proof}[Proof of Lemma \ref{lem:saddle_point}] 
Since $x_* = \arg\min_{x'} P(x')$, $\nabla P(x_*) = \nabla_x f(x_*, \hat{y}(x_*)) = 0$, where the first equality holds by the Lemma \ref{lem:primal_gradient_danskin}. Then noting the $x$-side PL condition $2\mu_x(f(x_*, \hat{y}(x_*)) - \min_{x'} f(x', \hat{y}(x_*))) \leq \|\nabla_x f(x_*, \hat{y}(x_*))\|^2 = 0$, we have
\begin{equation} 
x_* \in \hat{x}(\hat{y}(x_*)),
\end{equation} 
which is one of the optimal $x$ corresponding to the $\hat{y}(x_*)$.

Then we can conclude that $(x_*, \hat{y}(x_*))$ is a saddle point of $f(x,y)$, i.e., for any $x$ and $y\in\mathcal{Y}$,
\begin{equation}
f(x_*, y) \leq f(x_*, \hat{y}(x_*)) \leq f(x,  \hat{y}(x_*)).
\end{equation} 
\end{proof}

\section{Using Different Step Sizes for Primal and Dual Variables}

In previous sections, we used the same  step size for for primal and dual Variables in order to  simplify the analysis. Actually, step sizes for primal and dual variables can be set different. In this section, we provide an analysis and rewrite the algorithm in Algorithm \ref{alg:main_xy_diff}, \ref{alg:subroutine_xy_diff}. Similar as before, we first provide a unified theorem.

\begin{algorithm}[htbp]
\caption {Proximal Stage Stochastic Method: PES-$\mathcal A$}
\begin{algorithmic}[1]
\STATE{Initialization: $\bx_0\in \R^d, \by_0 \in \mathcal{Y}, \gamma, T_1, \eta^x_1, \eta^y_1, a$. } 
\FOR{$k=1,2, ..., K$} 
\STATE{$x_0^k = \bx_{k-1}$, $y_0^k = \by_{k-1}$;} 
\STATE{$ (\bx_k, \by_k)$ = $\mathcal A(f, x^k_0, y^k_0, \eta^x_k,\eta^y_k, T_k, \gamma)$;} 
\STATE{$\eta^x_{k+1}=\eta^x_k/a$,$\eta^y_{k+1}=\eta^y_k/a$,
$T_{k+1}=a T_k$;} 
\ENDFOR 
\STATE{\textbf{return} $(\bx_K, \by_K)$.}
\end{algorithmic}
\label{alg:main_xy_diff}
\end{algorithm}

\begin{algorithm}[htbp]
\caption {Stochastic Algorithm $\mathcal A$($f, x_0, y_0, \eta^x, \eta^y, T, \gamma, u_0, v_0)$}  
\begin{algorithmic}
\STATE{Initialization: 
$(x_0, y_0)$} 
\STATE Let $\{\xi_0, \xi_1,\ldots, \xi_T\}$
be independent random variables. 
\FOR{$t=1, ..., T$} 
\STATE {~~~~~$x_t = \Pi^{\gamma}_{x_{t-1}, x_0}(\eta^x \mathcal G(x_{t-1}; \xi_{t-1})) $;}
\STATE {~~~~~$y_t = \Pi_{y_{t-1}}(\eta^y \mathcal G(y_{t-1}; \xi_{t-1})) $;} 
\vspace*{0.1in}
\ENDFOR  
\STATE{\textbf{return} $\bar{x} = \frac{1}{T}  \sum\limits_{t=1}^{T} x_t, \bar{y} = \frac{1}{T}  \sum\limits_{t=1}^{T} y_t$.} 
\end{algorithmic} 
\label{alg:subroutine_xy_diff}  
\end{algorithm}

\begin{theorem} 
Suppose Assumption \ref{ass1} and Assumption  \ref{ass3} hold. 
Assume we have a subroutine in the $k$-th stage of Algorithm \ref{alg:main_xy_diff} that can return $\bx_k, \by_k$ such that
\begin{align}
\begin{split}
\E[\emph{\text{Gap}}_k(\bx_k, \by_k)] 
\leq \E[\frac{C_1}{\eta^x_k T_k}\|\hat{x}_k(\bar{y}_k) - x_0^k\|^2 + \frac{C_1}{\eta^y_k T_k}\|\hat{y}_k(\bar{x}_k) - y_0^k\|^2] 
+ (\eta^x_k+\eta^x_k) C_2,
\end{split} 
\label{equ:thm:appendix_unify_3_sub_xydiff}
\end{align}
where $C_1$ and $C_2$ are constants corresponding to the specific subroutine. 
Take $\gamma = 2\rho$ and
denote $\hat{L} = L + 2\rho$ and $c = 4\rho+\frac{248}{53} \hat{L} \in O(L + \rho)$.
Define $\Delta_k = P(x_0^k) - P(x_*) + \frac{8\hat{L}}{53c}\emph{\text{Gap}}_k(x_0^k, y_0^k)$  and $\epsilon_0=\emph{\text{Gap}}(\bx_0, \by_0)$.
Then we can set $\eta^x_k = \eta^x_0  \exp(-(k-1)\frac{2\mu}{c+2\mu})$,  $\eta^y_k = \eta^y_0  \exp(-(k-1)\frac{2\mu}{c+2\mu})$, $T_k = \left\lceil\frac{212 C_1}{\min\{\eta^x_0\rho, \eta^y_0\mu_y\}}\exp\left((k-1)\frac{2\mu}{c+2\mu}\right)\right\rceil$. 
After $K = \left\lceil\max\left\{\frac{c+2\mu}{2\mu}\log \frac{4\epsilon_0}{\epsilon},
\frac{c+2\mu}{2\mu} \log \frac{16(\eta^x_0+\eta^y_0) \hat{L}KC_2}{(c+2\mu)\epsilon} \right\}\right\rceil$ stages, we can have $\Delta_{K+1} \leq \epsilon$. 
The total stochastic first-order oracle call complexity is 
$\widetilde{O}\left(\max\left\{\frac{(L+\rho)C_1 \epsilon_0}{\min\{\eta^x_0\rho, \eta^y_0\mu_y\} \mu \epsilon}, (\frac{1}{\rho} + \frac{\eta^x_0}{\eta^y_0 \mu_y} + \frac{\eta^y_0}{\eta^x_0\rho} + \frac{1}{\mu_y})
\frac{(L + \rho)^2 C_2} {\mu^2 \epsilon} \right\}\right)$. 
\label{thm:appendix_unify_3_xy_diff} 
\end{theorem} 
\textbf{Remark.} As long as $O(\eta^x_0)\leq \eta^y_0 \leq O(\frac{\eta^x_0}{\mu_y})$, $\eta^x_0\geq O(\mu \mu_y)$, and $\eta^y_0\geq O(\mu)$, then $\eta^x_0$, $\eta^y_0$ can be separately tuned without harming the complexity bound. 

\begin{proof}[Proof of Theorem \ref{thm:appendix_unify_3_xy_diff}] 
Since $f(x, y)$ is $\rho$-weakly convex in $x$ for any $y$, $P(x) = \max\limits_{y'\in \mathcal{Y}} f(x, y')$ is also $\rho$-weakly convex. 
Taking $\gamma = 2\rho$, we have
\begin{align}  
\begin{split} 
P(\bx_{k-1}) &\geq P(\bx_k) + \langle \nabla P(\bx_k), \bx_{k-1} - \bx_k\rangle - \frac{\rho}{2} \|\bx_{k-1} - \bx_k\|^2 \\ 
& = P(\bx_k) + \langle \nabla P(\bx_k) + 2 \rho (\bx_k - \bx_{k-1}), \bx_{k-1} - \bx_k\rangle + \frac{3\rho}{2} \|\bx_{k-1} - \bx_k\|^2 \\ 
& \overset{(a)}{=} P(\bx_k) + \langle \nabla P_k(\bx_k), \bx_{k-1} - \bx_k \rangle + \frac{3\rho}{2} \|\bx_{k-1} - \bx_k\|^2 \\ 
& \overset{(b)}{=}  P(\bx_k) - \frac{1}{2\rho} \langle \nabla P_k(\bx_k), \nabla P_k(\bx_k) - \nabla P(\bx_k) \rangle + \frac{3}{8\rho} \|\nabla P_k(\bx_k) - \nabla P(\bx_k)\|^2 \\ 
& = P(\bx_k) - \frac{1}{8\rho} \|\nabla P_k(\bx_k)\|^2
- \frac{1}{4\rho} \langle \nabla P_k(\bx_k), \nabla P(\bx_k)\rangle + \frac{3}{8\rho} \|\nabla P(\bx_k)\|^2,
\end{split} 
\label{local:P_weakly_2} 
\end{align} 
where $(a)$ and $(b)$ hold by the definition of $P_k(x)$.

Rearranging the terms in (\ref{local:P_weakly_2}) yields
\begin{align}
\begin{split}
P(\bx_k) - P(\bx_{k-1}) &\leq \frac{1}{8\rho} \|\nabla P_k(\bx_k)\|^2 + \frac{1}{4\rho}\langle \nabla P_k(\bx_k), \nabla P(\bx_k)\rangle - \frac{3}{8\rho} \|\nabla P(\bx_k)\|^2 \\
&\overset{(a)}{\leq} \frac{1}{8\rho} \|\nabla P_k(\bx_k)\|^2
+ \frac{1}{8\rho} (\|\nabla P_k(\bx_k)\|^2 + \|\nabla P(\bx_k)\|^2) - \frac{3}{8\rho} \|P(\bx_k)\|^2 \\
& = \frac{1}{4\rho}\|\nabla P_k(\bx_k)\|^2 - \frac{1}{4\rho}\|\nabla P(\bx_k)\|^2\\
& \overset{(b)}{\leq} \frac{1}{4\rho} \|\nabla P_k(\bx_k)\|^2 - \frac{\mu}{2\rho}(P(\bx_k) - P(x_*)),
\end{split} 
\end{align}
where $(a)$ holds by using $\langle \mathbf{a}, \mathbf{b}\rangle \leq \frac{1}{2}(\|\mathbf{a}\|^2 +  \|\mathbf{b}\|^2)$, and $(b)$ holds by the $\mu$-PL property of $P(x)$.

Thus, we have 
\begin{align}
\left(4\rho+2\mu\right) (P(\bx_k) - P(x_*)) - 4\rho (P(\bx_{k-1}) - P(x_*)) \leq \|\nabla P_k(\bx_k)\|^2. 
\label{local:nemi_thm_nabla_P_k_2} 
\end{align} 

Since $\gamma = 2\rho$, $f_k(x, y)$ is $\rho$-strongly convex in $x$ and $\mu_y$ strong concave in $y$.
Apply Lemma \ref{lem:Yan1} to $f_k$, we know that 
\begin{align} 
\frac{\rho}{4} \|\hat{x}_k(\by_k) - x_0^k\|^2 + \frac{\mu_y}{4} \|\hat{y}_k(\bx_k) - y_0^k\|^2 \leq {\text{Gap}}_k(x_0^k, y_0^k) + {\text{Gap}}_k(\bx_k, \by_k). 
\end{align}

By the setting of $\eta^x_k = \eta^x_0  \exp\left(-(k-1)\frac{2\mu}{c+2\mu}\right)$, $\eta^y_k = \eta^y_0  \exp\left(-(k-1)\frac{2\mu}{c+2\mu}\right)$, and 
$T_k = \left\lceil\frac{212 C_1}{\min\{\eta^x_0\rho, \eta^y_0\mu_y\}} \exp  \left((k-1)\frac{2\mu}{c+2\mu}\right)\right\rceil$, we note that $\frac{C_1}{\eta^x_k T_k} \leq \frac{\rho}{212}$ and $\frac{C_1}{\eta^y_k T_k} \leq \frac{\mu_y}{212}$. 
Applying (\ref{equ:thm:appendix_unify_3_sub_xydiff}), we have  
\begin{align}  
\begin{split} 
&\E[{\text{Gap}}_k(\bx_k, \by_k)] 
\leq (\eta^x_k + \eta^y_k) C_2 
+ \frac{1}{53} \E\left[\frac{\rho}{4}\|\hat{x}_k(\by_k) - x_0^k\|^2 + \frac{\mu_y}{4}\|\hat{y}_k(\bx_k) - y_0^k\|^2 \right]
\\  
& \leq (\eta^x_k+\eta^y_k) C_2 + \frac{1}{53} \E\left[{\text{Gap}}_k(x_0^k, y_0^k) + {\text{Gap}}_k(\bx_k, \by_k)\right]. 
\end{split} 
\end{align} 

Since $P(x)$ is $L$-smooth and $\gamma = 2\rho$, then $P_k(x)$ is $\hat{L} = (L+2\rho)$-smooth. 
According to Theorem 2.1.5 of \citep{DBLP:books/sp/Nesterov04}, we have 
\begin{align}
\begin{split} 
& \E[\|\nabla P_k(\bx_k)\|^2] \leq 2\hat{L}\E(P_k(\bx_k) - \min\limits_{x\in \mathbb{R}^{d}} P_k(x)) \leq 2\hat{L}\E[{\text{Gap}}_k(\bx_k, \by_k)] \\
& = 2\hat{L}\E[4{\text{Gap}}_k(\bx_k, \by_k) - 3{\text{Gap}}_k(\bx_k, \by_k)] \\ 
&\leq 2\hat{L} \E \left[4\left((\eta^x_k+\eta^y_k) C_2+  \frac{1}{53}\left({\text{Gap}}_k(x_0^k, y_0^k) + {\text{Gap}}_k(\bx_k, \by_k)\right)\right) - 3{\text{Gap}}_k(\bx_k,\by_k)\right] \\
& = 2\hat{L} \E \left[4(\eta^x_k+\eta^y_k) C_2 +   \frac{4}{53}{\text{Gap}}_k(x_0^k, y_0^k) - \frac{155}{53}{\text{Gap}}_k(\bx_k, \by_k)\right].
\end{split}  
\label{local:nemi_thm_P_smooth_2}
\end{align}

Applying Lemma \ref{lem:Yan5} to  (\ref{local:nemi_thm_P_smooth_2}), we have
\begin{align*}
\begin{split} 
& \E[\|\nabla P_k(\bx_k)\|^2] \leq 2\hat{L} \E \bigg[4(\eta^x_k+\eta^y_k) C_k + \frac{4}{53}{\text{Gap}}_k(x_0^k, y_0^k) \\  
&~~~~~~~~~~~~~~~~~~~~~~~~~~~~~~~~~~~~~~~~  
- \frac{155}{53} \left(\frac{3}{50} {\text{Gap}}_{k+1}(x_0^{k+1}, y_0^{k+1}) + \frac{4}{5} (P(x_0^{k+1}) - P(x_0^k))\right) \bigg] \\
& = 2\hat{L}\E \bigg[4(\eta^x_k+\eta^y_k) C_2 \!+ \! \frac{4}{53}{\text{Gap}}_k(x_0^k, y_0^k) \!-\! \frac{93}{530}{\text{Gap}}_{k+1}(x_0^{k+1}, y_0^{k+1}) \!-\! 
\frac{124}{53} (P(x_0^{k+1}) - P(x_0^k)) \bigg]. 
\end{split} 
\end{align*}

Combining this with  (\ref{local:nemi_thm_nabla_P_k_2}), rearranging the terms, and defining a constant $c = 4\rho + \frac{248}{53}\hat{L} \in O(L+\rho)$, we get
\begin{align} 
\begin{split}
&\left(c + 2\mu\right)\E [P(x_0^{k+1}) - P(x_*)] + \frac{93}{265}\hat{L} \E[{\text{Gap}}_{k+1}(x_0^{k+1}, y_0^{k+1})] \\ 
&\leq \left(4\rho + \frac{248}{53} \hat{L}\right) \E[P(x_0^{k}) - P(x_*)] 
+ \frac{8\hat{L}}{53} \E[{\text{Gap}}_k(x_0^k, y_0^k)] 
+ 8 (\eta^x_k+\eta^y_k) \hat{L}C_2  \\ 
& \leq c \E\left[P(x_0^k) - P(x_*) + \frac{8\hat{L}}{53c} {\text{Gap}}_k(x_0^k, y_0^k)\right] + 8 (\eta^x_k+\eta^y_k) \hat{L}C_2.
\end{split}  
\end{align} 

Using the fact that $\hat{L} \geq \mu$,
\begin{align}
\begin{split}
(c+2\mu) \frac{8\hat{L}}{53c} = \left(4\rho + \frac{248}{53}\hat{L} + 2\mu\right)\frac{8\hat{L}}{53(4\rho + \frac{248}{53}\hat{L})} \leq \frac{8\hat{L}}{53} + \frac{16\mu \hat{L}}{248\hat{L}} \leq \frac{93}{265} \hat{L}. 
\end{split}
\end{align}

Then, we have
\begin{align}
\begin{split} 
&(c+2\mu)\E \left[P(x_0^{k+1}) - P(x_*) + \frac{8\hat{L}}{53c}{\text{Gap}}_{k+1}(x_0^{k+1}, y_0^{k+1})\right] \\ 
&\leq c \E \left[P(x_0^{k}) - P(x_*) 
+  \frac{8\hat{L}}{53c}{\text{Gap}}_{k}(x_0^{k},  y_0^{k})\right] 
+ 8 (\eta^x_k+\eta^y_k) \hat{L}C_2. 
\end{split}
\end{align}

Defining $\Delta_k = P(x_0^{k}) - P(x_*) +  \frac{8\hat{L}}{53c}{\text{Gap}}_{k}(x_0^{k}, y_0^{k})$, then
\begin{align}
\begin{split}
&\E[\Delta_{k+1}] \leq \frac{c}{c+2\mu} \E[\Delta_{k}] +  \frac{8(\eta^x_k+\eta^y_k) \hat{L} C_2}{c+2\mu}.
\end{split}
\end{align}

Using this inequality recursively, it yields
\begin{align} 
\begin{split}
& E[\Delta_{K+1}] \leq \left(\frac{c}{c+2\mu}\right)^K E[\Delta_1]
+ \frac{8 \hat{L} C_2}{c+2\mu} \sum\limits_{k=1}^{K} \left((\eta^x_k+\eta^y_k) \left(\frac{c}{c+2\mu}\right)^{K+1-k} \right).
\end{split}
\end{align} 

By definition,
\begin{align*} 
\begin{split}
\Delta_1 &= P(x_0^1) - P(x_*) + \frac{8\hat{L}}{53c}\text{Gap}_1(x_0^1, y_0^1) \\
& = P(\bx_0) - P(x_*) + \left( f(\bx_0, \hat{y}_1(\bx_0)) +  \frac{\gamma}{2}\|\bx_0 - \bx_0\|^2 -  f(\hat{x}_1(\by_0), \by_0) - \frac{\gamma}{2}\|\hat{x}_1(\by_0) - \bx_0\|^2 \right) \\ 
& \leq \epsilon_0 + f(\bx_0, \hat{y}_1(\bx_0)) - f(\hat{x}(\by_0), \by_0) \leq 2\epsilon_0.
\end{split} 
\end{align*}
Using inequality $1-x \leq \exp(-x)$,
we have
\begin{align*}
\begin{split} 
&\E[\Delta_{K+1}] \leq \exp\left(\frac{-2\mu K}{c+2\mu}\right)\E[\Delta_1] + \frac{8 (\eta^x_0 + \eta^y_0) \hat{L} C_2} {c+2\mu} 
\sum\limits_{k=1}^{K}\exp\left(-\frac{2\mu K}{c+2\mu}\right) \\
&\leq 2\epsilon_0 \exp\left(\frac{-2\mu K}{c+2\mu}\right)
+ \frac{8 (\eta^x_0+\eta^y_0) \hat{L} C_2 }{c+2\mu}  
K\exp\left(-\frac{2\mu K}{c+2\mu}\right). 
\end{split}  
\end{align*}  

To make this less than $\epsilon$, it suffices to make
\begin{align*}
\begin{split}
& 2\epsilon_0 \exp\left(\frac{-2\mu K}{c+2\mu}\right) \leq \frac{\epsilon}{2},\\
&\frac{8 (\eta^x_0+\eta^y_0) \hat{L} C_2 }{c+2\mu} 
K\exp\left(-\frac{2\mu K}{c+2\mu}\right) \leq \frac{\epsilon}{2}.
\end{split} 
\end{align*}

Let $K$ be the smallest value such that $\exp\left(\frac{-2\mu K}{c+2\mu}\right) \leq \min \{ \frac{\epsilon}{4\epsilon_0}, 
\frac{(c+2\mu)\epsilon}{16 (\eta^x_0+\eta^y_0) \hat{L} K C_2}\}$. 
We can set $K = \left\lceil\max\bigg\{\frac{c+2\mu}{2\mu}\log \frac{4\epsilon_0}{\epsilon}, 
\frac{c+2\mu}{2\mu}\log \frac{16 (\eta^x_0+\eta^y_0) \hat{L} K C_2} {(c+2\mu)\epsilon}  \bigg\}\right\rceil$. 
Then, the total stochastic first-order oracle call complexity is 
\begin{align*} 
\sum\limits_{k=1}^{K}T_k &\leq O\left( \frac{212 C_1}{\min\{\eta^x_0,\eta^y_0\}\min\{\rho,\mu_y\}}
\sum\limits_{k=1}^{K}\exp\left((k-1)\frac{2\mu}{c+2\mu}\right) \right)\\    
& \leq O\bigg(\frac{212C_1}{\min\{\eta^x_0,\eta^y_0\}\min\{\rho,\mu_y\}}  \frac{\exp(K\frac{2\mu}{c+2\mu}) -  1}{\exp(\frac{2\mu}{c+2\mu})-1}  \bigg)\\ 
& \overset{(a)}{\leq} \widetilde{O}   \left(\frac{cC_1}{\min\{\eta^x_0\rho,\eta^y_0\mu_y\}\mu}  \max\left\{\frac{\epsilon_0}{\epsilon}, 
\frac{(\eta^x_0+\eta^y_0) \hat{L} K C_2} 
{(c+2\mu)\epsilon} \right\}\right) \\
& \leq  \widetilde{O}\left(\max\left\{\frac{(L+\rho)C_1 \epsilon_0}{\min\{\eta^x_0\rho,\eta^y_0\mu_y\} \mu \epsilon}, 
\frac{(\eta^x_0 + \eta^y_0)(L + \rho)^2 C_2} {\min\{\eta^x_0\rho, \eta^y_0\mu_y\}\mu^2 \epsilon}\right\} \right) \\
&\leq \widetilde{O}\left(\max\left\{\frac{(L+\rho)C_1 \epsilon_0}{\min\{\eta^x_0\rho,\eta^y_0\mu_y\} \mu \epsilon}, 
(\frac{1}{\rho} + \frac{\eta^x_0}{\eta^y_0 \mu_y} + \frac{\eta^y_0}{\eta^x_0\rho} + \frac{1}{\mu_y})\frac{(L + \rho)^2 C_2} {\mu^2 \epsilon}\right\} \right), 
\end{align*} 
where $(a)$ uses the setting of $K$ and $\exp(x) - 1\geq x$, and $\widetilde{O}$ suppresses logarithmic factors. 
\end{proof}

\begin{theorem}
Consider Algorithm \ref{alg:main_xy_diff} that uses Algorithm \ref{alg:subroutine_xy_diff} as a subroutine.
Suppose Assumption \ref{ass1}, \ref{ass5}, \ref{ass3} hold.
Assume $\E\|\nabla_x f(x, y; \xi)\|^2\leq B^2$ and $\E\|\nabla_y f(x, y; \xi)\|^2\leq B^2$.
Take $\gamma = 2\rho$ and
denote $\hat{L} = L + 2\rho$ and $c = 4\rho+\frac{248}{53} \hat{L} \in O(L + \rho)$.
Define $\Delta_k = P(x_0^k) - P(x_*) + \frac{8\hat{L}}{53c}\emph{\text{Gap}}_k(x_0^k, y_0^k)$  and $\epsilon_0=\emph{\text{Gap}}(\bx_0, \by_0)$.
Then we can set $\eta^x_k = \eta^x_0  \exp(-(k-1)\frac{2\mu}{c+2\mu})\leq \frac{1}{\rho}$,$\eta^y_k = \eta^y_0  \exp(-(k-1)\frac{2\mu}{c+2\mu})$, $T_k = \left\lceil\frac{212 C_1}{ \min\{\eta^x_0\rho, \eta^y_0\mu_y\}}\exp\left((k-1)\frac{2\mu}{c+2\mu}\right)\right\rceil$. 
After $K =\left\lceil \max\left\{\frac{c+2\mu}{2\mu}\log \frac{4\epsilon_0}{\epsilon},
\frac{c+2\mu}{2\mu} \log \frac{80 \eta_0 \hat{L}KB_2}{(c+2\mu)\epsilon} \right\}\right\rceil$ stages, we can have $\Delta_{K+1} \leq \epsilon$. 
The total stochastic first-order oracle call complexity is 
$\widetilde{O}\left(
\frac{(L + \rho)^2 B^2} {\mu^2 \min\{\rho,\mu_y\} \epsilon}\right)$. 
\label{thm:gda_primal_diff_eta} 
\end{theorem}

\begin{proof}[Proof of Theorem \ref{thm:gda_primal_diff_eta}] 
Using Lemma \ref{lem:Yan4}, we can set $\gamma = 2\rho$ and 
$\eta_0 = \frac{1}{\rho}$. 
Then it follows that 
\begin{equation*} 
\begin{split}
E[\text{Gap}_k(\bx_k, \by_k)] \leq 
\frac{2}{\eta^x_k T_k} E[\|\hat{x}_k(\bar{y}_k) - x_0^k\|^2 + \frac{2}{\eta^y_k T_k} E[\|\hat{y}_k(\bar{x}_k) - y_0^k\|^2 + \frac{5(\eta^x_k+\eta^y_k) B^2}{2}.
\end{split} 
\end{equation*} 
We plug in $C_1 = 2$ and $C_2 = 5B^2/2 $ to Theorem \ref{thm:appendix_unify_3_xy_diff}
and the conclusion follows.
\end{proof} 
\end{document}